%% file: main.tex
\documentclass[a4paper, 11pt]{article}

\pdfoutput=1

\usepackage{amsmath}
\usepackage{amsfonts}
\usepackage{amssymb}
\usepackage{mathtools}
\usepackage{amsthm}
\usepackage{bm}
\usepackage{hyperref}
\usepackage{titlesec}
\usepackage{graphicx}
\usepackage{subcaption}
\usepackage{array}
\usepackage{authblk}
\usepackage{fullpage}
\usepackage{setspace}
\usepackage{algorithm}
\usepackage{algpseudocode}
\usepackage{url}

\ifpdf
  \DeclareGraphicsExtensions{,.pdf,.png,.jpg}
\else
  \DeclareGraphicsExtensions{}
\fi

\newcommand{\tensor}[1]{\mathbf{#1}}
\newcommand{\mexp}[1]{\mathbb{E}\left\{#1\right\}}
\newcommand{\trans}{\top}
\newcommand{\cov}{\mathrm{Cov}}

\newcommand{\half}{\frac{1}{2}}

\numberwithin{equation}{section}
\renewcommand{\theequation}{\thesection.\arabic{equation}}

\newtheorem{thm}{Theorem}[section]

\makeatletter
\def\mathcenterto#1#2{\mathclap{\phantom{#1}\mathclap{#2}}\phantom{#1}}
\let\old@widetilde\widetilde
\def\widetildeto#1#2{\mathcenterto{#2}{\old@widetilde{\mathcenterto{#1}{#2\,}}}}
\makeatother

\newcommand{\wtl}{\widetildeto{I}{L}}

\newcommand{\wtX}{\widetildeto{X}{\bm X}}
\newcommand{\wty}{\widetildeto{y}{\bm y}}
\newcommand{\wtm}{\widetildeto{y}{\bm\mu}}
\newcommand{\wtc}{\widetildeto{y}{\bm c}}
\newcommand{\wtC}{\widetildeto{C}{\tensor C}}

\title{Physics-Informed CoKriging: A Gaussian-Process-Regression-Based
Multifidelity Method for Data-Model Convergence}

\author[1]{Xiu Yang\footnote{xiu.yang@pnnl.gov}}
\author[1]{David Barajas-Solano\footnote{david.barajas-solano@pnnl.gov}}
\author[2]{Guzel Tartakovsky\footnote{guzel.tartakovsky@pnnl.gov}} 
\author[1]{Alexandre M. Tartakovsky\footnote{alexandre.tartakovsky@pnnl.gov}}
\affil[1]{Advanced Computing, Mathematics and Data Division, Pacific Northwest
  National Laboratory, Richland, WA 99352}
\affil[2]{Hydrology Group, Pacific Northwest National Laboratory, Richland, WA 99352}

\begin{document}
\maketitle


\begin{abstract}
In this work, we propose a new Gaussian process regression (GPR)-based multifidelity method: physics-informed CoKriging (CoPhIK).
  In CoKriging-based multifidelity methods, the quantities of interest are
  modeled as linear combinations of multiple parameterized stationary Gaussian
  processes (GPs), and the hyperparameters of these GPs are estimated from data
  via optimization.
  In CoPhIK, we construct a GP representing low-fidelity data using physics-informed Kriging (PhIK), and %
  model the discrepancy between low- and high-fidelity data using a parameterized GP with hyperparameters identified via optimization.
  Our approach reduces the cost of optimization for inferring hyperparameters by incorporating partial physical knowledge.
  We prove that the physical constraints in the form of deterministic linear operators are satisfied up to an error bound.
  Furthermore, we combine CoPhIK with a greedy active learning algorithm for guiding the selection of additional observation locations.
  The efficiency and accuracy of CoPhIK are demonstrated for reconstructing the
  partially observed modified Branin function, reconstructing the sparsely
  observed state of a steady state heat transport problem, and learning a
  conservative tracer distribution from sparse tracer concentration measurements.


\end{abstract}

\input{intro}

\input{method}

\input{results}

\section{Conclusion}
\label{sec:concl}

In this work, we propose CoPhIK, a CoKriging-based multifidelity method that 
uses the PhIK method to combine numerical simulations of physical systems with
accurate observations. The CoPhIK method first constructs a ``low-fidelity'' GP
$Y_{_L}$ via PhIK by estimating its mean and covariance function from the output
of a stochastic physical model reflecting partial knowledge of the system; then,
it models the discrepancy between high-fidelity data (e.g., observations of the
system) and the low-fidelity GP using auxiliary GP $Y_{_d}$. Whereas in PhIK the
(prior) mean and covariance function are entirely defined by the stochastic 
model outputs, CoPhIK incorporates high-fidelity data in constructing the prior
mean and covariance. In addition, we propose a modified version of PhIK by
introducing a correction term to the prior mean. We also provide upper bounds
for the error in enforcing physical constraints using CoPhIK and modified PhIK.
Finally, we demonstrate that an active learning algorithm in combination with
Kriging, PhIK and CoPhIK suggests very different locations for new observations,
and the two physics-informed methods result in significantly more accurate
predictions with reduced uncertainty.

The CoPhIK method presented in this work consists of a non-stationary part
$Y_{_L}$, and a stationary part $Y_{_d}$, in contrast with the ``data-driven''
Kriging method, for which the prior mean and covariance are estimated from data
only, usually requiring an assumption of stationarity. The accuracy of CoPhIK 
predictions and of enforcing physical constraints depends both on the accuracy
of the physical model and the selection of the kernel for $Y_{_d}$. One can 
further improve CoPhIK by employing a non-stationary model for $Y_{_d}$, thus
rendering the GP $Y_{_H}$ fully non-stationary. The choice of non-stationary 
kernel for $Y_{_d}$ is problem dependent, and it may be achieved by exploiting 
additional physical information whenever available.

The presented physics-informed methods are nonintrusive, and can utilize 
existing domain codes to compute the necessary ensembles. Therefore, these 
methods are suitable for large-scale complex applications for which physical 
models and codes are available.

\section*{Acknowledgments}
We thank Dr. Nathan Baker for fruitful discussion.
This work was supported by the U.S. Department of Energy (DOE), Office of 
Science, Office of Advanced Scientific Computing Research (ASCR) as part of the
Uncertainty Quantification in Advection-Diffusion-Reaction Systems. A portion of the 
research described in this paper was conducted under the Laboratory Directed 
Research and Development Program at Pacific Northwest National Laboratory (PNNL).
PNNL is operated by Battelle for the DOE under Contract DE-AC05-76RL01830.

\input{append}

\bibliographystyle{plain}
\bibliography{ref}

\end{document}

%% file: intro.tex
\section{Introduction}

Gaussian processes (GPs) are a widely used tool in applied mathematics,
statistics, and machine learning for regression, classification, and
optimization~\cite{forrester2008engineering,sacks1989design,stein2012interpolation}.
GP regression (GPR), also known as \emph{Kriging} in geostatistics, constructs a
statistical model of a partially observed process by assuming that its observations
are a realization of a GP. A GP is uniquely described by its mean and
covariance function (also known as \emph{kernel}). In standard (referred to here
as \emph{data-driven}) GPR, usually parameterized forms of mean and covariance
functions are assumed, and the hyperparameters of these functions (e.g., variance
and correlation length) are estimated from data by maximizing the log marginal
likelihood of the data. GPR is also closely related to kernel machines in machine
learning, but it provides a richer characterization in the result, as it provides 
uncertainty estimates~\cite{williams2006gaussian}. GP is also connected to infinite
neural networks, that is, networks with an infinite number of hidden
units~\cite{neal2012bayesian}.

There are several variants of GPR, including simple, ordinary, and universal
Kriging~\cite{kitanidis1997introduction}. Ordinary Kriging is the most widely 
used GPR method. It assumes stationarity of the random field, including constant
mean and variance, and a prescribed stationary covariance function. The
stationarity assumption reduces the number of hyperparameters and the model
complexity. For example, in universal Kriging, the mean is modeled as a linear
combination of basis functions~\cite{armstrong1984problems}, which increases the
number of unknown parameters and may lead to non-convex optimization problems.
Although the assumption of stationarity may not be suitable for some application
problems, it is often necessary as there are usually not enough data to compute
accurate estimates of non-stationary mean and covariance functions. Progress 
have been made at incorporating physical knowledge into
kernels, e.g.,~\cite{schober2014probabilistic,hennig2015probabilistic,
chkrebtii2016bayesian,cockayne2017bayesian, raissi2017machine,raissi2018numerical}
by computing kernels for systems governed by linear and weakly nonlinear 
(allowing accurate linearization) ordinary and partial differential equations.
Such kernels are computed by substituting a GPR approximation of the system's 
state variables into the governing equation and obtaining a system of equations
for the kernel hyperparameters. For complex linear systems, computing the kernel
in such a way can become prohibitively expensive, and for strongly nonlinear
systems, it may not be possible at all.

In our previous work~\cite{YangTT18}, we proposed the physics-informed Kriging
method (PhIK) that incorporates (partial) physical knowledge into GPRs.
In modeling complex systems, it is common to treat unknown parameters and fields
as random parametrs and fields, and the resulting realizations of the state of 
the system are employed to study the uncertainty of the model or the system.
The goal of PhIK is to exploit the information of the system provided by these
realizations to assimilate observations. In PhIK, such random realizations are 
used to compute the prior mean and covariance. A similar idea is used in the
ensemble Kalman filter (EnKF)~\cite{evensen2003ensemble} and the formula of
the ``filtering step" is equivalent to the PhIK prediction Eq.~\eqref{eq:phik}.
Whereas EnKF introduces uncertainty mainly from the observation noise and 
evolves an ensemble of state variables drawn from the posterior distribution of
the previous time step, PhIK utilizes the stochasticity in models and directly
uses simulation outputs for prediction without redrawing the ensemble in each 
time step. Not only does PhIK provide prediction or reconstruction in the form
of posterior mean, it also performs \emph{uncertainty reduction} (UR) in the 
form of posterior variance. More importantly, PhIK posterior mean satisfies 
linear physical constraints with a bounded error~\cite{YangTT18}, which is 
critical for guaranteeing the predictive value of the method. The main drawback
of PhIK is that it is highly dependent on the physical model, because the prior
mean and covariance are determined entirely by the model and are not informed by
data. Therefore, convergence of PhIK to the true solution with the increasing
number of available observations is slower than in the data-driven GPR if the
physical model is incorrect.

In this work, we propose a physics-informed CoKriging (CoPhIK) method, an
extension of the CoKriging-based multifidelity 
framework~\cite{kennedy2000predicting, forrester2007multi} to physics-informed
Kriging. In this context, the direct observations of a physical system are
considered as high-fidelity data and the stochastic physical model outputs are
treated as low-fidelity data. CoPhIK uses PhIK to construct a GP $Y_{_L}$ that
regresses low-fidelity data, and uses another parameterized GP $Y_{_d}$ to model
the discrepancy between low- and high-fidelity data by assuming a specific 
kernel; then it infers hyperparameters of the GP model for $Y_{_d}$ via
optimization. Subsequently, CoPhIK uses a linear combination of $Y_{_L}$ and
$Y_{_d}$ to represent high-fidelity data. The mean and covariance in CoPhIK
integrate physical model outputs and observation data; therefore, CoPhIK is
expected to have better accuracy than PhIK in some applications (e.g., the first
two numerical examples in Section~\ref{sec:numeric}). On the other hand, due to
the introduction of the GP $Y_d$, CoPhIK may lose some capacity for satisfying
physical constraints with respect to PhIK, as will be shown in the error 
estimate provided by Theorem~\ref{thm:phicok}.

This work is organized as follows: Section~\ref{sec:method} summarizes the GPR
framework and physics-informed Kriging (Sections~\ref{subsec:gpr} 
to~\ref{subsec:enkrig}), and introduces the CoPhIK method 
(Section~\ref{subsec:phicok}). Section~\ref{sec:numeric} provides three 
numerical examples to demonstrate the efficiency of the proposed method. 
Conclusions are presented in Section~\ref{sec:concl}.


%% file: method.tex
\section{Methodology}
\label{sec:method}

We begin this section by reviewing the general GPR
framework~\cite{williams2006gaussian}, the ordinary Kriging method based on the
assumption of stationary GP~\cite{forrester2008engineering}, and the PhIK 
method~\cite{YangTT18}. Then, we introduce the modified PhIK and CoPhIK methods.

\subsection{GPR framework}
\label{subsec:gpr}

We consider the spatial dependence of a scalar state of a physical system.
Let $\mathbb{D} \subseteq \mathbb{R}^d,~d\in\mathbb{N}$, be the spatial domain,
$y : \mathbb{D} \to \mathbb{R}$ denote the state of interest, and let 
$y^{(1)}, y^{(2)} \dotsc, y^{(N)}$, denote $N$ observations of $y$ collected at the
observation locations $\bm X = \{ \bm x^{(i)} \}^N_{i = 1}$, where 
$x^{(i)} \in \mathbb{D} \subseteq \mathbb{R}^d, y^{(i)}\in\mathbb{R}$. The
observations are arranged into the observation vector 
$\bm y=(y^{(1)}, y^{(2)}, \dotsc, y^{(N)})^\trans$. We aim to predict $y$ at any
new location $\bm x^* \in \mathbb{D}$. The GPR method assumes that the 
observation vector $\bm y$ is a realization of the $N$-dimensional random vector
with multivariate Gaussian distribution
\begin{equation*}
  \bm Y = \left(Y(\bm x^{(1)}, \omega), Y(\bm x^{(2)}, \omega), \dotsc, Y(\bm x^{(N)}, \omega)\right)^\trans,
\end{equation*}
where $Y(\cdot, \cdot) : \mathbb{D} \times \Omega \to \mathbb{R}$ is a GP 
defined on the probability space $(\Omega, \mathcal{F}, P)$. Of note, the
observation coordinates $\bm x^{(i)}$ can be considered as parameters for the GP
$Y$ such that $Y(\bm x^{(i)}, \cdot)$ is a Gaussian random variable for any 
$\bm x^{(i)} \in \mathbb{D}$. For brevity, we denote $Y(\bm x, \cdot)$ by 
$Y(\bm x)$. The GP $Y$ is usually represented using GP notation as
\begin{equation}
  \label{eq:gp0}
Y(\bm x) \sim \mathcal{GP}\left(\mu(\bm x), k(\bm x, \bm x')\right),
\end{equation}
where $\mu(\cdot):\mathbb{D}\rightarrow\mathbb{R}$ and $k(\cdot,\cdot):
\mathbb{D}\times\mathbb{D}\rightarrow\mathbb{R}$ are the mean and covariance
functions
\begin{align}
  \mu(\bm x) & = \mexp{Y(\bm x)},\\ 
  k(\bm x,\bm x') & = \cov\left\{Y(\bm x), Y(\bm x')\right\}
                    = \mexp{\left [ Y(\bm x)-\mu(\bm x) \right ] \left [Y(\bm x')-\mu(\bm x') \right ]}.
\end{align}
The variance of $Y(\bm x)$ is $k(\bm x, \bm x)$, and its standard deviation is
$\sigma(\bm x)=\sqrt{k(\bm x,\bm x)}$. The covariance matrix of the random 
vector $\bm Y$ is then given by
\begin{equation}
  \label{eq:cov_matrix0}
  \tensor C = 
  \begin{pmatrix}
    k(\bm x^{(1)}, \bm x^{(1)}) & \cdots & k(\bm x^{(1)}, \bm x^{(N)}) \\
    \vdots & \ddots & \vdots  \\
    k(\bm x^{(N)}, \bm x^{(1)}) & \cdots & k(\bm x^{(N)}, \bm x^{(N)})
  \end{pmatrix}.
\end{equation}
When the functions $\mu(\bm x)$ and $k(\bm x,\bm x')$ are parameterized, their
hyperparameters are identified by maximizing the log marginal likelihood of the
observations (see examples in 
Section~\ref{subsec:stationary_gpr})~\cite{williams2006gaussian} 
\begin{equation}
  \label{eq:lml}
  \ln L=-\dfrac{1}{2}(\bm y-\bm\mu)^\trans \tensor C^{-1} (\bm
  y-\bm\mu)-\dfrac{1}{2}\ln |\tensor C|-\dfrac{N}{2}\ln 2\pi,
\end{equation}
where $\bm\mu=(\mu(\bm x^{(1)}),\dotsc,\bm x^{(N)})^\trans$.

The GPR prediction at $\bm x^{*}$ consists of the posterior distribution 
$y(\bm x^*)\sim\mathcal{N}(\hat y(\bm x^*), \hat s^2(\bm x^*))$, with posterior
mean and variance given by
  \begin{align}
\label{eq:krig}
\hat y(\bm x^*) & = \mu(\bm x^*) + 
  \bm c(\bm x^*)^\trans\tensor{C}^{-1}(\bm y-\bm\mu),  \\
  \hat s^2(\bm x^*) & = 
   \sigma^2(\bm x^*)-\bm c(\bm x^*)^\trans \tensor{C}^{-1}\bm c(\bm x^*),
 \end{align}
and $\bm c(\bm x^*)$ is the vector of covariances
\begin{equation}
  \label{eq:cov_vec}
  \bm c(\bm x^*)=\left(k(\bm x^{(1)},\bm x^*), k(\bm x^{(2)},\bm x^*), 
             \cdots,k(\bm x^{(N)},\bm x^*)\right)^\trans.
\end{equation}
In practice, it is common to employ the posterior mean $\hat y(\bm x^*)$ as the
prediction. The variance $\hat s^2(\bm x^*)$ is often called the mean squared 
error (MSE) of the prediction because 
$\hat s^2(\bm x^*)=\mexp{(\hat y(\bm x^*)-Y(\bm x^*))^2}$~\cite{forrester2008engineering}.
Consequently, $\hat s(\bm x^*)$ is called the root mean squared error (RMSE). 

To account for observation noise, one can model the noises as independent and
identically distributed (iid) Gaussian random variables with zero mean and 
variance $\delta^2$, and replace $\tensor C$ in Eqs.~\ref{eq:lml} 
to~\ref{eq:cov_vec} with $\tensor C+\delta^2\tensor I$. In this study, we
assume that observations of $\bm y$ are noiseless. If $\tensor C$ is not 
invertible or its condition number is very large, one can add a small
regularization term $\alpha\tensor I$, where $\alpha$ is a small positive real
number, to $\tensor C$ such that it becomes full rank. Adding the regularization
term is equivalent to assuming there is iid observation noise with variance
$\alpha$.


\subsection{Stationary GPR}
\label{subsec:stationary_gpr}

In the widely used ordinary Kriging method, a stationary GP is assumed.
Specifically, $\mu$ is set as a constant $\mu(\bm x)\equiv\mu$, and 
$k(\bm x, \bm x')=k(\bm\tau)$, where $\bm\tau=\bm x-\bm x'$. Consequently,
$\sigma^2(\bm x)=k(\bm x,\bm x)=k(\bm 0)=\sigma^2$ is a constant. Popular forms
of kernels include polynomial, exponential, Gaussian (squared-exponential), and
Mat\'{e}rn functions. For example, the Gaussian kernel can be written as
$k(\bm\tau)=\sigma^2\exp\left(-\frac{1}{2}\Vert \bm \tau \Vert^2_w\right)$, 
where the weighted norm is defined as 
$\Vert \bm \tau \Vert^2_w=\sum_{i=1}^d \left( \tau_i / l_i \right)^2$,
The constant $\sigma$ and the correlation lengths along each direction, 
$l_i \in \mathbb{R}$, $i=1,\dotsc,d$ are the hyperparameters of the Gaussian
kernel.

For the stationary kernel, the covariance matrix of observations, $\tensor C$, 
can be written as $\tensor C=\sigma^2\tensor\Psi$, where 
$\Psi_{ij}=\exp(-\frac{1}{2}\Vert\bm x^{(i)}-\bm x^{(j)}\Vert_w^2)$ for the
Gaussian kernel. In the maximum likelihood (MLE) framework, the estimators of
$\mu$ and $\sigma^2$, denoted as $\hat\mu$ and $\hat\sigma^2$, are
\begin{equation}
  \label{eq:krig_est}
   \hat\mu=\dfrac{\bm 1^\trans \tensor\Psi^{-1}\bm y}
                  {\bm 1^\trans\tensor\Psi^{-1}\bm 1}, \qquad
   \hat\sigma^2=\dfrac{(\bm y-\bm 1\mu)^\trans\tensor\Psi^{-1}(\bm y-\bm 1\mu)}{N},
\end{equation}
where $\bm 1$ is a vector of $1$s. The hyperparameters $l_i$ are estimated by
maximizing the log marginal likelihood, Eq.~\eqref{eq:lml}. The prediction of 
$y$ at location $\bm x^*$ is
\begin{equation}
  \label{eq:krig_pred1}
  \hat y(\bm x^*) = \hat\mu + \bm\psi^\trans\tensor\Psi^{-1}(\bm y-\bm 1\hat\mu), 
\end{equation}
where $\bm\psi$ is a vector of correlations between the observed data and the
prediction, given by
\begin{equation*}
  \bm\psi=\bm\psi(\bm x^*)=\frac1{\sigma^2}\left(k(\bm x^{(1)}-\bm x^*), 
    \cdots,k(\bm x^{(N)}-\bm x^*)\right)^\trans,
\end{equation*}
and the MSE of the prediction is
\begin{equation}
\label{eq:krig_var1}
  \hat s^2(\bm x^*) = \hat\sigma^2\left(1-\psi^\trans\tensor\Psi^{-1}\bm\psi\right).
\end{equation}

A more general approach to GPR is to employ parameterized nonstationary 
covariance kernels. Nonstationary kernels can be obtained by modifying 
stationary covariance kernels, e.g.,~\cite{sampson1992nonparametric,mackay1998introduction,paciorek2004nonstationary, plagemann2008nonstationary, brahim2004gaussian, monterrubio2018posterior}, 
or from neural networks with specific activation functions,
e.g.,~\cite{neal2012bayesian,pang2018neural}, among other approaches.
Many of these approaches assume a specific functional form for the correlation
function, chosen according to expert knowledge.
The key computational challenge in these data-driven GPR is the optimization 
step of maximizing the (log marginal) likelihood. In many practical cases, this
is a non-convex optimization problem, and the condition number of $\tensor C$ or
$\tensor\Psi$ can be quite large. Another fundamental challenge is that
parameterized models for mean and covariance usually don't account for physical
constraints, and therefore require a large amount of data to accurately model
physics.


\subsection{PhIK}
\label{subsec:enkrig}

The recently proposed PhIK method~\cite{YangTT18} takes advantage of existing
expert knowledge in the form of stochastic physics-based models. These 
stochastic models for physical systems include random parameters or random 
fields to reflect the lack of understanding (of physical laws) or knowledge (of
the coefficients, parameters, etc.) of the real system. Monte Carlo (MC)
simulations of the stochastic physical model can be conducted to generate an
ensemble of state variables, from which the mean and covariance are estimated.
This mean and covariance estimates are then employed to construct a GP model of
the state variables. As such, there is no need to assume a specific parameterized
covariance kernel or solve an optimization problem for the hyperparameters of the
kernel.  

Given $M$ realizations of a stochastic model 
$u(\bm x;\omega)$, $\bm x \in\mathbb{D}$, $\omega\in\Omega$, denoted as 
$\{Y^m(\bm x)\}_{m=1}^M$,
we build the following GP model:
\begin{equation} 
  Y(\bm x) \sim \mathcal{GP}(\mu_{_\mathrm{MC}}(\bm x), k_{_\mathrm{MC}}(\bm x, \bm x')),
\end{equation} 
where $\mu_{_\mathrm{MC}}$ and $k_{_\mathrm{MC}}$ are the ensemble mean and
covariance functions
\begin{equation}
\begin{aligned}
  \label{eq:phik}
  \mu(\bm x) & \approx \mu_{_\mathrm{MC}}(\bm x)=\dfrac{1}{M}\sum_{m=1}^M Y^m(\bm x), \\
  k(\bm x, \bm x') &\approx k_{_\mathrm{MC}}(\bm x, \bm x')=\dfrac{1}{M-1} \sum_{m=1}^M
  \left(Y^m(\bm x)-\mu_{_\mathrm{MC}}(\bm x)\right) \left(Y^m(\bm x')-\mu_{_\mathrm{MC}}(\bm x')\right).
\end{aligned}
\end{equation}
The covariance matrix of observations can be estimated as
\begin{equation}
  \label{eq:mc_cov}
  \tensor C\approx \tensor C_{_\mathrm{MC}}=\dfrac{1}{M-1} \sum_{m=1}^M
  \left(\bm Y^m-\bm\mu_{_\mathrm{MC}} \right)
  \left(\bm Y^m-\bm\mu_{_\mathrm{MC}} \right)^\trans,
\end{equation}
where $\bm Y^m=\left(Y^m(\bm x^{(1)}), \dotsc, Y^m(\bm x^{(N)})\right)^\trans$
and $\bm\mu_{_\mathrm{MC}}=\left(\mu_{_\mathrm{MC}}(\bm x^{(1)}), \dotsc, \mu_{_\mathrm{MC}}(\bm x^{(N)})\right)^\trans$.
The prediction and MSE at location $\bm x^*\in\mathbb{D}$ are
\begin{align}
\label{eq:krig_pred2}
\hat y(\bm x^*) & = \mu_{_\mathrm{MC}}(\bm x^*) + 
\bm c_{_\mathrm{MC}}(\bm x^*)^\trans\tensor{C}_{_\mathrm{MC}}^{-1}(\bm y-\bm\mu_{_\mathrm{MC}}),  \\
\hat s^2(\bm x^*) & = \hat\sigma_{_\mathrm{MC}}^2(\bm x^*)-\bm c_{_\mathrm{MC}}(\bm x^*)^\trans 
\tensor{C}_{_\mathrm{MC}}^{-1}\bm c_{_\mathrm{MC}}(\bm x^*),
\end{align}
where $\hat\sigma^2_{_\mathrm{MC}}(\bm x^*)=k_{_\mathrm{MC}}(\bm x^*,\bm x^*)$
is the variance of the set $\{Y^m(\bm x^*)\}_{m=1}^M$, and  \\
$\bm c_{_\mathrm{MC}}(\bm x^*)=\left(k_{_\mathrm{MC}}(\bm x^{(1)},\bm x^*),
\dotsc,k_{_\mathrm{MC}}(\bm x^{(N)},\bm x^*)\right)^\trans$.

It was demonstrated in \cite{YangTT18} that PhIK predictions satisfy 
\emph{linear} physical constraints up to an error bound that depends on the
numerical error, the discrepancy between the physical model and real system, and
the smallest eigenvalue of matrix $\tensor C$. Linear physical constraints 
include periodic, Dirichlet or Neumann boundary condition, and linear equation 
$\mathcal{L} u = g$, where $\mathcal{L}$ is a linear differential or
integral operator. For example, let $u(\bm x;\omega)$ be a stochastic model of 
the velocity potential for a incompressible flow, i.e.,
$\nabla\cdot(\nabla u(\bm x;\omega))=0$; then PhIK guarantees that 
$\nabla\hat y(\bm x)$ is a divergence-free field.

In PhIK, MC simulation of the stochastic physical model for computing
$\mu_{_\mathrm{MC}}$ and $k_{_\mathrm{MC}}$ can be replaced by more efficient
approaches, such as quasi-Monte Carlo~\cite{niederreiter1992random}, multi-level
Monte Carlo (MLMC)~\cite{giles2008multilevel}, probabilistic 
collocation~\cite{XiuH05}, Analysis Of Variance (ANOVA)~\cite{YangCLK12},
compressive sensing~\cite{YangK13}, the moment equation and PDF 
methods~\cite{Tart2017WRR,barajassolano-2018-probability}, and the bi-fidelity
method~\cite{zhu2017multi}. Linear physical constraints are preserved if
$\mu_{_\mathrm{MC}}(\bm x)$ and $k_{_\mathrm{MC}}(\bm x,\bm x')$ are computed 
using a \emph{linear combination} of the realizations $\{Y^m(\bm x)\}_{m=1}^M$.
As an example, we present the MLMC-based PhIK~\cite{YangTT18} in Appendix A.

Further, the matrix $\tensor C$ and vector $\bm\mu$ are fixed in 
Eq.~\eqref{eq:lml} for a given ensemble $\{Y^m\}_{m=1}^M$. Thus, the log 
marginal likelihood is fixed. We can modify PhIK by adding a correction term to
$\mu_{_\mathrm{MC}}(\bm x)$ to increase the likelihood. Specifically, we replace
$\mu_{_\mathrm{MC}}(\bm x)$ by $\mu_{_\mathrm{MC}}(\bm x) + \Delta\mu$, where
$\Delta\mu$ is a constant. Then, taking the derivative of $\ln L$ with respect 
to $\Delta\mu$ and setting it to be zero yields
\begin{equation}
  \Delta\mu=\dfrac{\bm 1^\trans \tensor\Psi^{-1}(\bm y-\bm 1\mu)}
                  {\bm 1^\trans\tensor\Psi^{-1}\bm 1}. 
\end{equation}
This modification has a potential to increase the accuracy of the prediction,
but it may also violate some physical constraints, e.g., the Dirichlet boundary
condition. We name this method \emph{modified PhIK}, and provide the following
theorem on how well it preserves linear physical constraints.
\begin{thm}
 \label{thm:mphik}
Assume that a stochastic model $u(\bm x;\omega)$ defined on 
$\mathbb{D}\times \Omega$ ($\mathbb{D}\subseteq\mathbb{R}^d$) satisfies
$\Vert\mathcal{L}u(\bm x;\omega)-g(\bm x;\omega)\Vert\leq\epsilon$ for any
$\omega\in\Omega$, where $\mathcal{L}$ is a deterministic bounded linear 
operator, $g(\bm x;\omega)$ is a well-defined function on
$\mathbb{R}^d\times\Omega$, and $\Vert\cdot\Vert$ is a well-defined function
norm. $\{Y^m(\bm x)\}_{m=1}^M$ are a finite number of realizations of
$u(\bm x;\omega)$, i.e., $Y^m(\bm x)=u(\bm x;\omega^m)$. Then, the prediction
$\hat y(\bm x)$ from modified PhIK satisfies
\begin{equation}
  \label{eq:err_bound}
  \begin{aligned}
    \Big\Vert\mathcal{L}\hat y(\bm x)-\overline{g(\bm x)}\Big\Vert 
    & \leq \epsilon+\left[ 2\epsilon\sqrt{\dfrac{M}{M-1}}+\sigma\left(g(\bm
x;\omega^m)\right) \right] \cdot \\
    &\quad\qquad \Big\Vert\tensor C^{-1}_{_\mathrm{MC}}\Big\Vert_2\Big\Vert\bm
    y-\bm\mu_{_\mathrm{MC}}-\Delta\mu\bm 1\Big\Vert_2\sum_{i=1}^N  \sigma(Y^m(\bm x^{(i)})) +
    \Big\Vert\mathcal{L}\Delta\mu\Big\Vert,
\end{aligned}
\end{equation}
where $\sigma\left(Y^m(\bm x^{(i)})\right)$ is the standard deviation of the
data set $\{Y^m(\bm x^{(i)})\}_{m=1}^M$ for fixed $\bm x^{(i)}$,
$\displaystyle\overline{g(\bm x)}=\dfrac{1}{M}\sum_{m=1}^M g(\bm x;\omega^m)$,
and $\displaystyle\sigma\left(g(\bm x;\omega^m)\right)=
\left(\dfrac{1}{M-1}\sum_{m=1}^M 
\left\Vert g(\bm x;\omega^m)-\overline{g(\bm x)}\right\Vert^2 \right)^{\half}$.
\end{thm}
\begin{proof}
  The modified PhIK prediction can be written as
  \begin{equation}
    \label{eq:krig_form1}
    \hat y(\bm x) = \mu_{_\mathrm{MC}}(\bm x) + \Delta\mu 
                  + \sum_{i=1}^N \tilde a_i k_{_\mathrm{MC}}(\bm x, \bm x^{(i)}), 
  \end{equation}
  where $\tilde a_i$ is the $i$-th entry of 
  $\tensor C_{_\mathrm{MC}}^{-1}(\bm y-\bm\mu_{_\mathrm{MC}}-\Delta\mu\bm 1)$. 
  According to Theorem 2.1 and Corollary 2.2\ in \cite{YangTT18},
  \begin{equation*}
  \begin{aligned}
  \left\Vert\mathcal{L}\hat y(\bm x)-\overline{g(\bm x)}\right\Vert 
   \leq & \left\Vert\mathcal{L}\Big(\mu_{_\mathrm{MC}}(\bm x)
     +\sum_{i=1}^N \tilde a_i k_{_\mathrm{MC}}(\bm x, \bm x^{(i)})\Big)
     -\overline{g(\bm x)}\right\Vert + \left\Vert\mathcal{L}\Delta\mu\right\Vert\\
  \leq & \epsilon + \left[ 2\epsilon\sqrt{\dfrac{M}{M-1}} + \sigma\left(g(\bm x;\omega^m)\right)\right]
  \left\Vert\tensor C_{_\mathrm{MC}}^{-1}\right\Vert_2
  \left\Vert\bm y-\bm\mu_{_\mathrm{MC}}-\Delta\mu\bm 1\right\Vert_2 
  \sum_{i=1}^N \sigma\left(Y^m(\bm x^{(i)})\right) \\ 
  & + \Big\Vert\mathcal{L}\Delta\mu\Big\Vert.
    \end{aligned}
  \end{equation*}
\end{proof}
For $\Delta\mu=0$, the bound~\eqref{eq:err_bound} reverts to the bound 
in~\cite{YangTT18} and the modified PhIK method reverts to PhIK. In some cases,
the term $\Vert\mathcal{L}\Delta\mu\Vert=0$, e.g, when $\mathcal{L}$ is a
differential operator such as the Neumann boundary condition operator.


\subsection{CoPhIK}
\label{subsec:phicok}

CoKriging was originally formulated to compute predictions of sparsely observed
states of physical systems by leveraging observations of other states or 
parameters of the system~\cite{stein1991universal,knotters1995comparison}.
Recently, it has been employed for constructing multi-fidelity
models~\cite{kennedy2000predicting,le2014recursive,perdikaris2015multi}, and has
been applied in various areas, 
e.g.,~\cite{laurenceau2008building,brooks2011multi,pan2017optimizing}.
In this work, we propose a novel multi-fidelity method, CoPhIK, that
integrates PhIK and CoKriging by combining numerical simulations and 
high-fidelity observations. Our multi-fidelity method is based on Kennedy and
O'Hagan's CoKriging framework presented 
in~\cite{kennedy2000predicting, forrester2008engineering}.

We briefly review the formulation of CoKriging for two-level multi-fidelity
modeling in \cite{forrester2007multi}. Suppose that we have high-fidelity data
(e.g., accurate measurements of states) 
$\bm y_{_H}=\left(y_{_H}^{(1)}, \dotsc,y_{_H}^{(N_H)}\right)^\trans$ at 
locations $\bm X_{_H} = \{\bm x_{_H}^{(i)}\}_{i=1}^{N_H}$, and low-fidelity data
(e.g., simulation results) 
$\bm y_{_L}=\left(y_{_L}^{(1)}, \dotsc,y_{_L}^{(N_L)}\right)^\trans$ at 
locations 
$\bm X_{_L} = \{\bm x_{_L}^{(i)}\}_{i=1}^{N_L}$, where $y_{_H}^{(i)}, y_{_L}^{(i)}\in\mathbb{R}$ and $\bm x_{_H}^{(i)}, \bm x_{_L}^{(i)}\in\mathbb{D}\subseteq\mathbb{R}^d$.
By concatenating the observation locations and data respectively, i.e.,
$\wtX=\{\bm X_{_L}, \bm X_{_H}\}$ and 
$\wty = \left(\bm y_{_L}^{\trans}, \bm y_{_H}^{\trans} \right)^{\trans}$, 
we can construct a \emph{multivariate} GP via Kriging as detailed
in~\cite{barajas2018multivariate}. Kennedy and O'Hagan proposed an alternative
formulation of CoKriging based on the auto-regressive model for $Y_{_H}$
\begin{equation}
  \label{eq:arm}
  Y_{_H}(\bm x) = \rho Y_{_L}(\bm x) + Y_{_d}(\bm x),
\end{equation}
where $\rho \in \mathbb{R}$ is a regression parameter and $Y_d(\bm x)$ is a GP
that models the difference between $Y_{_H}$ and $\rho Y_{_L}$. This model 
assumes that
\begin{equation}
  \label{eq:arm-cov}
  \cov \left\{Y_{_H}(\bm x), Y_{_L}({\bm x}') \mid Y_{_L}(\bm x) \right\}=0,
  \text{for all } \bm x' \neq \bm x, \ \bm x, \bm x' \in \mathbb{D}.
\end{equation}
It was shown in~\cite{ohagan-1998-markov} that the assumption of
Eq.~\eqref{eq:arm-cov} implies the auto-regressive model of Eq.~\eqref{eq:arm}.
The covariance of observations, $\wtC$, is then given by
\begin{equation}
  \label{eq:cokrig_cov}
  \wtC=
  \begin{pmatrix}
    \tensor C_{_L}(\bm X_{_L}, \bm X_{_L}) & \rho \tensor C_{_L}(\bm X_{_L}, \bm X_{_H}) \\
    \rho \tensor C_{_L} (\bm X_{_H}, \bm X_{_L}) & \rho^2 \tensor C_{_L}(\bm
    X_{_H}, \bm X_{_H}) + \tensor C_{_d}(\bm X_{_H}, \bm X_{_H})
  \end{pmatrix}
\end{equation}
where $\tensor C_{_L}$ is the covariance matrix based on GP $Y_{_L}$'s kernel 
$k_{_L}(\cdot,\cdot)$, and $\tensor C_{_d}$ is the covariance matrix based on GP
$Y_{_d}$'s kernel $k_{_d}(\cdot,\cdot)$. One can assume parameterized forms for
these kernels (e.g., Gaussian kernel) and then simultaneously identify their
hyperparameters along with $\rho$ by maximizing the following log marginal
likelihood: 
\begin{equation}
  \label{eq:lml2}
  \ln \wtl=-\dfrac{1}{2}(\wty-\wtm)^\trans\wtC^{-1} (\wty
  -\wtm)-\dfrac{1}{2}\ln \left\vert\wtC\right\vert-\dfrac{N_{_H}+N_{_L}}{2}\ln 2\pi.
\end{equation}
Alternatively, one can employ the following two-step
approach~\cite{forrester2007multi, forrester2008engineering}:
\begin{enumerate}
\item Use Kriging to construct $Y_{_L}$ using  $\{\bm X_{_L}, \bm y_{_L}\}$. 
\item Denote $\bm y_{_d} = \bm y_{_H} - \rho \bm y_{_L}(\bm X_{_H})$, where 
$\bm y_{_L}(\bm X_{_H})$ are the values of $\bm y_{_L}$ at locations common to
those of $\bm X_{_H}$, then construct $Y_d$ using $\{\bm X_{_H}, \bm y_{_d}\}$
via Kriging. 
\end{enumerate}
The posterior mean and variance of $Y_{_H}$ at $\bm x^*\in \mathbb{D}$ are given
by
\begin{align}
  \label{eq:cokrig_mean}
  \hat y(\bm x^*) & =  \mu_{_H}(\bm x^*) + 
      \wtc(\bm x^*)^\trans\wtC^{-1}(\wty-\wtm),  \\
  \label{eq:cokrig_var}
      \hat s^2(\bm x^*) & = \rho^2\sigma^2_{_L}(\bm x^*) + \sigma^2_{_d}(\bm x^*)
      - \wtc(\bm x^*)^\trans \wtC^{-1}\wtc(\bm x^*),
\end{align}
where $\mu_{_H}(\bm x^*)=\rho\mu_{_L}(\bm x^*)+\mu_d(\bm x^*)$, $\mu_{_L}(\bm x)$
is the mean of $Y_{_L}(\bm x)$, $\mu_d(\bm x)$ is the mean of $Y_d(\bm x)$,
$\sigma^2_{_L}(\bm x^*)=k_{_{L}}(\bm x^*, \bm x^*)$,
$\sigma^2_d(\bm x^*)=k_d(\bm x^*, \bm x^*)$, and
\begin{align}
  \label{eq:cokrig_mu}
\wtm & = \begin{pmatrix}\bm\mu_{_L}\\ \bm\mu_{_H}\end{pmatrix}
  = \begin{pmatrix}\big(\mu_{_L}(\bm x_{_{L}}^{(1)})\cdots, 
    \mu_{_L}(\bm x_{_{L}}^{(N_{_L})})\big)^\trans \\
     \big(\mu_{_H}(\bm x_{_{H}}^{(1)})\cdots, 
\mu_{_H}(\bm x_{_{H}}^{(N_{_H})})\big)^\trans\end{pmatrix},  \\
  \label{eq:cokrig_c}
  \wtc(\bm x^*)& = \begin{pmatrix}\rho\bm c_{_L}(\bm x^*)\\ \bm c_{_H}(\bm
x^*)\end{pmatrix}
= \begin{pmatrix}\big(\rho k_{_L}(\bm x^*,\bm x_{_L}^{(1)}),\cdots,
  \rho k_{_{L}}(\bm x^*,\bm x_{_L}^{(N_{_L})})\big)^\trans \\
  \big(k_{_H}(\bm x^*,\bm x_{_H}^{(1)}),\cdots,k_{_{H}}(\bm x^*,\bm x_{_H}^{(N_{_H})})
\big)^\trans\end{pmatrix}, 
\end{align}
where
$k_{_H}(\bm x,\bm x') = \rho^2k_{_L}(\bm x,\bm x') + k_{_d}(\bm x, \bm x')$.
Here, we have neglected a small contribution to $\hat s^2$
(see~\cite{forrester2008engineering}).

Now we describe the CoPhIK method. We set $\bm X_{_L}=\bm X_{_H}$ to simplify 
the formula and computing, and denote $N=N_{_H}=N_{_L}$. 
We employ PhIK to construct GP $Y_{_L}$ using realizations 
$\{Y^m(\bm x)\}_{m=1}^M$ of a stochastic model $u(\bm x;\omega)$ on
$\mathbb{D}\times\Omega$. Specifically, we set
$\mu_{_L}(\bm x)=\mu_{_\mathrm{MC}}(\bm x)$ and 
$k_{_L}(\bm x,\bm x')=k_{_\mathrm{MC}}(\bm x,\bm x')$, where 
$\mu_{_\mathrm{MC}}$ and $k_{_\mathrm{MC}}$ are given by Eq.~\eqref{eq:phik}.
The GP model $Y_d$ is constructed using the same approach as in the second step
of the Kennedy and O'Hagan CoKriging framework. In other words, CoPhIK replaces
the first step of their framework with PhIK, and follows the same procedure for
the second step.

We proceed to describe the construction of $Y_{_d}$ in more detail. First, we 
set $\bm y_{_d}=\bm y_{_H}-\rho\bm\mu_{_{L}}(\bm X_{_L})$. The reason for this
choice is that $\mu_{_L}(\bm X_{_H})$ is the most probable observation of the GP
$Y_{_L}$. Next, we need to assume a specific form of the kernel function. 
Without loss of generality, in the following theoretical analysis and 
computational examples, we use the stationary Gaussian kernel model and constant
$\mu_{_d}$. Once $\bm y_{_d}$ is computed, and the form of $\mu_{_d}(\cdot)$ and
$k_{_d}(\cdot,\cdot)$ are decided, $Y_{_d}$ can the constructed as in ordinary
Kriging. Now that all components in $\ln\tilde L$ are specified except for the 
$\bm y_{_L}$ in $\wty$. We set $\bm y_{_L}$ as the realization from the ensemble
$\{Y^m\}_{m=1}^M$ that maximizes $\ln\tilde L$. The algorithm is summarized in
Algorithm~\ref{algo:phicok}.
\begin{algorithm}[!h]
  \caption{CoPhIK using stochastic simulation model $u(\bm x;\omega)$ on
  $\mathbb{D}\times\Omega$ ($\mathbb{D}\subseteq\mathbb{R}^d$), and
  high-fidelity observation 
  $\bm y_{_H}=(y_{_H}^{(1)}, \dotsc y_{_{H}}^{(N)})^\trans$ at locations 
  $\bm X_{_H}=\{\bm x_{_H}^{(i)}\}_{i=1}^{N}$.}
  \label{algo:phicok}
  \begin{algorithmic}[1]
    \State Conduct stochastic simulation, e.g., MC simulation, using 
    $u(\bm x;\omega)$ to generate realizations $\{Y^m\}_{m=1}^M$ on the
    entire domain $\mathbb{D}$.
    \State Use PhIK to construct GP $Y_{_L}$ on $D\times\Omega$, i.e.,
    $\mu_{_L}(\cdot)=\mu_{_\mathrm{MC}}(\cdot)$ and
    $k_{_{L}}(\cdot,\cdot)=k_{_\mathrm{MC}}(\cdot,\cdot)$ in Eq.~\eqref{eq:phik}.
    Compute $\mu_{_L}(\bm X_{_L})=\left(\mu_{_L}(\bm x_{_L}^{(1)}),\dotsc
    \mu_{_L}(\bm x_{_L}^{(N)})\right)^\trans$, and 
    $\tensor C_{_L}(\bm X_{_L}, \bm X_{_L})$ whose $ij$-th element is 
    $k_{_L}(\bm x_{_H}^{(i)},\bm x_{_H}^{(j)})$.
    Set $\tensor C_{_L}(\bm X_{_L}, \bm X_{_H})=\tensor C_{_L}(\bm X_{_H},
    \bm X_{_L})=\tensor C_{_L}(\bm X_{_H}, \bm X_{_H})=\tensor C_{_L}(\bm
    X_{_L}, \bm X_{_L})$ (because $\bm X_{_L}=\bm X_{_H}$). 
    \State Denote $\bm y_d=\bm y_{_H}-\rho\mu_{_L}(\bm X_{_L})$, choose a specific
    kernel function $k_d(\cdot,\cdot)$ (Gaussian kernel in this work) for the GP
    $Y_d$, and identify hyperparameters via maximizing the log marginal 
    likelihood Eq.~\eqref{eq:lml}, where $\bm y, \bm\mu, \tensor C$ are
    specified as $\bm y_{_d}, \bm\mu_{_d}, \tensor C_{_d}$, respectively. Then
    construct $\wtm$ in Eq.~\eqref{eq:cokrig_mu}, and $\tensor C_{_d}$ whose 
    $ij$-th element is $k_d(\bm x_{_H}^{(i)}, \bm x_{_H}^{(j)})$.
    \State Iterate over the set $\{Y^m\}_{m=1}^M$ to identify $Y^m$ that
    maximizes $\ln\tilde L$ in Eq.~\eqref{eq:lml2}, where 
    $\bm y_{_L}=(Y^m(\bm x_{_{H}}^{(1)}), \cdots, \bm x_{_{H}}^{(N)})^{\trans}$
    is used in $\wty$.
    \State Compute the posterior mean using Eq.~\eqref{eq:cokrig_mean}, and
    variance using Eq.~\eqref{eq:cokrig_var} for any $\bm x^*\in\mathbb{D}$.
  \end{algorithmic}
\end{algorithm}

Next, we analyze the form of $\wtC$. Recalling the choice 
$\bm X_{_L} = \bm X_{_H}$ and introducing the notation 
$\tensor C_1= C_{_{L}}(\bm X_{_L}, \bm X_{_L})$ and 
$\tensor C_2 = C_d(\bm X_{_H}, \bm X_{_H})$ in Eq.~\eqref{eq:cokrig_cov}, we can
write the inverse of $\wtC$ as
\begin{equation}
  \label{eq:cov_inv}
  \wtC^{-1} =
  \begin{pmatrix}
    \tensor C_1^{-1} + \rho^2\tensor C_2^{-1} & -\rho\tensor C_2^{-1} \\
    -\rho\tensor C_2^{-1} & \tensor C_2^{-1}
  \end{pmatrix}.
\end{equation}
Thus, 
\begin{equation}
  \begin{aligned}
    \wtC^{-1}(\wty-\wtm) & = 
  \begin{pmatrix}
    \tensor C_1^{-1} + \rho^2\tensor C_2^{-1} & -\rho\tensor C_2^{-1} \\
    -\rho\tensor C_2^{-1} &  \tensor C_2^{-1}
  \end{pmatrix}
  \begin{pmatrix}
    \bm y_{_L} - \bm\mu_{_L} \\
    \bm y_{_H} - \bm\mu_{_H} 
  \end{pmatrix} \\
  & = \begin{pmatrix}
  \big(\tensor C_1^{-1} + \rho^2\tensor C_2^{-1}\big)\big(\bm y_{_L} - \bm\mu_{_L}\big)
  -\rho\tensor C_2^{-1}\big(\bm y_{_H} - \bm\mu_{_H}\big) \\
  -\rho\tensor C_2^{-1}\big(\bm y_{_L} - \bm\mu_{_L}\big) 
  +  \tensor C_2^{-1}\big(\bm y_{_H} - \bm\mu_{_H}\big) 
      \end{pmatrix} \\
      & =\begin{pmatrix}
      \tensor C_1^{-1}\big(\bm y_{_L} - \bm\mu_{_L}\big) - \rho\tensor
    C_2^{-1}\big((\bm y_{_H} - \bm\mu_{_H})-\rho(\bm y_{_L} - \bm\mu_{_L})\big) \\
    \tensor C_2^{-1}\big((\bm y_{_H} - \bm\mu_{_H})-\rho(\bm y_{_L} - \bm\mu_{_L})\big)
    \end{pmatrix} \\
    & = \begin{pmatrix}
    \tensor C_1^{-1}\big(\bm y_{_L} - \bm\mu_{_L}\big) - \rho\tensor
    C_2^{-1}\big(\bm y_{_H}-\rho\bm y_{_L} -\bm 1\mu_d\big) \\
    \tensor C_2^{-1} \big(\bm y_{_H}-\rho\bm y_{_L} -\bm 1\mu_d\big)
    \end{pmatrix},
\end{aligned}
\end{equation}
where 
$\bm\mu_{_d}=\left(\mu_d(\bm x_{_H}^{(1)}),\dotsc,\mu_{_d}(\bm x_{_H}^{(N)})\right)^\trans=\bm 1\mu_{_d}$.
Therefore, the posterior mean at $\bm x^*\in\mathbb{D}$, given by 
Eq.~\eqref{eq:cokrig_mean}, can be rewritten as
\begin{equation}
  \begin{aligned}
    \hat y(\bm x^*) & = 
      \mu_{_H}(\bm x^*) + \big(\rho\bm c_{_L}(\bm x^*)^\trans,\bm c_{_H}(\bm x^*)^\trans\big)
\begin{pmatrix}
    \tensor C_1^{-1}\big(\bm y_{_L} - \bm\mu_{_L}\big) - \rho\tensor
    C_2^{-1}\big(\bm y_{_H}-\rho\bm y_{_L} -\bm 1\mu_d\big) \\
    \tensor C_2^{-1} \big(\bm y_{_H}-\rho\bm y_{_L} -\bm 1\mu_d\big)
    \end{pmatrix} \\
    & = \mu_{_H}(\bm x^*) 
    + \rho\bm c_{_L}(\bm x^*)^\trans \tensor C_1^{-1}\big(\bm y_{_L} - \bm\mu_{_L}\big)  
     + \big(\bm c_{_H}(\bm x^*)- \rho^2\bm c_{_L}(\bm
     x^*)\big)^\trans\tensor C_2^{-1} \big(\bm y_{_H}-\rho\bm y_{_L} -\bm 1\mu_d\big) \\
     & = \rho\left(\mu_{_L}(\bm x^*) 
    + \bm c_{_L}(\bm x^*)^\trans \tensor C_1^{-1}\big(\bm y_{_H} - \bm\mu_{_L}\big)\right)  
  - \rho\bm c_{_L}(\bm x^*)^\trans \tensor C_1^{-1}\big(\bm y_{_H} - \bm y _{_L}\big)   \\
  & \quad + \mu_{_d} + \big(\bm c_{_H}(\bm x^*)-\rho^2\bm c_{_L}(\bm x^*)\big)^\trans\tensor C_2^{-1} 
  \big(\bm y_{_H}-\rho\bm y_{_L} -\bm 1\mu_d\big),
  \end{aligned}
\end{equation}
which indicates that the posterior mean is of the functional form
\begin{equation}
  \label{eq:y_phicok_split}
  \begin{aligned}
    \hat y(\bm x) = %
    \underbrace{ \rho\left [ \mu_{_L}(\bm x) + \sum_{i=1}^Na_ik_{_L}(\bm x, \bm x^{(i)}) \right ]}_{S_1}%
    -\underbrace{\rho\sum_{i=1}^N b_i k_{_L}(\bm x, \bm x^{(i)})}_{S_2}%
    +\underbrace{\mu_{_d}+ \sum_{i=1}^N q_i k_{_d}(\bm x, \bm x^{(i)})}_{S_3},
  \end{aligned}
\end{equation}
where $a_i, b_i, q_i$ are the $i$th entries of 
$\tensor C_1^{-1}(\bm y_{_H}-\bm\mu_{_L}), \tensor C_1^{-1}(\bm y_{_H}-\bm y_{_L}),
\tensor C_2^{-1}(\bm y_{_H}-\rho\bm y_{_L}-\bm 1\mu_{_d})$, respectively.
Here, $S_1$ is the PhIK prediction multiplied by $\rho$, and $(-S_2 + S_3)$ can
be considered as the CoPhIK correction term. Furthermore, we note that
\begin{equation}
  \label{eq:y_phicok_split_part1}
  S_1-S_2 = \rho\left [ \mu_{_L}(\bm x) 
  + \sum_{i=1}^N(a_i-b_i)k_{_L}(\bm x,\bm x^{(i)}) \right ],
\end{equation}
has the same form as the PhIK prediction multiplied by $\rho$. This indicates 
that the error bound in preserving the physical constraints for this part is
similar to the PhIK's error bound. Therefore, using notations in 
Eq.~\eqref{eq:y_phicok_split}, we extend Theorem 2.1\ in~\cite{YangTT18} for 
CoPhIK as the following.
\begin{thm}
  \label{thm:phicok}
Assume that a stochastic model $u(\bm x;\omega)$ defined on 
$\mathbb{D}\times \Omega$ ($\mathbb{D}\subseteq\mathbb{R}^d$) satisfies
$\Vert\mathcal{L}u(\bm x;\omega)-g(\bm x;\omega)\Vert\leq\epsilon$ for any
$\omega\in\Omega$, where $\mathcal{L}$ is a deterministic bounded linear 
operator, $g(\bm x;\omega)$ is a well-defined function on
$\mathbb{R}^d\times\Omega$, and $\Vert\cdot\Vert$ is a well-defined function
norm. $\{Y^m(\bm x)\}_{m=1}^M$ are a finite number of realizations of
$u(\bm x;\omega)$, i.e., $Y^m(\bm x)=u(\bm x;\omega^m)$.
  Then, the prediction $\hat y(\bm x)$ from CoPhIK satisfies
  \begin{equation}
    \label{eq:err_bound_phicok}
    \begin{aligned}
      \Big\Vert\mathcal{L}\hat y(\bm x)-\overline{g(\bm x)}\Big\Vert
      & \leq \rho\epsilon  + (1-\rho) \Big\Vert\overline{g(\bm x)}\Big\Vert \\
      & \quad + \rho\left[2\epsilon\sqrt{\dfrac{M}{M-1}}+\sigma\left(g(\bm
  x;\omega^m)\right) \right] \cdot \left\Vert\tensor C^{-1}_1\right\Vert_2 
    \Vert\bm y_{_L}-\bm\mu_{_L}\Vert_2\sum_{i=1}^N\sigma\left(Y^m(\bm x^{(i)})\right)  \\
    & \quad + \left\Vert\mathcal{L}\mu_d\right\Vert + \left\Vert\tensor
    C_2^{-1}\right\Vert_2\Vert \bm y_{_H}-\rho\bm y_{_L}-\bm 1\mu_d\Vert_2
    \sum_{i=1}^N \left\Vert\mathcal{L} k_{_d}(\bm x, \bm x^{(i)})\right\Vert,
    \end{aligned}
  \end{equation}
where $\sigma\left(Y^m(\bm x^{(i)})\right)$, $\displaystyle\overline{g(\bm x)}$,
$\displaystyle\sigma\left(g(\bm x;\omega^m)\right)$ are defined in Theorem~\ref{thm:mphik}.
\end{thm}
\begin{proof}
  According to Eq.~\eqref{eq:y_phicok_split} and Eq.~\eqref{eq:y_phicok_split_part1},
  \begin{equation*}
    \begin{aligned}
      \left\Vert \mathcal{L}\hat y(\bm x) - \overline{g(\bm x)}\right\Vert 
    & \leq \left\Vert \mathcal{L}(S_1-S_2)-\overline{g(\bm x)}\right\Vert 
      + \left\Vert \mathcal{L} S_3\right\Vert \\
      & \leq \rho\left\Vert \mathcal{L}\mu_{_{L}}(\bm x)-\overline{g(\bm x)}\right\Vert +
      (1-\rho)\left\Vert\overline{g(\bm x)}\right\Vert 
      + \rho\left\Vert\mathcal{L}\bigg(\sum_{i=1}^N \tilde a_ik_{_L}\big(\bm
      x,\bm x^{(i)}\big)\bigg)\right\Vert 
      + \left\Vert \mathcal{L}S_3 \right\Vert.
    \end{aligned}
  \end{equation*}
  Following the same procedure outlined in Theorem 2.1 of~\cite{YangTT18}, we have
  \begin{equation*}
     \Big\Vert \mathcal{L}\mu_{_{L}}(\bm x)-\overline{g(\bm x)}\Big\Vert\leq\epsilon,
   \end{equation*}
   and
  \begin{equation*}
    \bigg\Vert\mathcal{L}\bigg(\sum_{i=1}^N \tilde a_i k_{_L}\big(\bm x,\bm x^{(i)}\big)\bigg)\bigg\Vert \leq 
    \left[2\epsilon\sqrt{\dfrac{M}{M-1}}+\sigma(g(\bm x;\omega^m))\right]\cdot
    \left\Vert\tensor C^{-1}_1\right\Vert_2
    \Vert\bm y_{_L}-\bm\mu_{_L}\Vert_2
    \sum_{i=1}^N\sigma\left(Y^m(\bm
      x^{(i)})\right).
  \end{equation*}
  Moreover, 
  \begin{equation*}
    \Vert\mathcal{L}S_3\Vert
    \leq \left\Vert\mathcal{L}\mu_d\right\Vert +\sum_{i=1}^N |q_i| \left\Vert\mathcal{L}k_{_d}(\bm x, \bm x^{(i)})\right\Vert
    \leq \left\Vert\mathcal{L}\mu_d\right\Vert + \max_i |q_i|\sum_{i=1}^N
    \Big\Vert\mathcal{L}k_{_d}(\bm x, \bm x^{(i)})\Big\Vert, 
  \end{equation*}
  where $q_i$ is defined in Eq.~\eqref{eq:y_phicok_split}, and
  \[\max_i|q_i|=\Vert\tensor C_2^{-1}(\bm y_{_H}-\rho\bm y_{_L}-\bm 1\mu_d)
    \Vert_{\infty} \leq
  \Vert\tensor C_2^{-1}(\bm y_{_H}-\rho\bm y_{_L}-\bm 1\mu_d) \Vert_2\leq
  \Vert\tensor C_2^{-1}\Vert_2\Vert\bm y_{_H}-\rho\bm y_{_L}-\bm 1\mu_d\Vert_2.\]
\end{proof}
Similar to Theorem~\ref{thm:mphik}, the upper bound includes
$\Vert\mathcal{L}\mu_{_d}\Vert$, where $\mu_{_d}$ is a constant. This term may
disappear for some choices of $\mathcal{L}$, e.g., when $\mathcal{L}$ is a
derivative operator. The norm $\Vert\mathcal{L}k_{_{d}}\Vert$ depends on the
selection of $k_{_d}$ and on the properties of $\mathcal{L}$. Therefore, it 
follows that carefully selecting a kernel for $Y_{_d}$ according to the 
properties of the system would result in a smaller error in preserving the
corresponding physical constraints. Of note, Kennedy and O'Hagan's framework can
be improved by using a more general nonlinear form of relation between low- and
high-fidelity data (see, e.g.,~\cite{perdikaris2017nonlinear}). Consequently, a
nonlinear relation between low- and high-fidelity data will change the upper 
bound in Theorem~\ref{thm:phicok}.

Furthermore, in Step 4 of Algorithm~\ref{algo:phicok}, if we set $\bm
y_{_L}=\mu_{_L}(\bm X_{_L})$, then the posterior mean will be
\begin{equation}
  \begin{aligned}
    \hat y(\bm x^*) & = \mu_{_H}(\bm x^*) + \big(\rho\bm c_{_L}(\bm x^*)^\trans,\bm c_{_H}(\bm x^*)^\trans\big)
    \begin{pmatrix}
      - \rho\tensor C_2^{-1}\big(\bm y_{_H}-\rho\bm\mu_{_L} -\bm 1\mu_d\big) \\
      \tensor C_2^{-1} \big(\bm y_{_H}-\rho\bm\mu_{_L} -\bm 1\mu_d\big)
    \end{pmatrix} \\
    & = \mu_{_H}(\bm x^*) + \big(\bm c_{_H}(\bm x^*)- \rho^2\bm c_{_L}(\bm
    x^*)\big)^\trans\tensor C_2^{-1} \big(\bm y_{_H}-\rho\bm y_{_L} -\bm 1\mu_d\big) \\
    & = \mu_{_H}(\bm x^*) + \bm c_{_d}(\bm x^*)\tensor C_d(\bm X_{_H}, \bm X_{_H})^{-1}(\bm y_d-\bm 1\mu_{_d}),
  \end{aligned}
\end{equation}
where $\bm c_{_d}(\bm x^*)=\bm c_{_H}(\bm x^*)-\rho^2\bm c_{_L}(\bm x^*)$.
In this form, the model information is included into $\mu_{_H}$ and $\bm y_{_d}$,
but not into the covariance matrix $\tensor C_{_d}$ explicitly. The error bound 
of preserving linear physical constraints can be derived from 
Theorem~\ref{thm:phicok} by setting $\tilde a_i=0$. In practice, we can modify 
Step 4 of Algorithm~\ref{algo:phicok} as identifying $\bm y_{_L}$ via iterating
over the set $\{Y^m\}_{m=1}^M\cup\{\bm\mu_{_L}\}$. 

Finally, we summarize this section by presenting in Figure~\ref{fig:spec} a 
spectrum of GP-based data-driven and physics-driven methods. In general, methods
closer to the physics-driven end result in predictions that better enforce 
physical constraints, while methods closer to the data-driven end may results in
more accurate posterior mean (e.g., smaller difference between posterior mean and
the ground truth). The standard kriging and cokring methods are at the data-driven 
end of the spectrum. PhIK, modified PhIK and CoPhIK assign different ``weights'' 
to data and model, and thus can be considered as placed along the spectrum 
between the data-driven and the physics-driven ends. The performance of the
physics-driven methods depends on the selection of stochastic model, and so as
the physics-informed methods. In general, we expect that adequate domain 
knowledge will lead to a good selection of stochastic models, which will result
in good accuracy of physics-informed GP methods. Moreover, as we point out in
the introduction, physical knowledge can also be incorporated directly into GP
kernels, especially for linear or linearized deterministic systems, which is an
alternative approach for data-model convergence under the GP framework.
\begin{figure}[!h]
  \centering
  \includegraphics[width=0.8\textwidth]{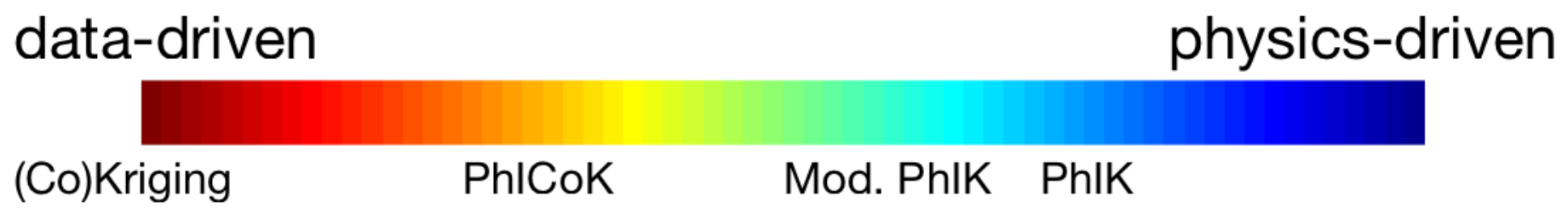}
  \caption{Spectrum of Physics-informed GP methods.}
  \label{fig:spec}
\end{figure}



%% file: results.tex
\section{Numerical examples}
\label{sec:numeric}

We present three numerical examples to demonstrate the performance of CoPhIK for
reconstructing spatially distributed states of physical systems from sparse
observations and numerical simulations. All numerical examples are 
two-dimensional in physical space. We compare CoPhIK against ordinary Kriging
(referred to as Kriging) and PhIK. For each example, the two-dimensional 
reference field and the reconstructed field (posterior mean) are discretized and
presented in matrix form. We denote by $\tensor F$ the reference field in matrix
form, and by $\tensor F_r$ the reconstructed field in matrix form. We employ the
RMSE $\hat s$, the point-wise difference $\tensor F_r-\tensor F$ and the 
relative error $\Vert\tensor F_r-\tensor F\Vert_F/\Vert \tensor F\Vert_F$ (where
$\Vert\cdot\Vert_F$ is the Frobenius norm) to compare the performance of 
different methods. We use the Gaussian kernel in Kriging and CoPhIK because the
fields in the examples are smooth. The kriging code employed in this work is 
based on the scripts in~\cite{forrester2008engineering}.
We also compare the CoPhIK, PhIK and Kriging performance for adaptively 
identifying new observations in order to reduce the uncertainty of the field
reconstruction. For this purpose we employ the active learning algorithm 
outlined in Appendix B.

\subsection{Branin function}
\label{subsec:branin}

Here, we reconstruct a function with unknown coefficients based on sparse
measurements of the function. Specifically, we consider the modified Branin
function~\cite{forrester2008engineering}
\begin{equation}
  \label{eq:branin_fun}
  f(\bm x) = a(\bar y-b\bar x^2+c\bar x-r)^2 + g(1-p)\cos(\bar x)+g + qx,
\end{equation}
where $\bm x=(x, y)^{\trans}$,
\begin{equation*}
  \bar x=15x-5,~\bar y=15y,~(x,y)\in \mathbb{D}=[0,1]\times [0,1],
\end{equation*}
and
\begin{equation*}
  a=1,~b=\frac{5.1}{4\pi^2},~c=\frac{5}{\pi},~r=6,~g=10,~p=\frac{1}{8\pi},~q=5.
\end{equation*}
The contours of $f$, together with eight randomly chosen observation locations 
are presented in Figure~\ref{fig:branin_truth}. The function $f$ is evaluated on
$41 \times 41$ uniform grids, so that the resulting discrete field $\tensor F$ 
is a $41\times 41$ matrix.
\begin{figure}[!h]
  \centering
  \includegraphics[width=0.4\textwidth]{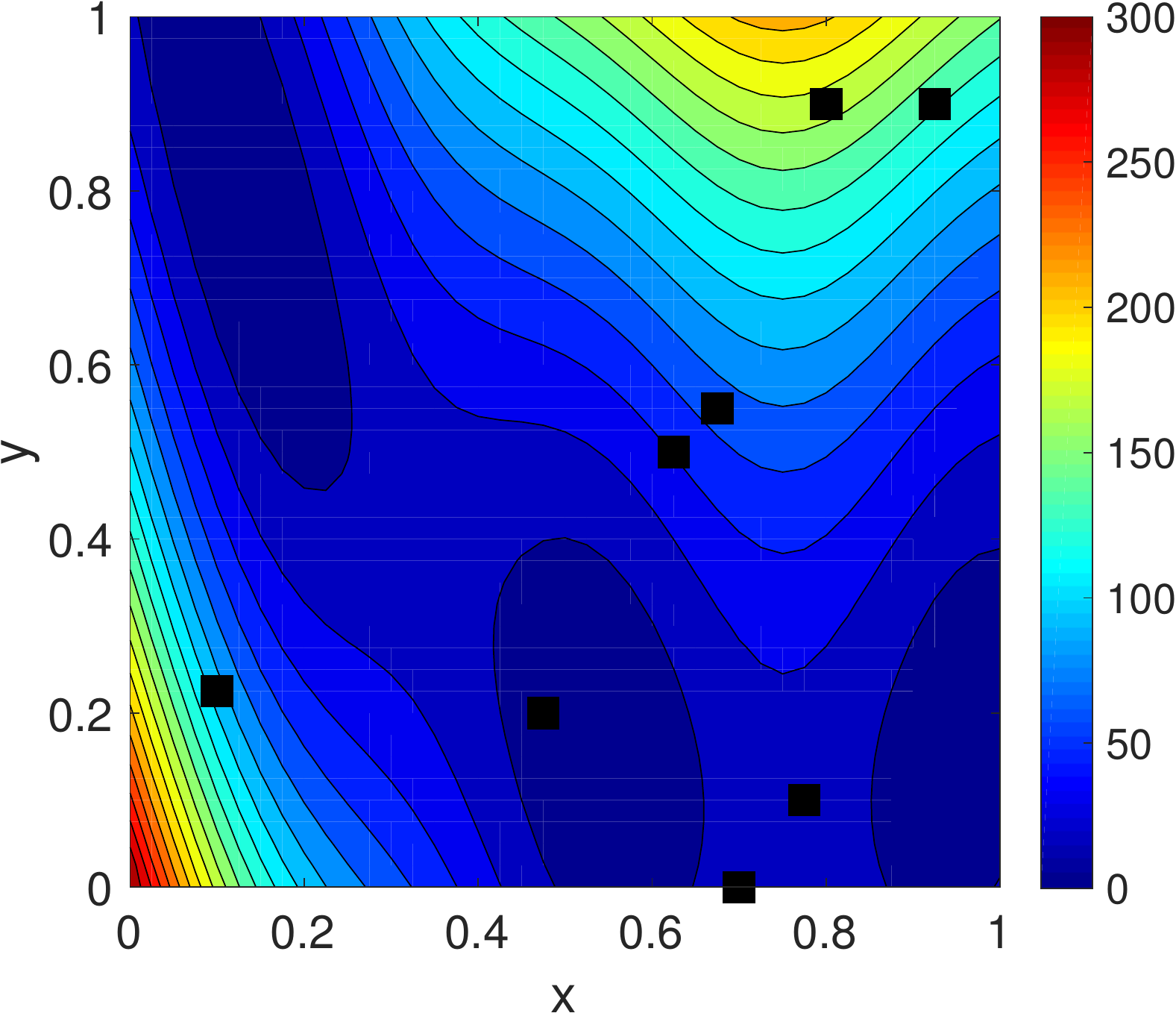}
  \caption{Contours of modified Branin function (on $41\times 41$ uniform grids)
  and locations of eight observations (black squares).}
  \label{fig:branin_truth}
\end{figure}
In this example, the stochastic model 
$u(\bm x;\omega): \mathbb{D}\times\Omega\rightarrow\mathbb{R}$ is obtained by
modifying the second $g$ in $f$, and treating the unknown coefficients $b$ and
$q$ in $f(\bm x)$ as random fields:
\begin{equation}
  \label{eq:branin_rf}
  \hat f(\bm x;\omega) = a(\bar y-\hat b(\bm x;\omega)\bar x^2+c\bar x-r)^2 
  + g(1-p)\cos(\bar x)+ \hat g + \hat q(\bm x;\omega)x,
\end{equation}
where $\hat g=20$, 
\begin{equation*}
  \begin{aligned}
    \hat b(\bm x;\omega) & = b\left\{0.9+\dfrac{0.2}{\pi}\sum_{i=1}^3\left[ \dfrac{1}{4i-1}\sin((2i-0.5)\pi x)\xi_{2i-1}(\omega) + \dfrac{1}{4i+1}\sin((2i+0.5)\pi y)\xi_{2i}(\omega)\right]\right\}, \\
    \hat q(\bm x;\omega) & = q\left\{1.0+\dfrac{0.6}{\pi}\sum_{i=1}^3\left[ \dfrac{1}{4i-3}\cos((2i-1.5)\pi x)\xi_{2i+5}(\omega) + \dfrac{1}{4i-1}\cos((2i-0.5)\pi y)\xi_{2i+6}(\omega)\right]\right\},
  \end{aligned}
\end{equation*}
and $\{\xi_i(\omega)\}_{i=1}^{12}$ are iid Gaussian random variables with zero
mean and unit variance. This stochastic model includes unkown ``physics" 
($\hat b$ and $\hat q$) and incorrect ``physics" ($\hat g$). We compute $M=300$
realizations of $u(\bm x;\omega)$, denoted as $\{\hat{\tensor F}^m\}_{m=1}^M$, 
by generating $M=300$ samples of $\{\xi_i(\omega)\}_{i=1}^{12}$ and 
evaluating $u(\bm x;\omega)$ on the $41\times 41$ uniform grids for each of
them. Of note, function $f$ is not a realization of $u(\bm x;\omega)$.

Figure~\ref{fig:branin_krig_mc} compares the purely data-driven reconstruction
(i.e., Kriging) and purely ``physics''-based reconstruction (i.e., mean and
variance obtained from the stochastic model Eq.~\eqref{eq:branin_rf} without
conditioning on data). The first row in Figure~\ref{fig:branin_krig_mc} presents
Kriging results obtained using the eight observations shown in 
Figure~\ref{fig:branin_truth}. The second row shows the ensemble mean (denoted
as $\mu(\hat{\tensor F}^m)$), standard deviation (denoted as 
$\sigma(\hat{\tensor F}^m)$) of $\{\hat{\tensor F}^m\}_{m=1}^M$, and
$\mu(\hat{\tensor F}^m)-\tensor F$. The relative error is more than $50\%$ for
Kriging and $19\%$ for the ensemble mean. In kriging, $\hat{s}$ is large in the
upper left subdomain where no observations are available. On the other hand, the
standard deviation of the realizations is large in the upper right region 
because of the randomness in the ``physical'' model. Similarly, 
$|\tensor F_r-\tensor F|$ is large in the upper left region, while
$|\mu(\hat{\tensor F}^m)-\tensor F|$ is large in the upper right region. These
results demonstrate that the data-driven and physics-based methods have
significantly different accuracy in reconstructing the entire field and result
in different estimates of uncertainty of the reconstruction.
\begin{figure}[!h]
  \centering%
  \begin{subfigure}[b]{0.23\textwidth}
    \centering%
    \includegraphics[height=\textwidth]{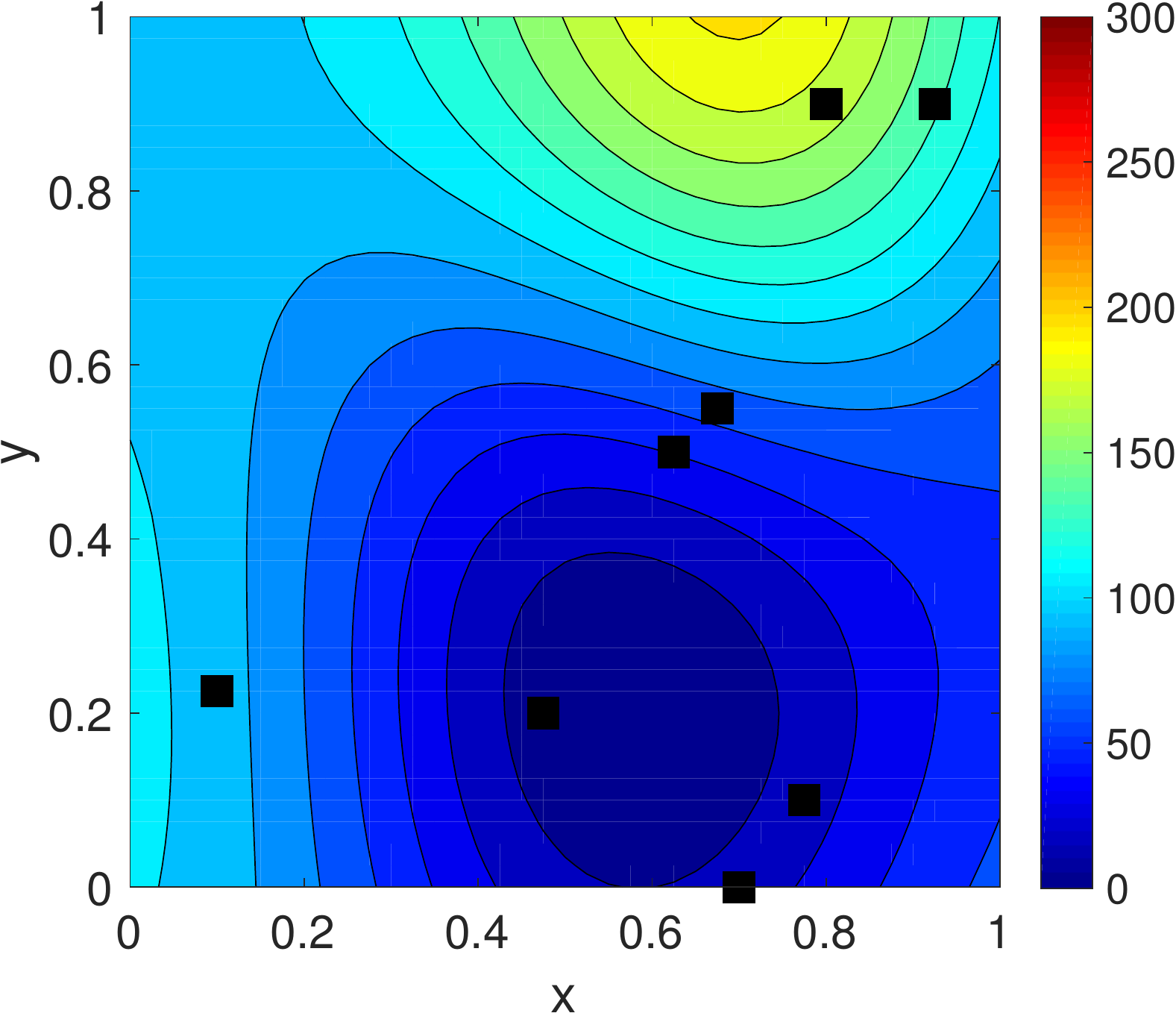}
    \caption{Kriging $\tensor F_r$}
  \end{subfigure}\qquad
  \begin{subfigure}[b]{0.23\textwidth}
    \centering%
    \includegraphics[height=\textwidth]{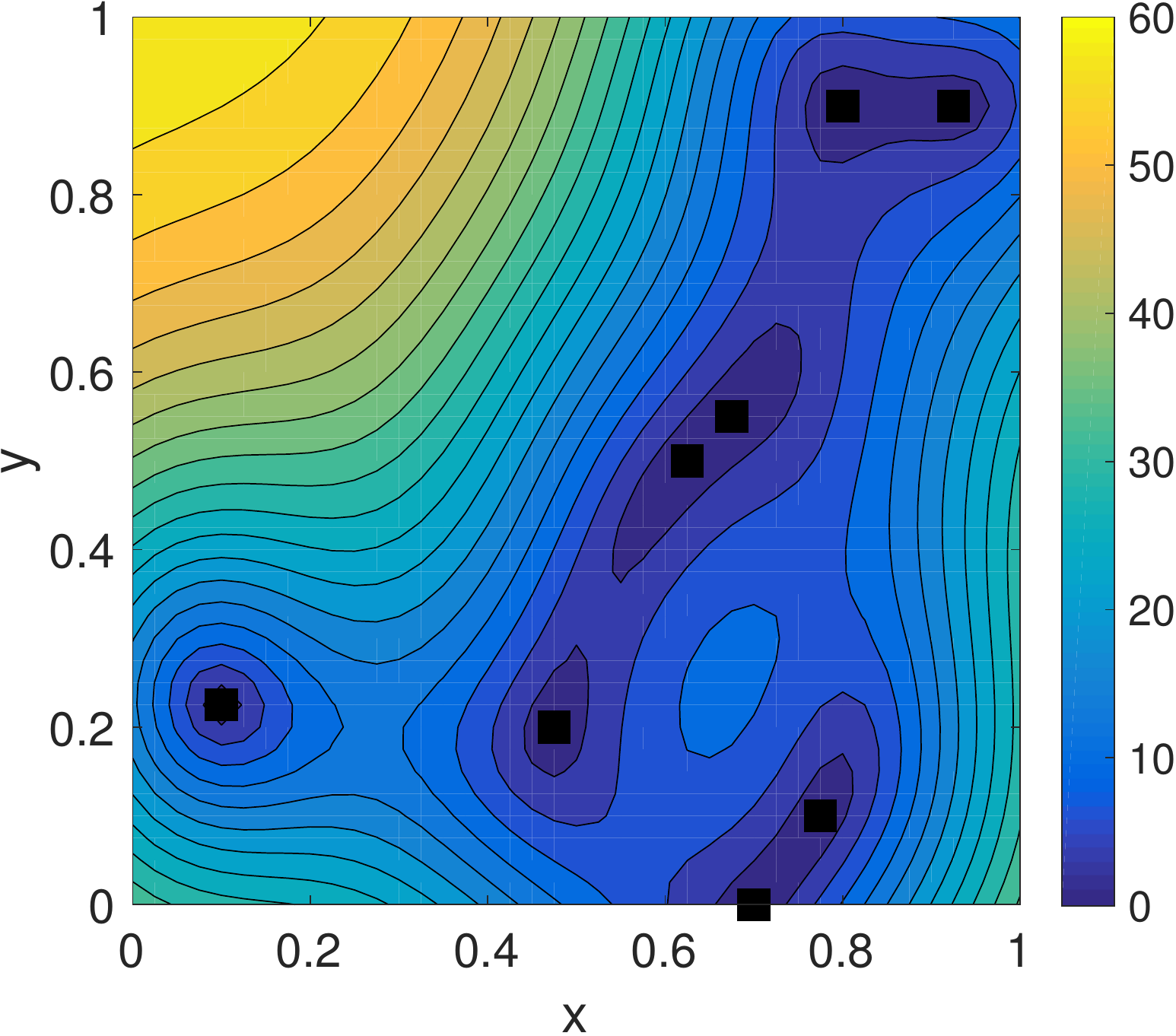}
    \caption{Kriging $\hat s$}
  \end{subfigure}\qquad
  \begin{subfigure}[b]{0.23\textwidth}
    \centering%
    \includegraphics[height=\textwidth]{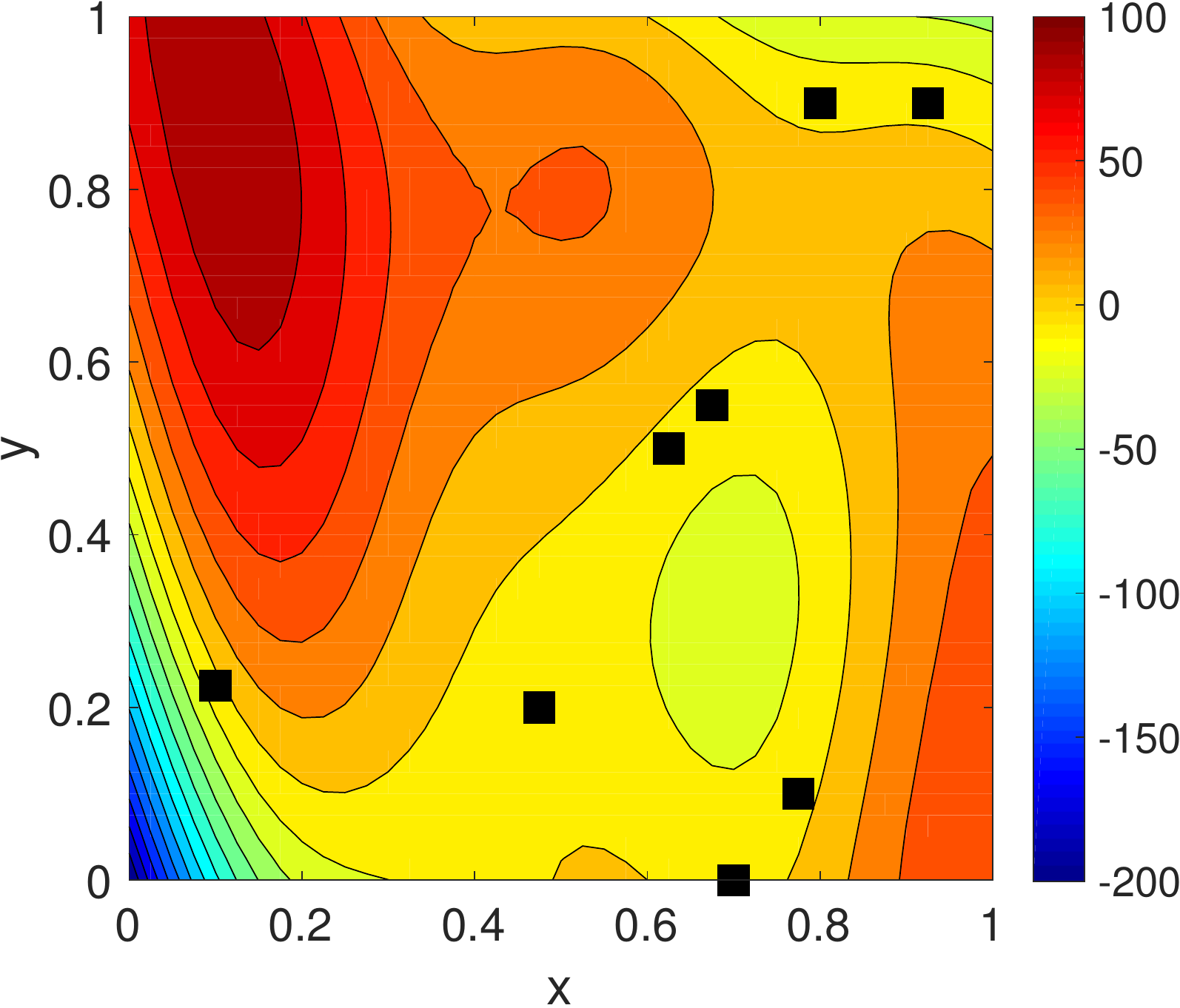}
    \caption{Kriging $\tensor F_r-\tensor F$}
  \end{subfigure}\\
  \begin{subfigure}[b]{0.23\textwidth}
    \centering%
    \includegraphics[height=\textwidth]{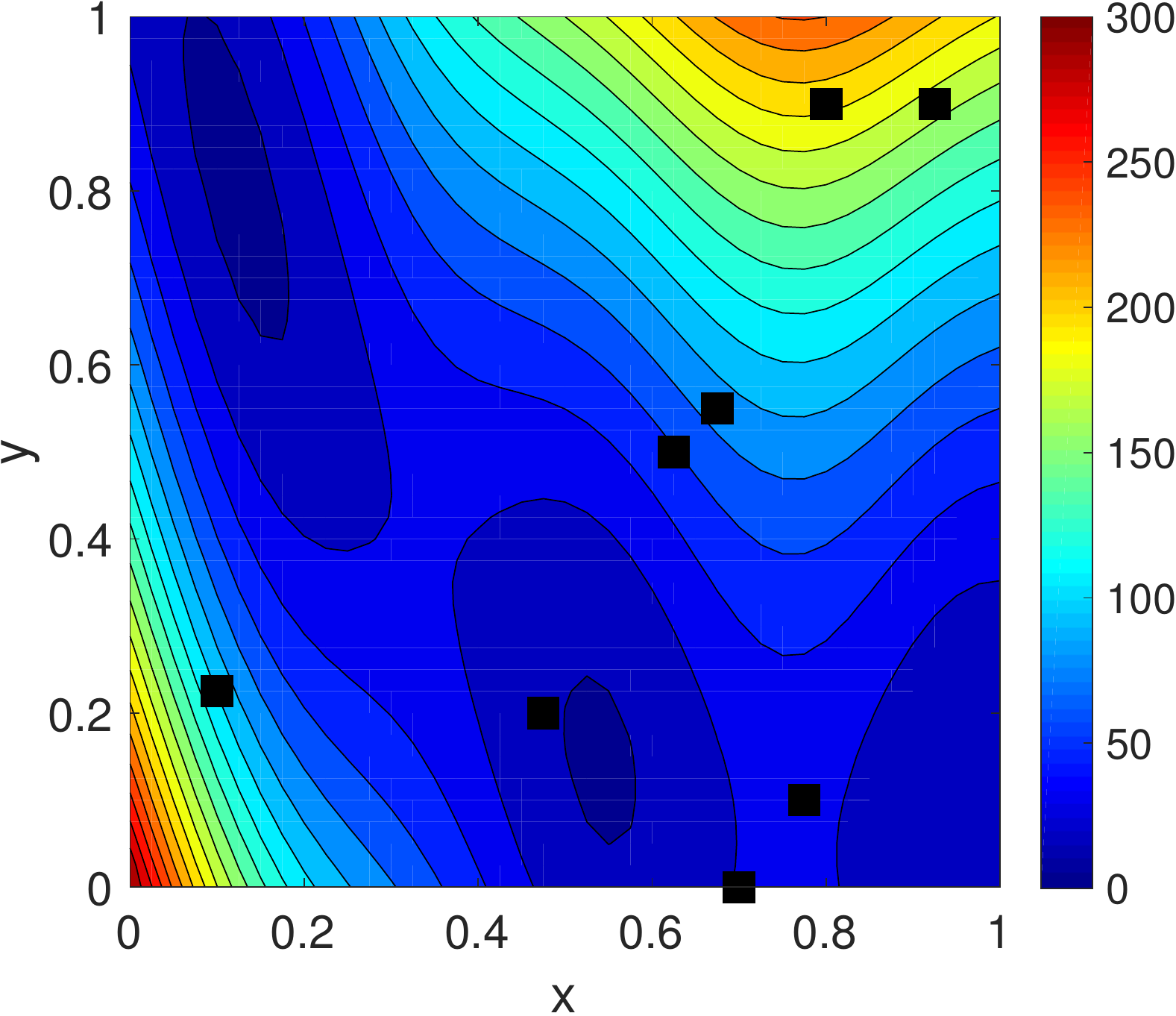}
    \caption{$\mu(\hat{\tensor F}^m)$}
  \end{subfigure}\qquad
  \begin{subfigure}[b]{0.23\textwidth}
    \centering%
    \includegraphics[height=\textwidth]{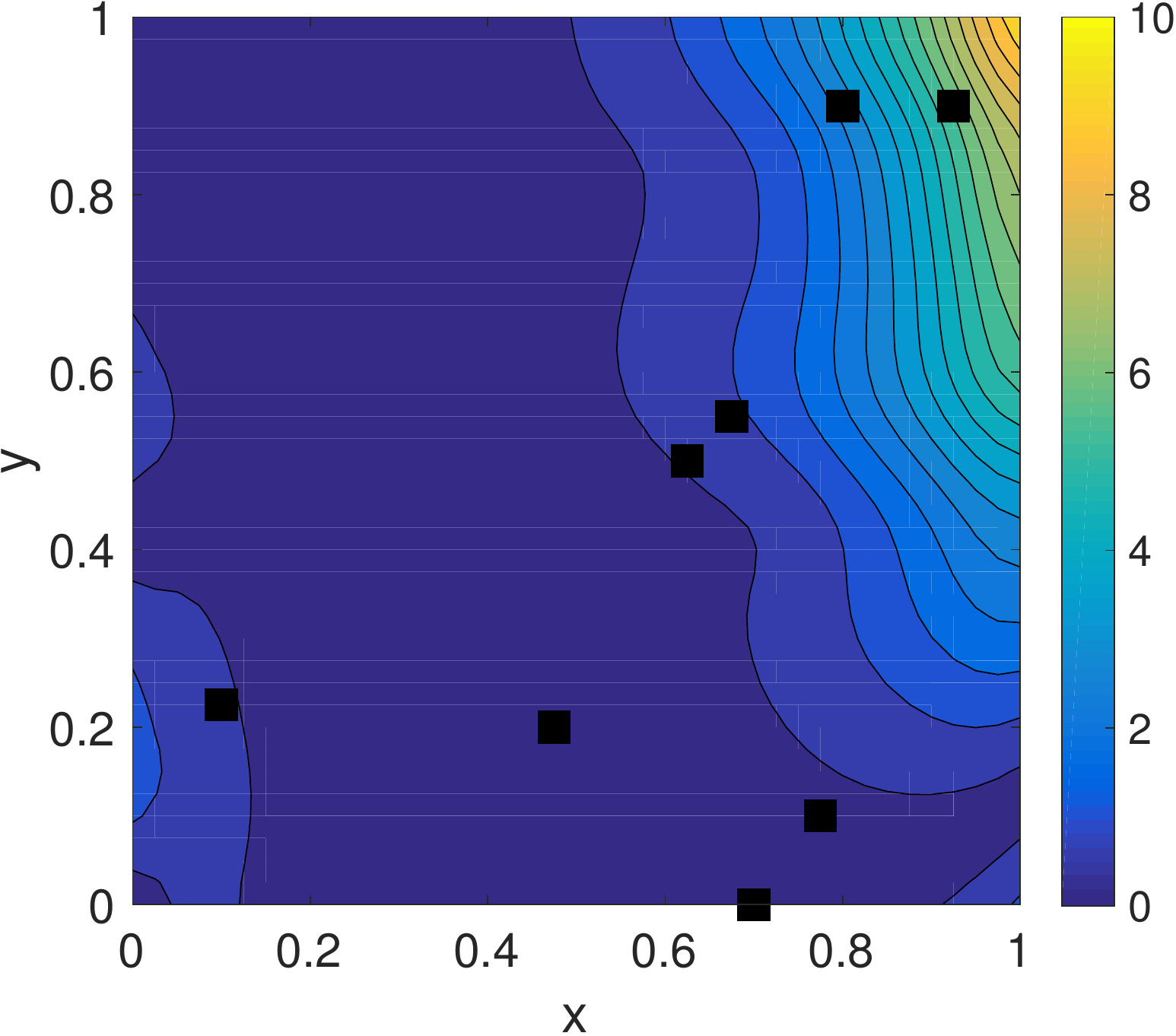}
    \caption{$\sigma(\hat{\tensor F}^m)$}
  \end{subfigure}\qquad
  \begin{subfigure}[b]{0.23\textwidth}
    \includegraphics[height=\textwidth]{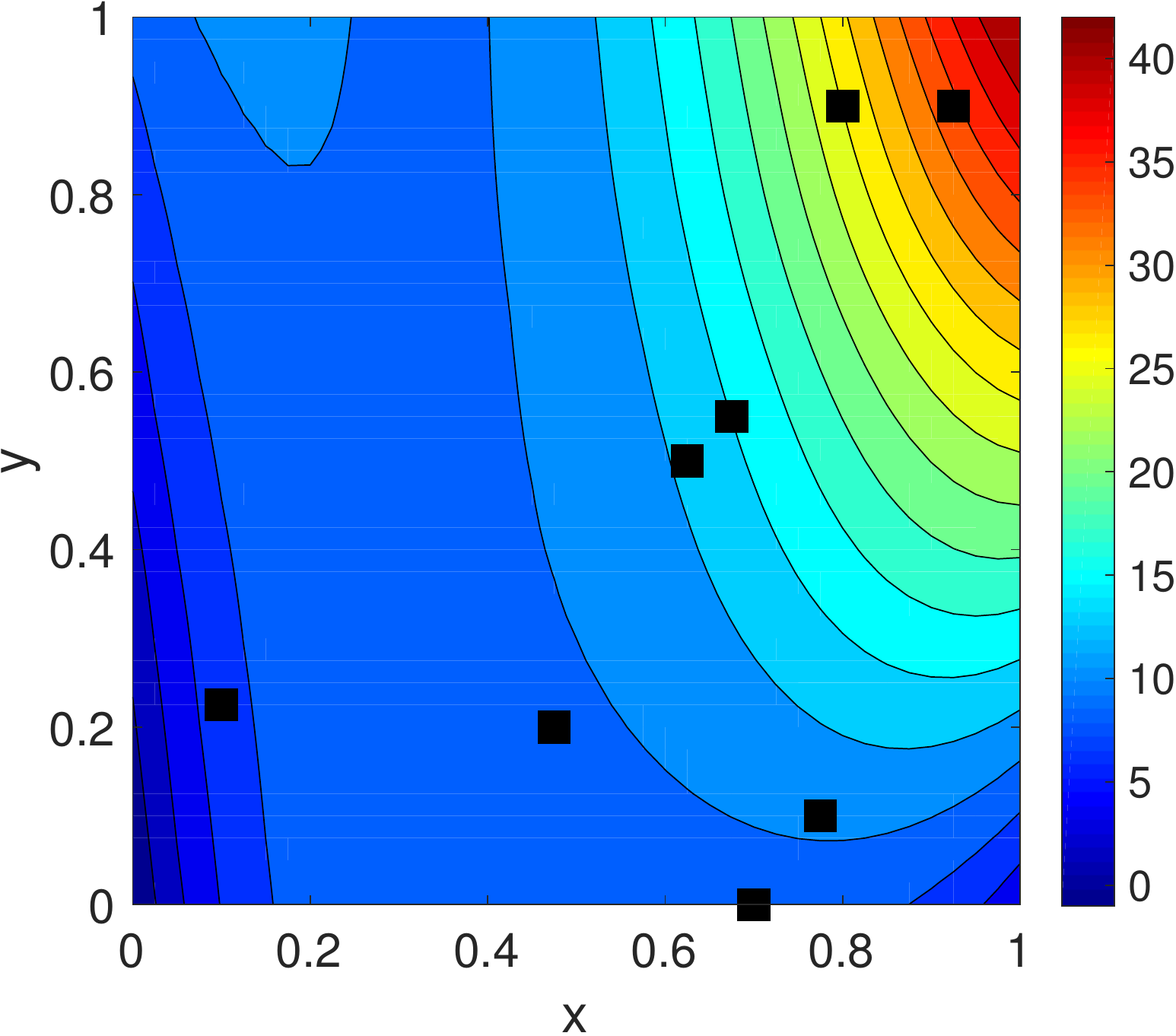}
    \caption{$\mu(\hat{\tensor F}^m)-\tensor F$}
  \end{subfigure}
  \caption{Reconstruction of the modified Branin function by Kriging (first row) and statistics of the ensemble $\{\hat{\tensor F}^m\}_{m=1}^M$ (second row).}
  \label{fig:branin_krig_mc}
\end{figure}

Figure~\ref{fig:branin_phik} presents results obtained using PhIK and CoPhIK.
In this case, CoPhIK outperforms PhIK in the accuracy, and both methods are
significantly more accurate than Kriging. However, PhIK shows smaller $\hat s$
than CoPhIK, and $\hat s$ in both PhIK and CoPhIK is significantly smaller than
in Kriging. This is because the prior covariance of PhIK is decided by the
realizations, and it doesn't account for the observation. Therefore, the
uncertainty of the PhIK result ($\hat s$ of PhIK) is bounded by the uncertainty
of the stochastic model ($\sigma(\hat{\bm F}^m)$). Also, the pattern of $\hat s$
in PhIK is similar to that of $\sigma(\hat{\tensor F}^m)$ in 
Figure~\ref{fig:branin_krig_mc}(e), i.e., it is large in the right subdomain and
small in the left subdomain. On the other hand, CoPhIK incorporates the 
observation in the posterior mean and kernel as we illustrate in 
Section~\ref{sec:method}, and its $\hat s$ pattern is more similar to Kriging's
in Figure~\ref{fig:branin_krig_mc}(b) as we use the Gaussian kernel for both
kriging and CoPhIK.
\begin{figure}[!h]
  \centering%
  \begin{subfigure}[b]{0.23\textwidth}
    \centering%
    \includegraphics[height=\textwidth]{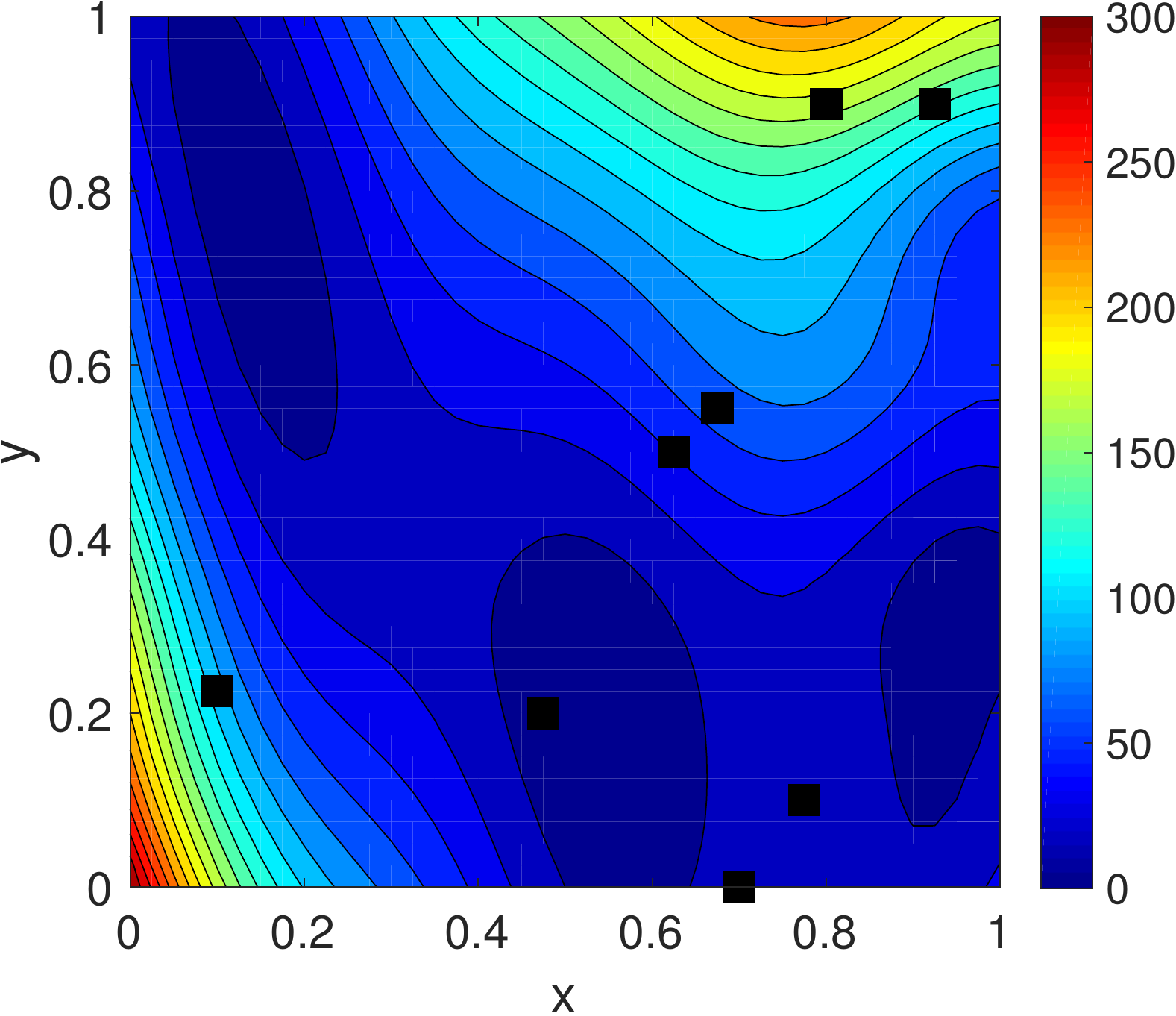}
    \caption{PhIK $\tensor F_r$}
  \end{subfigure}\qquad
  \begin{subfigure}[b]{0.23\textwidth}
    \centering%
    \includegraphics[height=\textwidth]{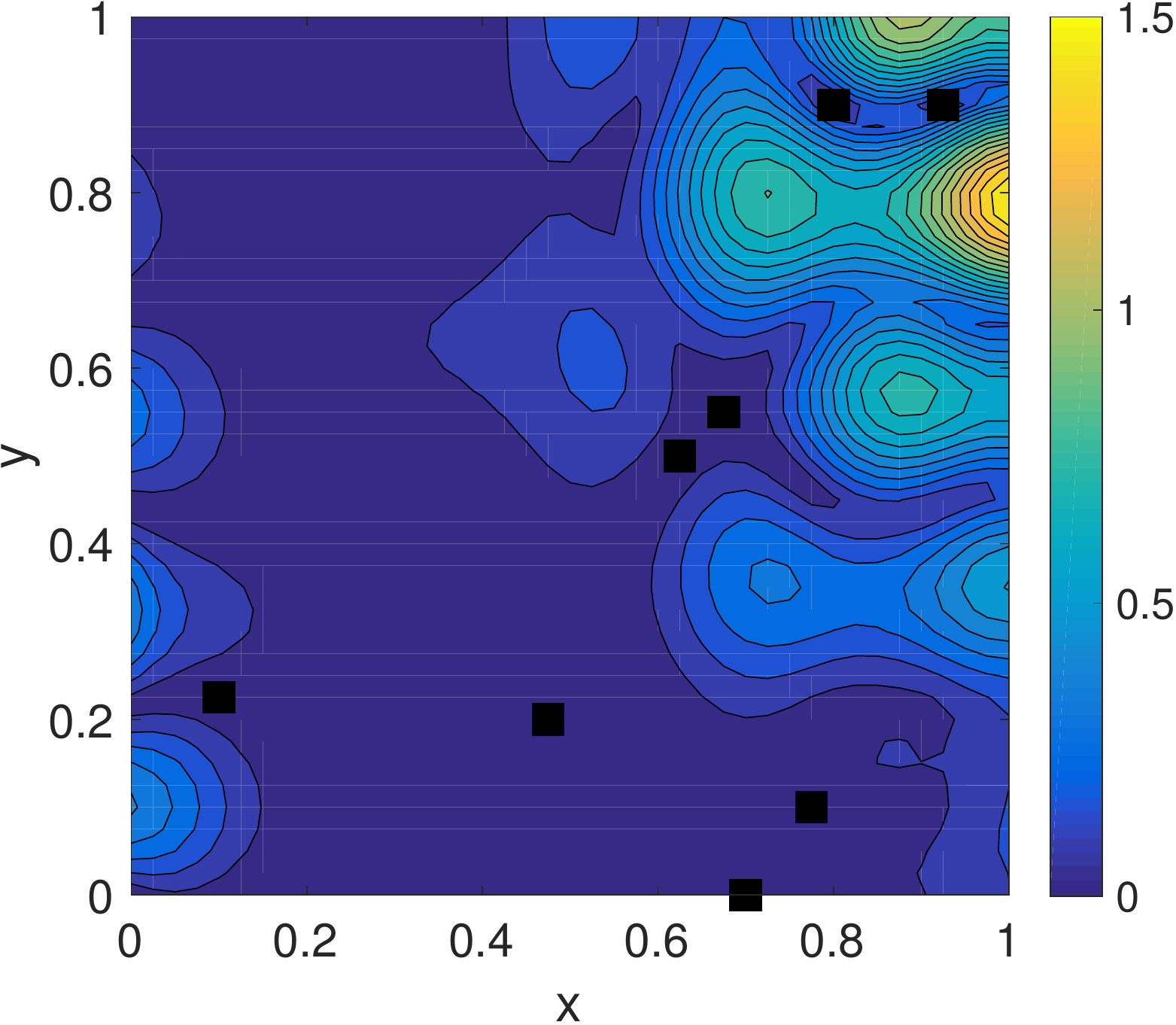}
    \caption{PhIK $\hat s$}
  \end{subfigure}\qquad
  \begin{subfigure}[b]{0.23\textwidth}
    \centering%
    \includegraphics[height=\textwidth]{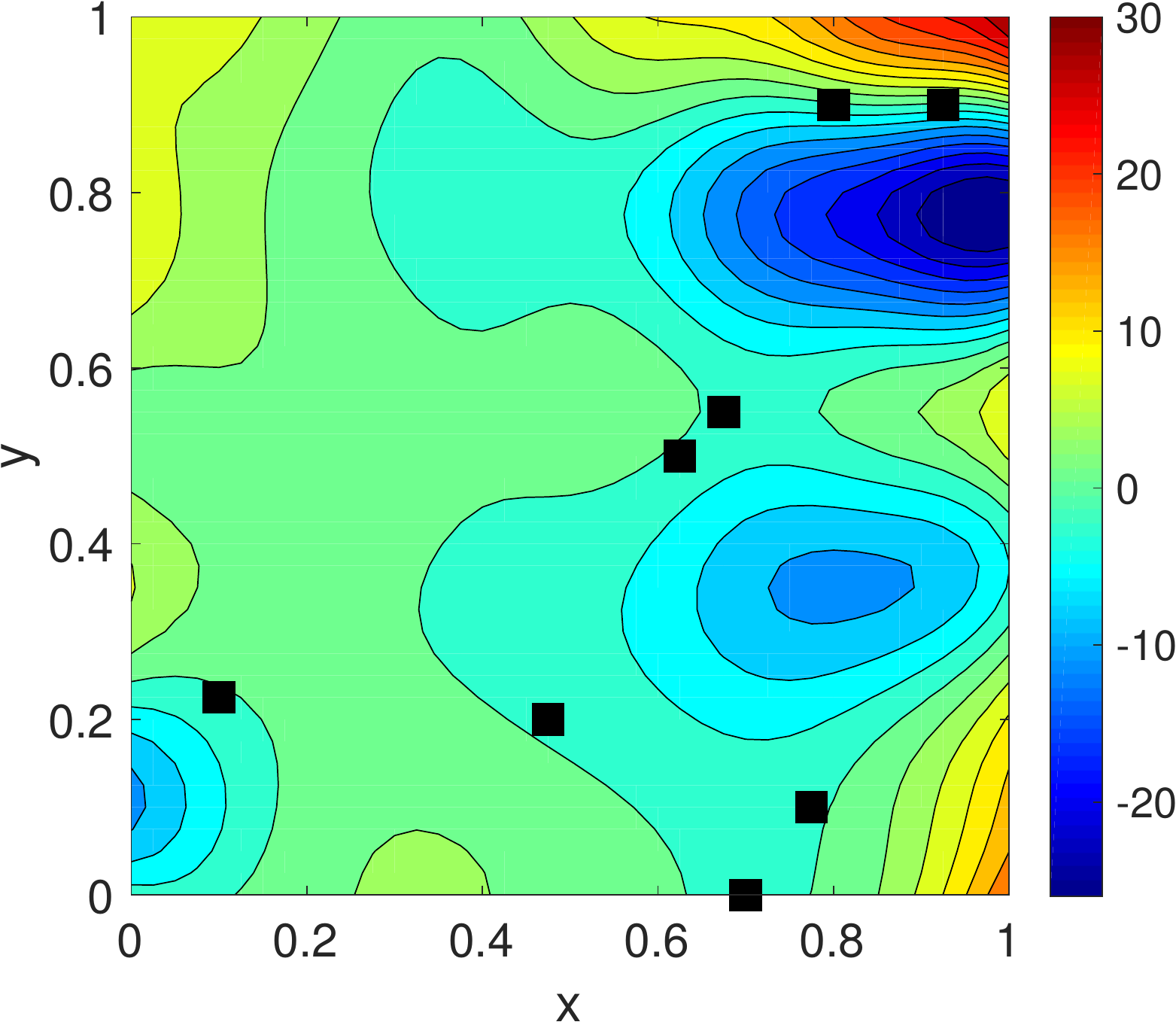}
    \caption{PhIK $\tensor F_r-\tensor F$}
  \end{subfigure}\\
  \begin{subfigure}[b]{0.23\textwidth}
    \centering%
    \includegraphics[height=\textwidth]{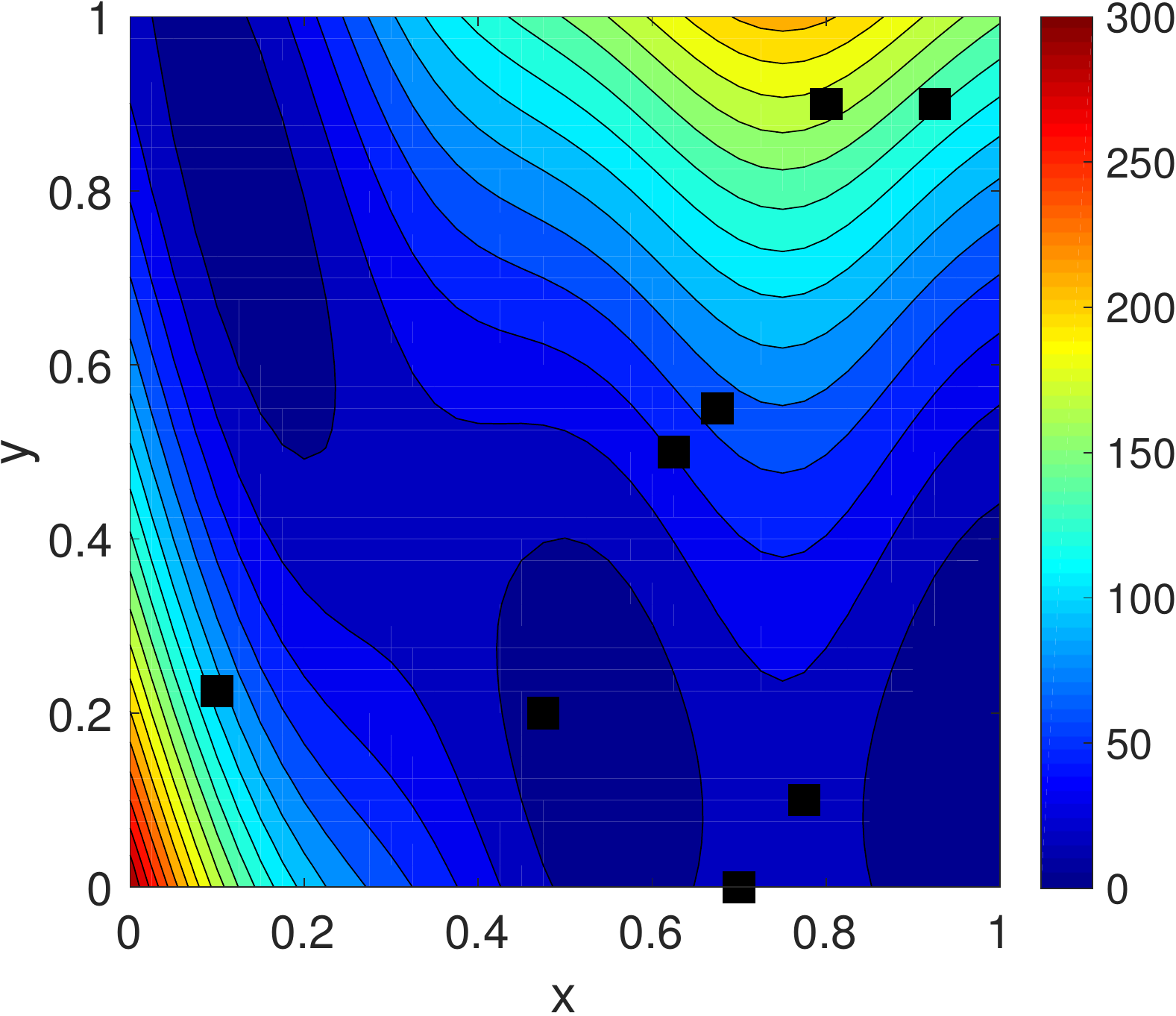}
    \caption{CoPhIK $\tensor F_r$}
  \end{subfigure}\qquad
  \begin{subfigure}[b]{0.23\textwidth}
    \centering%
    \includegraphics[height=\textwidth]{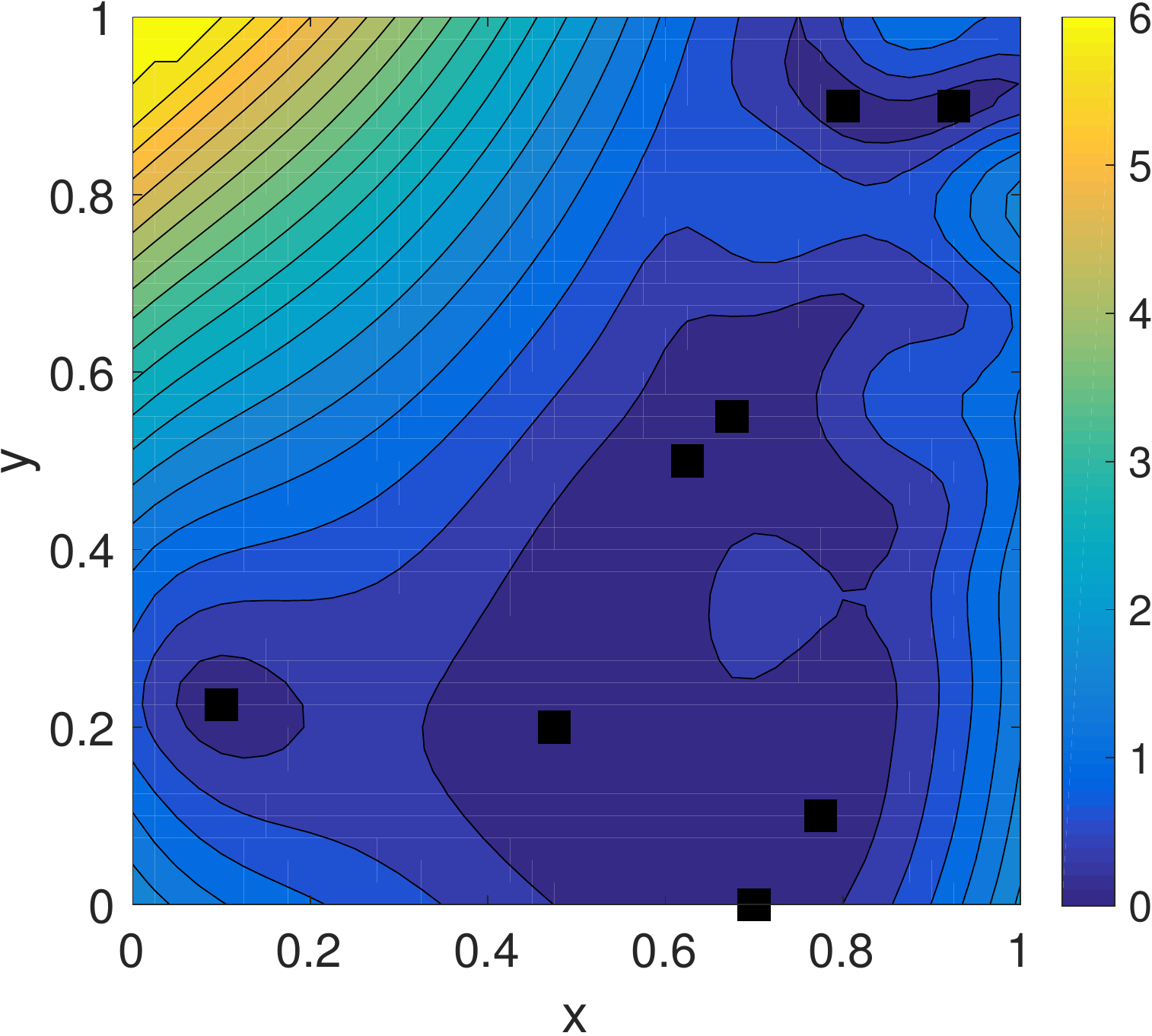}
    \caption{CoPhIK $\hat s$}
  \end{subfigure}\qquad
  \begin{subfigure}[b]{0.23\textwidth}
    \centering%
    \includegraphics[height=\textwidth]{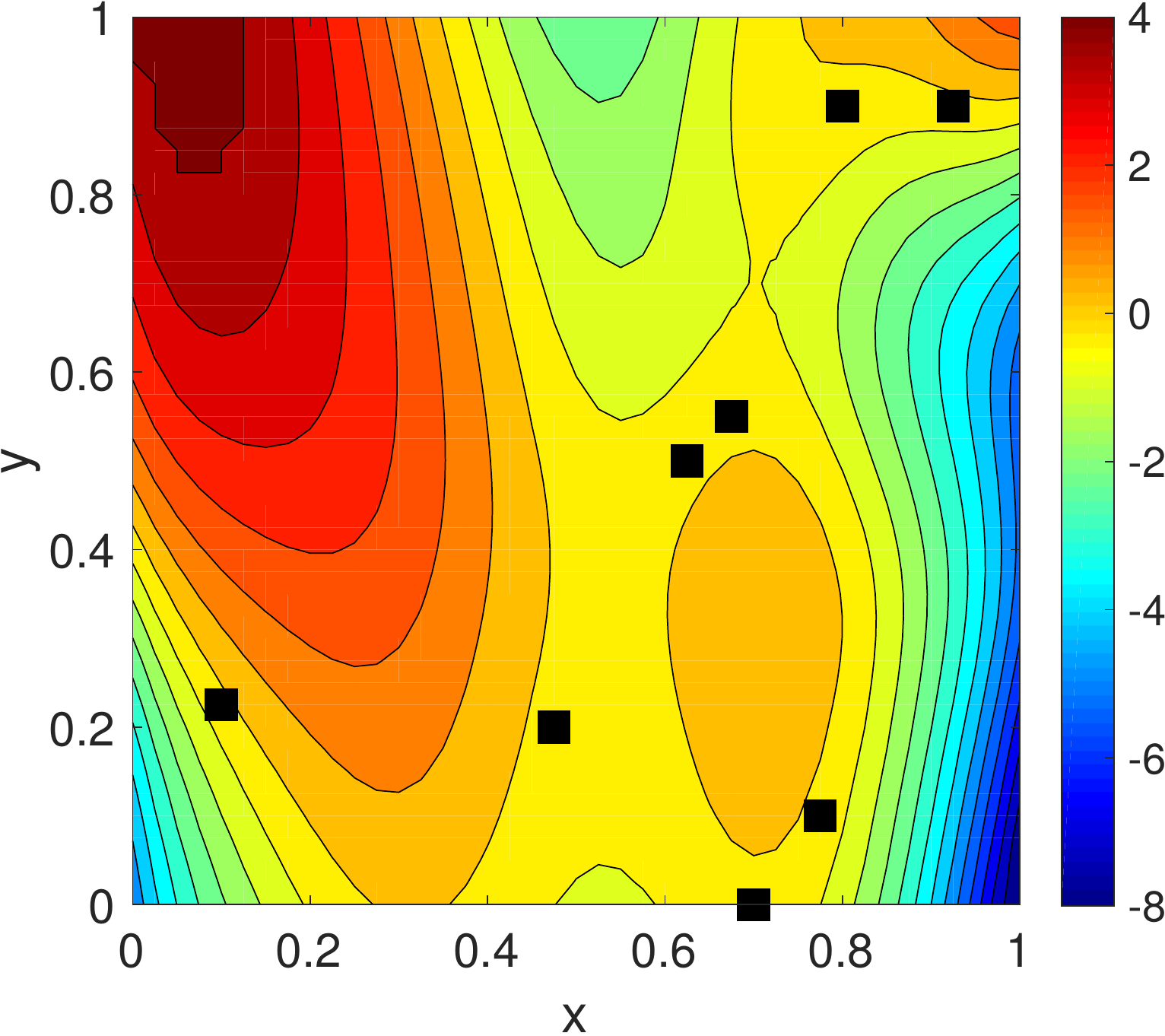}
    \caption{CoPhIK $\tensor F_r-\tensor F$}
  \end{subfigure}
  \caption{Reconstruction of the modified Branin function by PhIK (first row) and CoPhIK (second row).} 
  \label{fig:branin_phik}
\end{figure}

We use Algorithm~\ref{algo:act} in combination with kriging, PhIK, and CoPhIK to
perform active learning by adding one by one new observations of $f$ at the 
global maximum of $\hat s$. Figure~\ref{fig:branin_act} presents the 
reconstructions of the modified Branin function with eight new observations 
added using active learning. In this figure, the first, second and third row
corresponds to Kriging, PhIK and CoPhIK, respectively. The initial eight
observations are marked by squares, and added observations are marked by stars.
By comparing to Figures~\ref{fig:branin_krig_mc} and~\ref{fig:branin_phik} it 
can be seen that reconstruction accuracy increases as more observations are 
added, and the uncertainty in the reconstruction is reduced. It can also be 
seen that the additional observation locations identified by CoPhIK are similar
to that of kriging. In contrast, most of the additional observations identified
by PhIK are on the right boundary.
\begin{figure}[!h]
  \centering%
  \begin{subfigure}[b]{0.22\textwidth}
    \centering%
    \includegraphics[height=\textwidth]{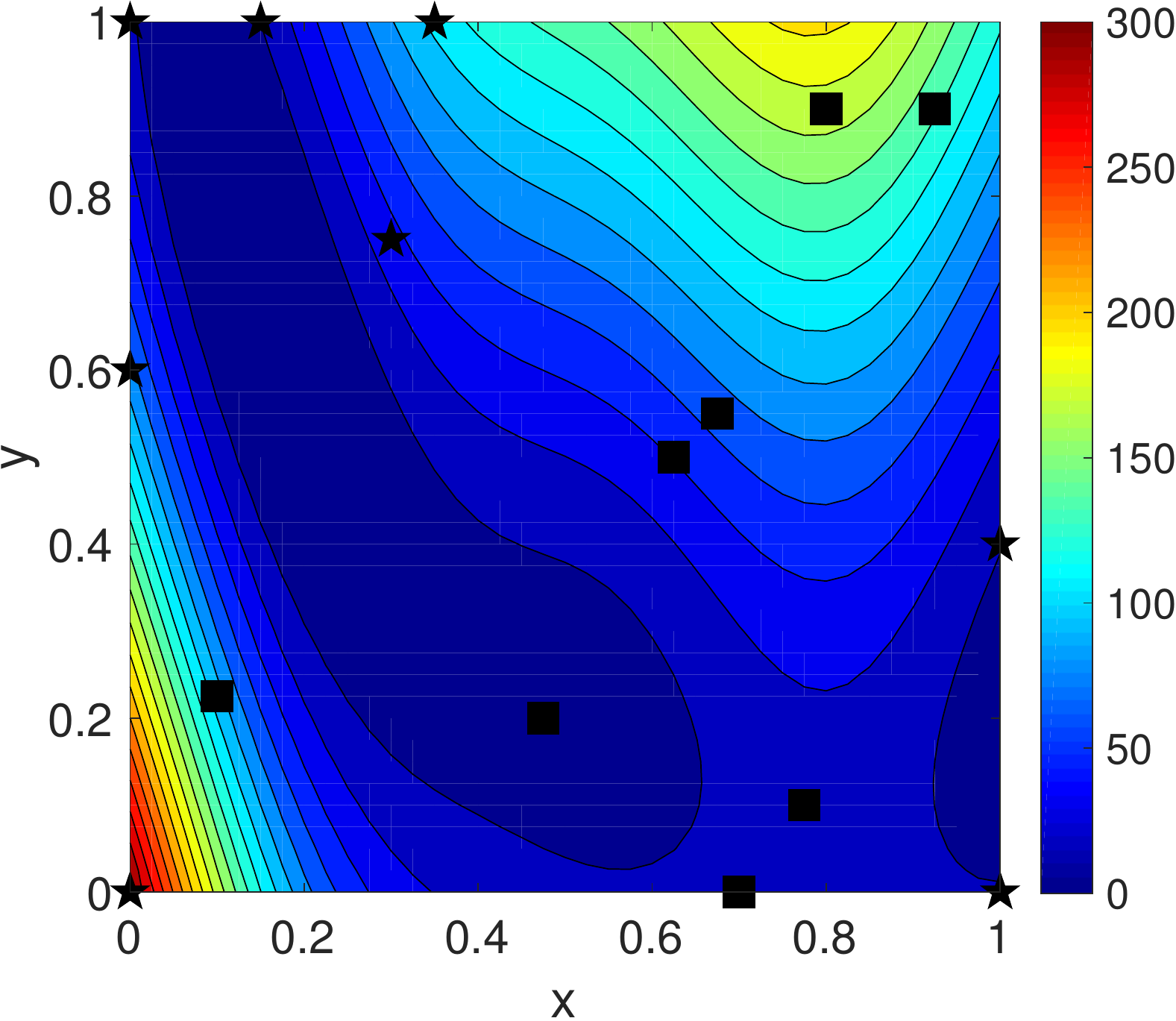}
    \caption{Kriging $\tensor F_r$}
  \end{subfigure}\qquad
  \begin{subfigure}[b]{0.22\textwidth}
    \includegraphics[height=\textwidth]{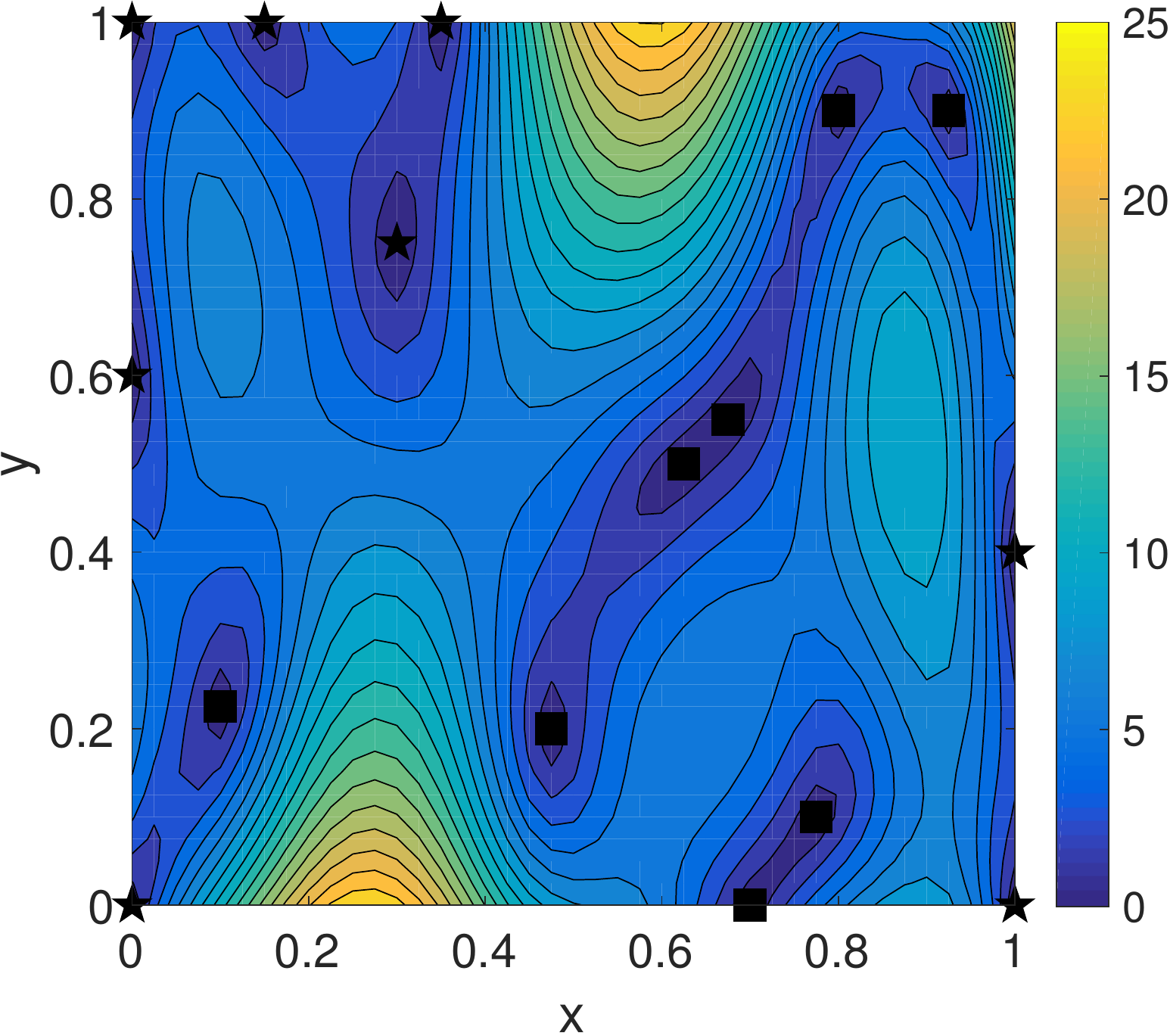}
    \caption{Kriging $\hat s$}
  \end{subfigure}\qquad
  \begin{subfigure}[b]{0.22\textwidth}
    \includegraphics[height=\textwidth]{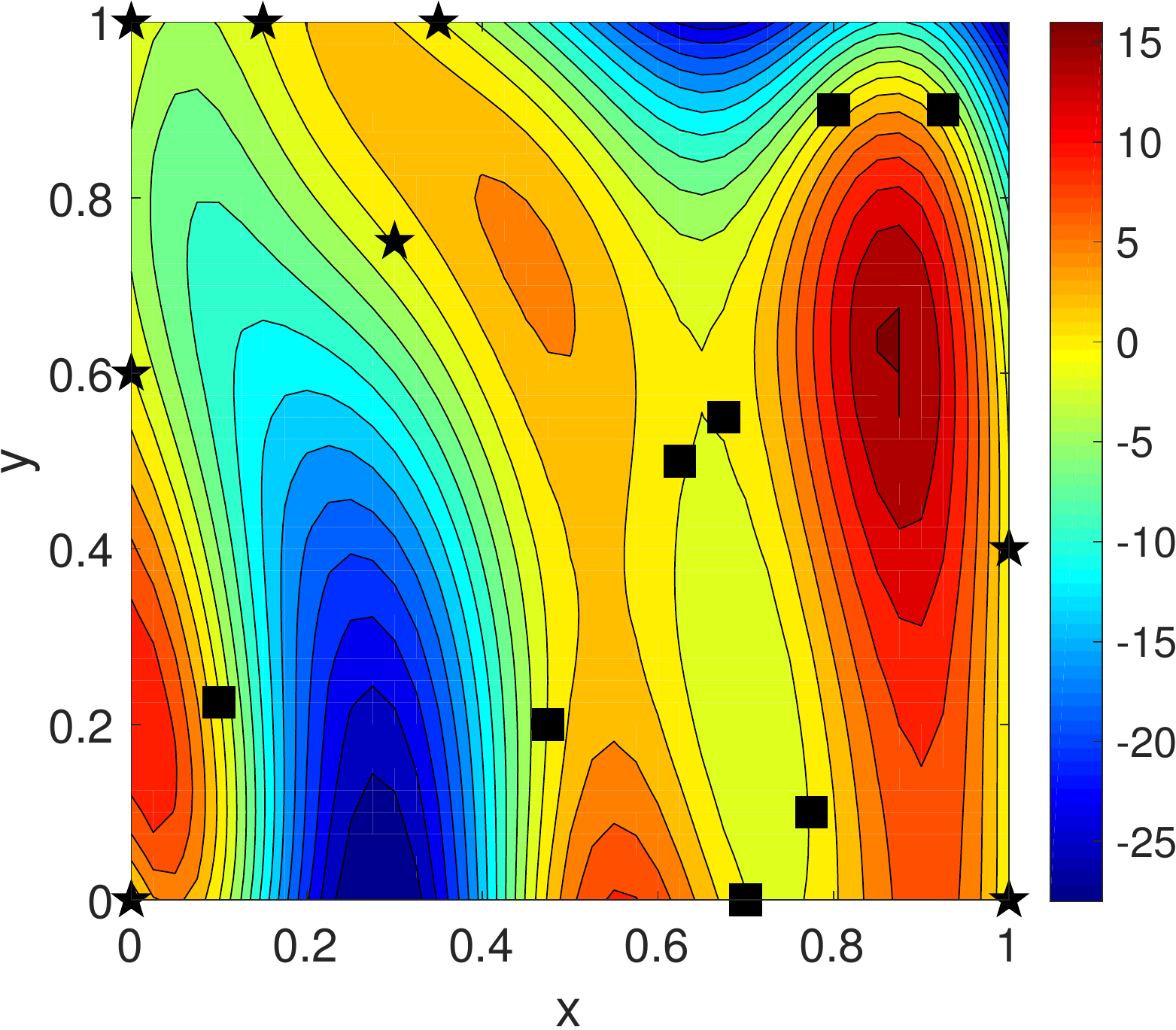}
    \caption{Kriging $\tensor F_r-\tensor F$}
  \end{subfigure}\\
  \begin{subfigure}[b]{0.22\textwidth}
    \includegraphics[height=\textwidth]{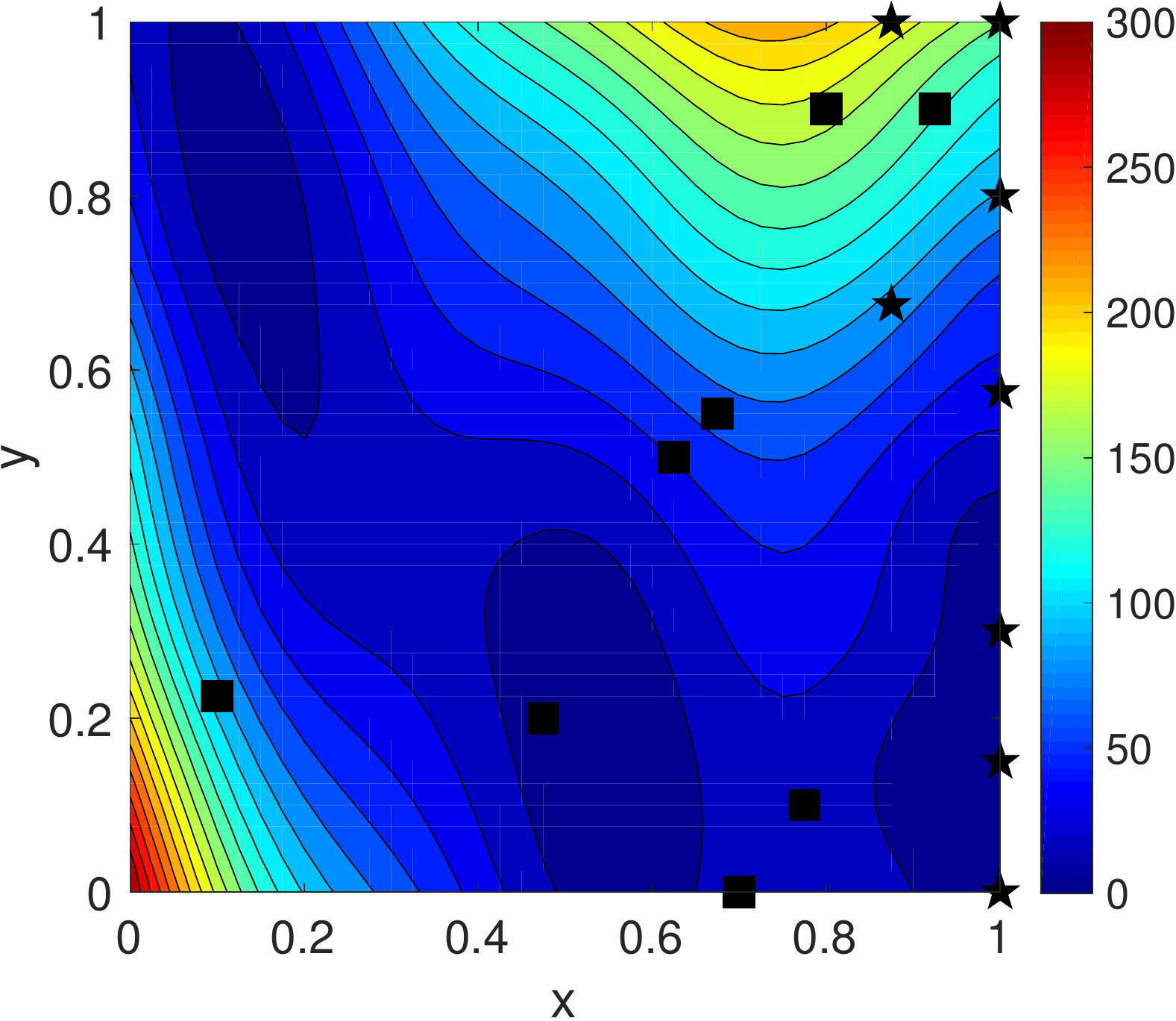}
    \caption{PhIK $\tensor F_r$}
  \end{subfigure}\qquad
  \begin{subfigure}[b]{0.22\textwidth}
    \includegraphics[height=\textwidth]{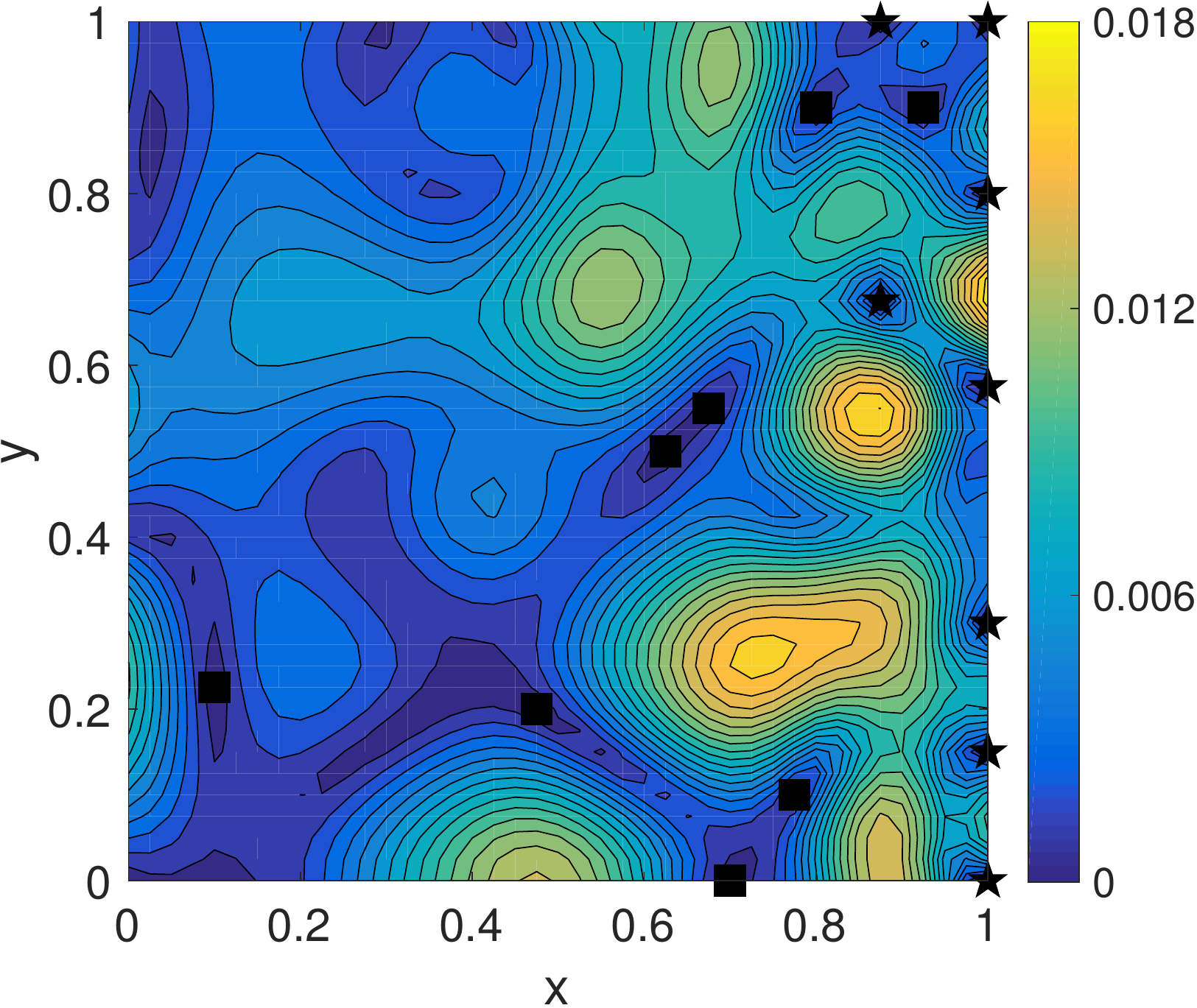}
    \caption{PhIK $\hat s$}
  \end{subfigure}\qquad
  \begin{subfigure}[b]{0.22\textwidth}
    \includegraphics[height=\textwidth]{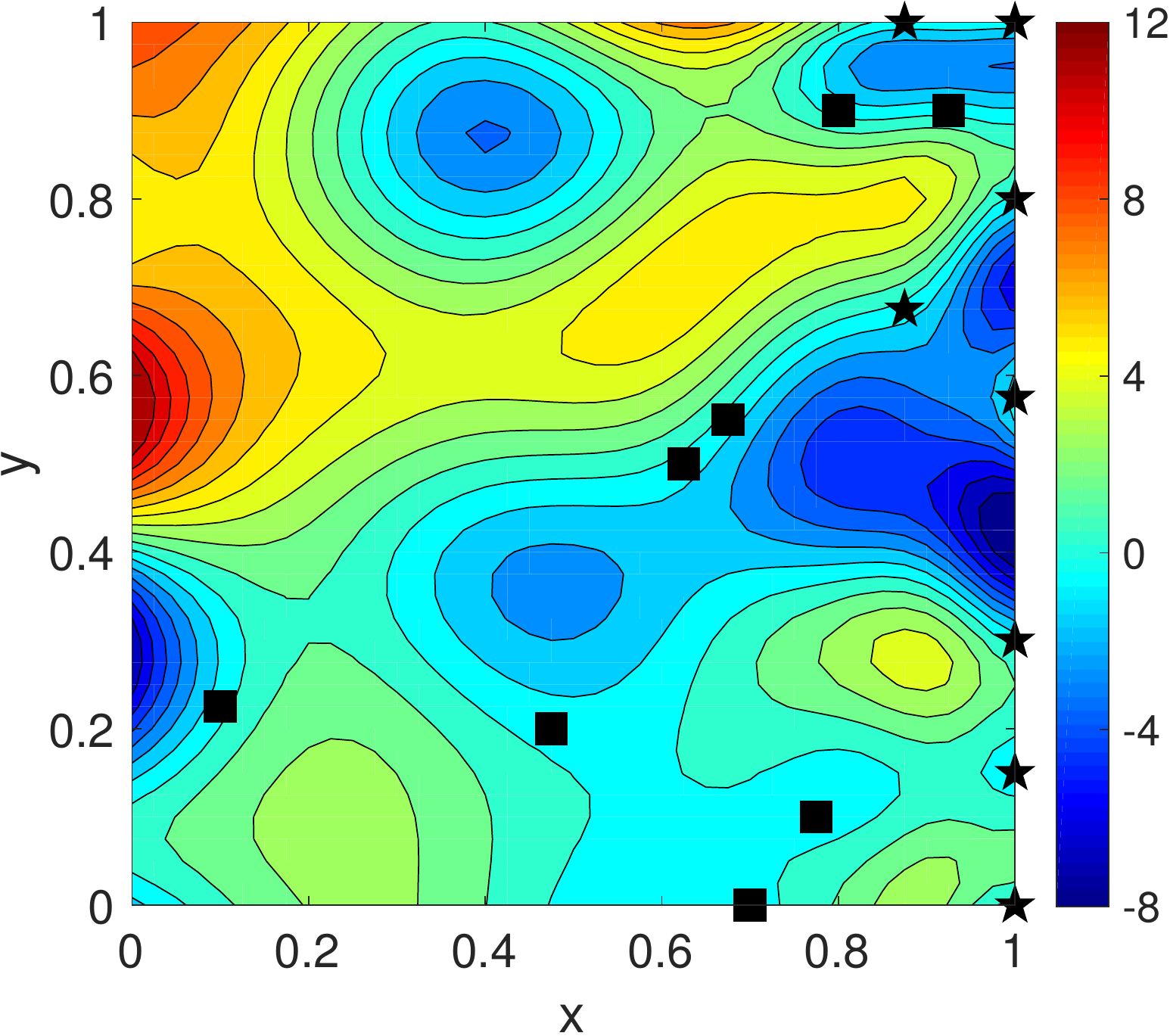}
    \caption{PhIK $\tensor F_r-\tensor F$}
  \end{subfigure}\\
  \begin{subfigure}[b]{0.22\textwidth}
    \includegraphics[height=\textwidth]{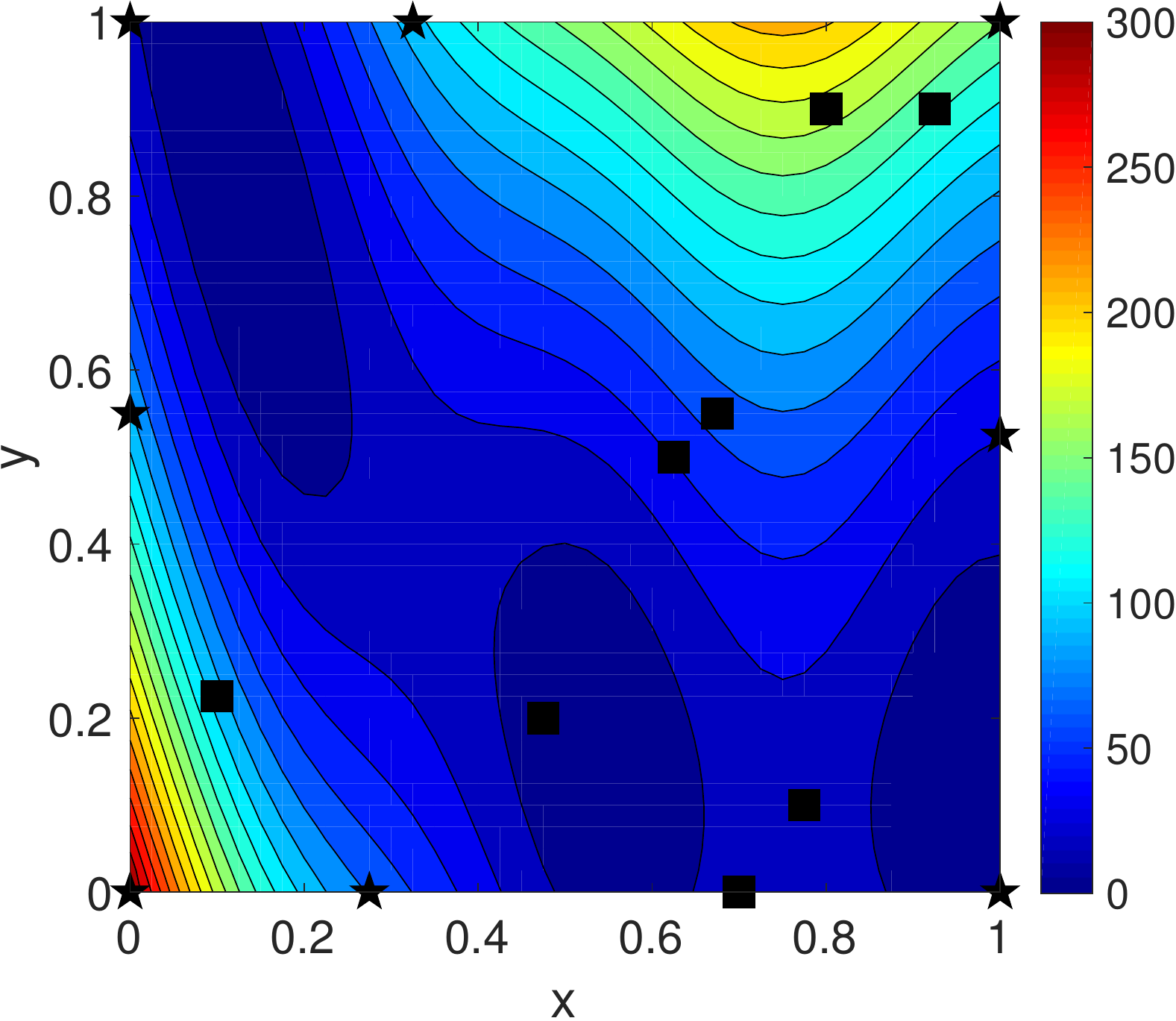}
    \caption{CoPhIK $\tensor F_r$}
  \end{subfigure}\qquad
  \begin{subfigure}[b]{0.22\textwidth}
    \includegraphics[height=\textwidth]{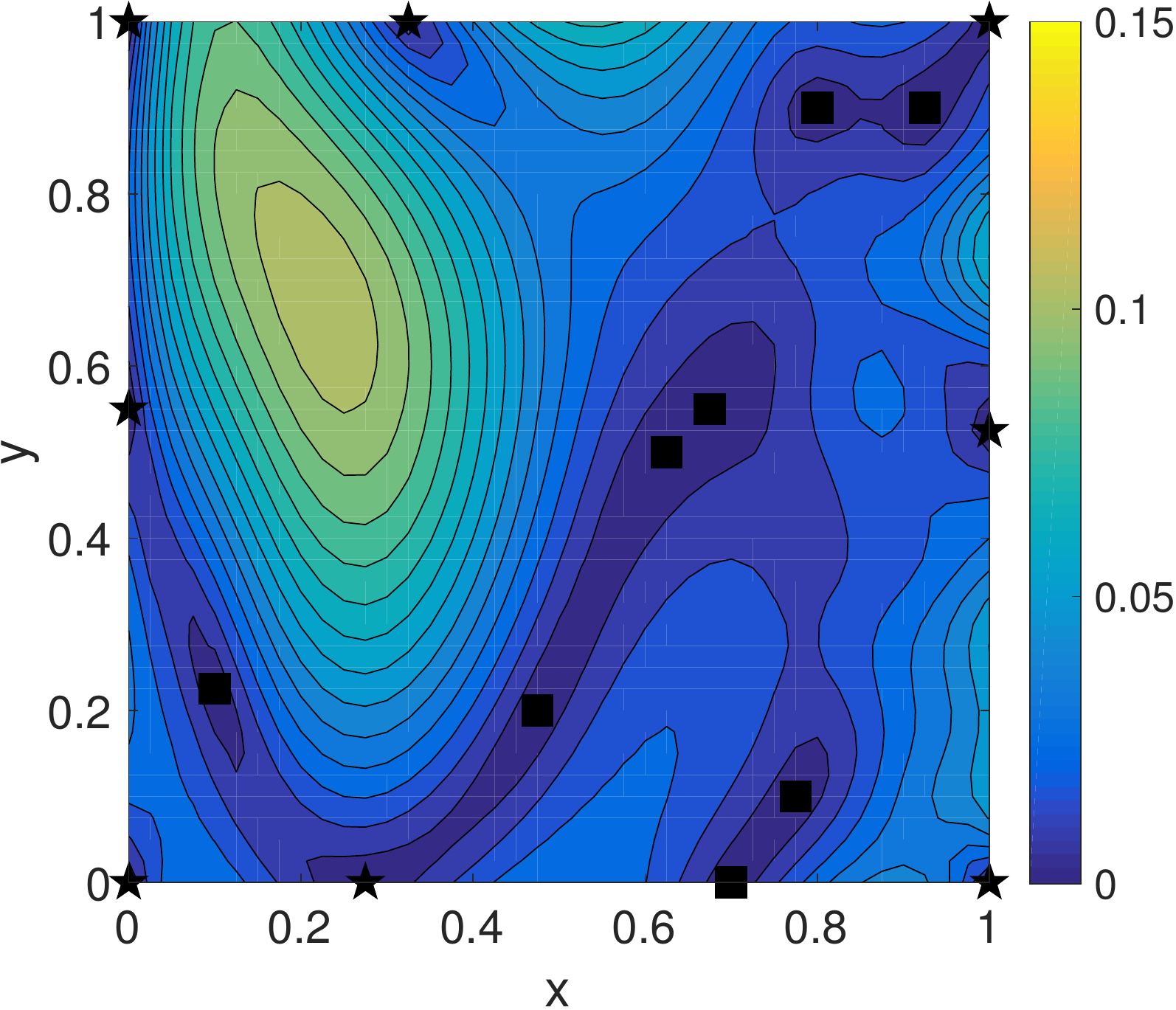}
    \caption{CoPhIK $\hat s$}
  \end{subfigure}\qquad
  \begin{subfigure}[b]{0.22\textwidth}
    \includegraphics[height=\textwidth]{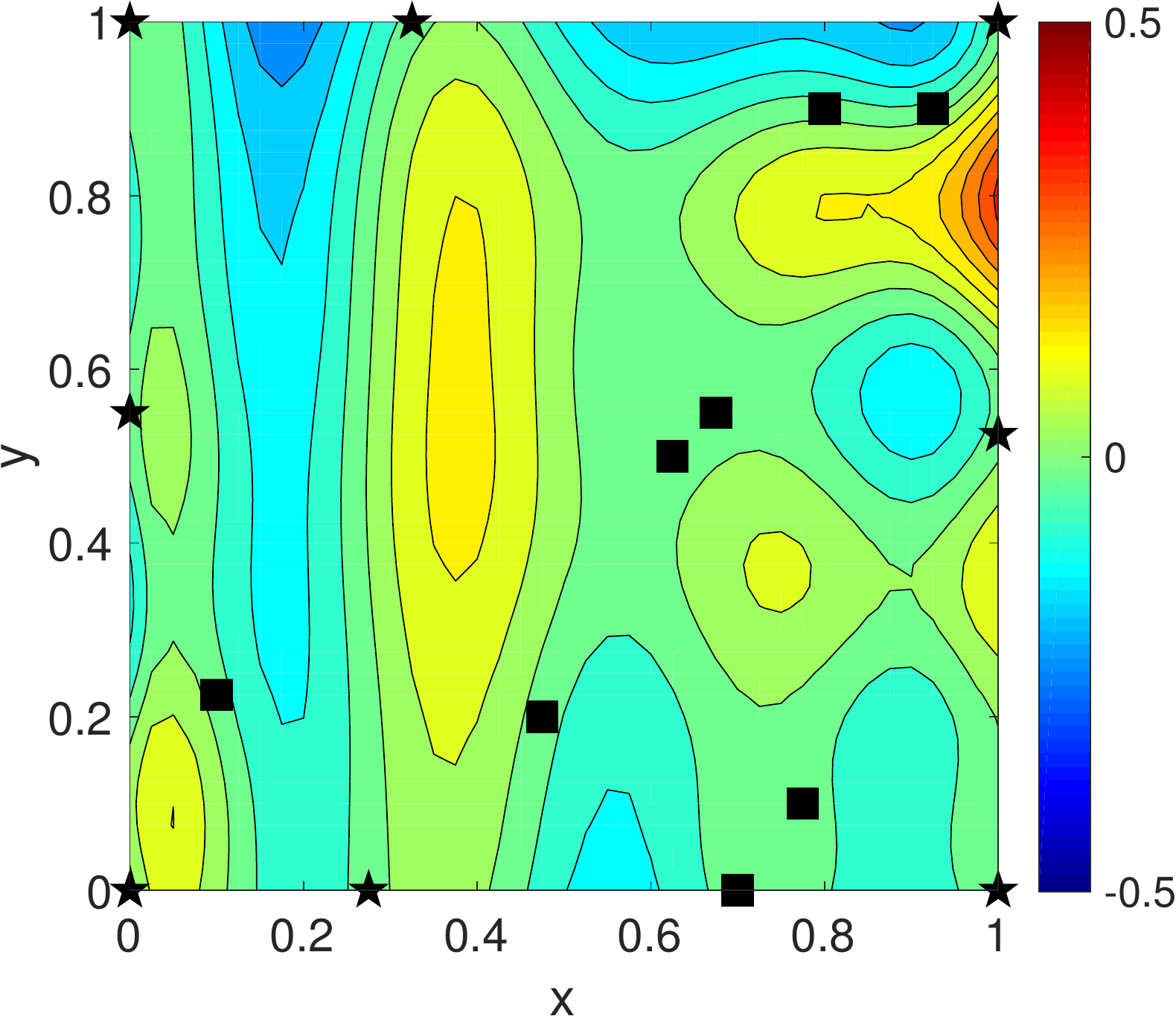}
    \caption{CoPhIK $\tensor F_r-\tensor F$}
  \end{subfigure}
  \caption{Reconstruction of the modified Branin function via active learning.
    Black squares are the locations of the original eight observations. Stars
  are newly added eight observations based on the actively learning algorithm.}
  \label{fig:branin_act}
\end{figure}

Figure~\ref{fig:branin_comp_l2} compares the relative error 
$\Vert\tensor F_r-\tensor F\Vert_F/\Vert \tensor F\Vert_F$ as a function of the
total number of observations for active learning based on kriging, PhIK,
modified PhIK and CoPhIK. For the original eight observations, the largest error
is in Kriging (over 50\%) followed by PhIK and modified PhIK errors 
(about $8\%$), with the smallest error in CoPhIK (less than $3\%$). As more
observations are added by the active learning algorithm, the error of Kriging
decreases to approximately $1.5\%$ at $24$ observations. The error of PhIK is
reduced from around $8\%$ to approximately $4\%$ at $12$ observations. Adding
$12$ more observations doesn't improve the accuracy. The error of modified PhIK 
is reduced from $8\%$ to $2\%$ at $12$ observations, then it changes very slowly
with additional observations. With $24$ observations, the accuracy of Kriging and
modified PhIK is approximately the same. CoPhIK has the best accuracy among all
the methods and, in general, its error decreases with additional measurements.
The error in CoPhIK reduces to less than $0.1\%$ with $24$ observations in 
total, which is more than one order of magnitude better than the other three
methods.
\begin{figure}[h]
  \centering%
  \includegraphics[width=0.4\textwidth]{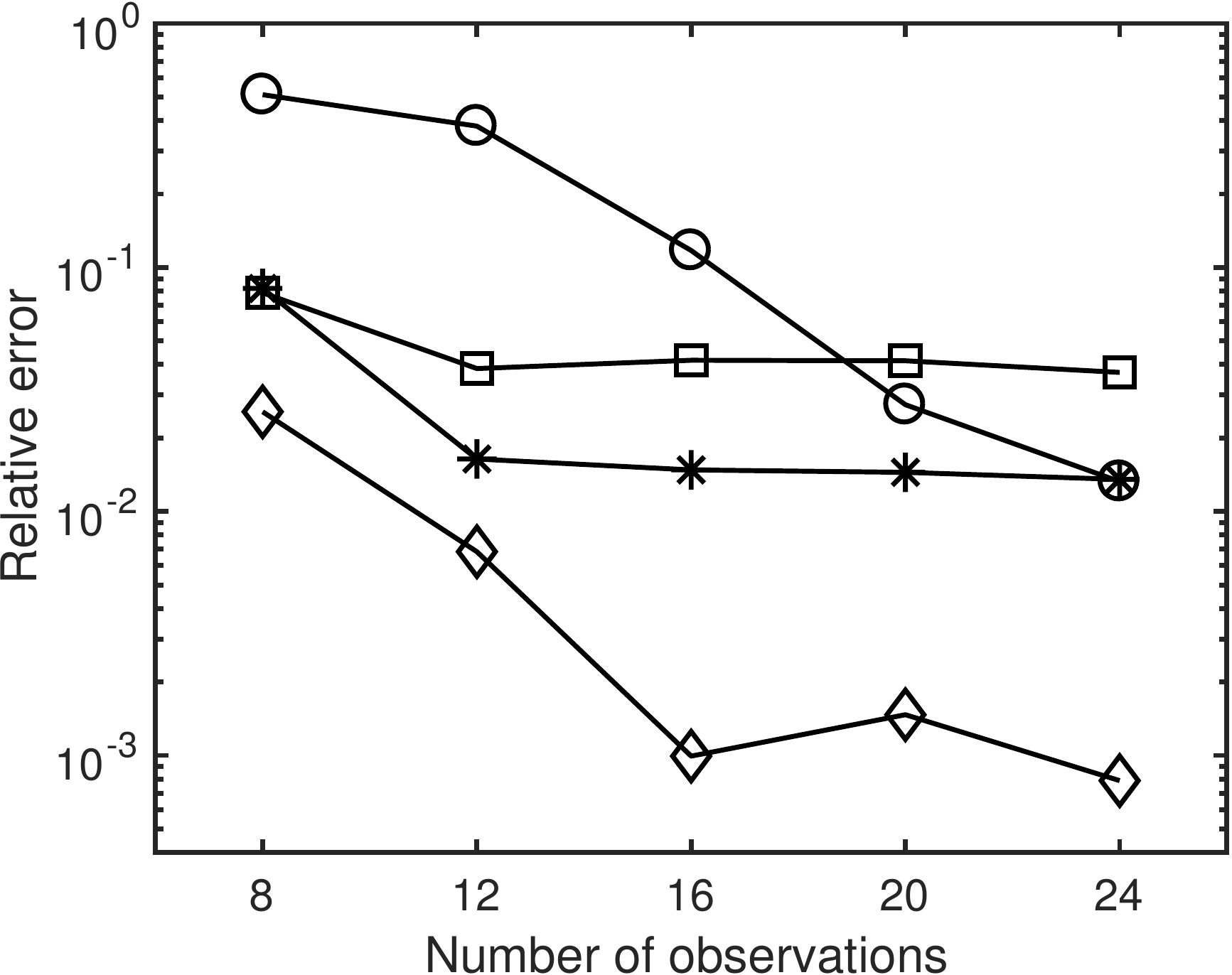}
  \caption{Relative error of reconstructed modified Branin function $\Vert\bm F_r-\bm F\Vert_F/\Vert\bm F\Vert_F$ using Kriging (``$\circ$"), PhIK (``$\square$"), modified PhIK (``$\ast$") and CoPhIK (``$\diamond$") with different numbers of total observations via active learning.}
  \label{fig:branin_comp_l2}
\end{figure}

Finally, this example also illustrates that smaller uncertainty ($\hat s$)
doesn't lead to smaller error in the posterior mean. In particular, in this
case, PhIK has the smallest $\hat s$, but CoPhIK posterior mean is the most
accurate.

\subsection{Heat transfer}
\label{subsec:heat}

In the second example, we consider the steady state of a heat transfer problem.
The dimensionless heat equation is given as
\begin{equation}
  \label{eq:heat}
  \dfrac{\partial T}{\partial t} - \nabla\cdot(\kappa(T) \nabla T) = 0, \quad
  \bm x\in\mathbb{D},
\end{equation}
subject to the boundary conditions:
\begin{equation}
  \label{eq:heat_bc}
  \left \{ \begin{aligned}
    T &= -30\cos(2\pi x)+40, && x\in\Gamma_1, \\
    \frac{\partial T}{\partial\bm n} &= -20, && x\in \Gamma_2,\\
    T &= 30\cos(2\pi(x+0.1))+40, && x\in\Gamma_3, \\
    \dfrac{\partial T}{\partial\bm n} &= 20, && x\in \Gamma_4, \\
    \dfrac{\partial T}{\partial\bm n} &= 0, && x\in \Gamma_5. 
  \end{aligned}
  \right .
\end{equation}
Here, $T(\bm x, t)$ is the temperature and $\kappa(T)$ is the 
temperature-dependent heat conductivity. The computational domain $\mathbb{D}$
is the rectangle $[-0.5, 0.5] \times [-0.2, 0.2]$ with two circular cavities
$R_1(O_1, r_1)$ and $R_2(O_2, r_2)$, with 
$O_1=(-0.3,0), O_2=(0.2,0), r_1=0.1, r_2=0.15$ (see Figure~\ref{fig:heat_geo}).
The reference conductivity is set as 
\begin{equation}
  \label{eq:heat_cond_true}
 \kappa(T)=1.0+\exp(0.02T),
\end{equation}
which results in the reference steady state temperature field shown in
Figure~\ref{fig:heat_truth}. This solution was obtained by solving 
Eq.~\ref{eq:heat} and~\ref{eq:heat_bc} using the finite element method with
unstructured triangular mesh implemented by the MATLAB PDE toolbox. The number
of degrees of freedom is $1319$, with a maximum grid size of $0.02$.
Observations of this exact profile are collected at six locations, marked by
black squares in Figure~\ref{fig:heat_truth}.
\begin{figure}[!h]
  \centering%
  \includegraphics[width=0.55\textwidth]{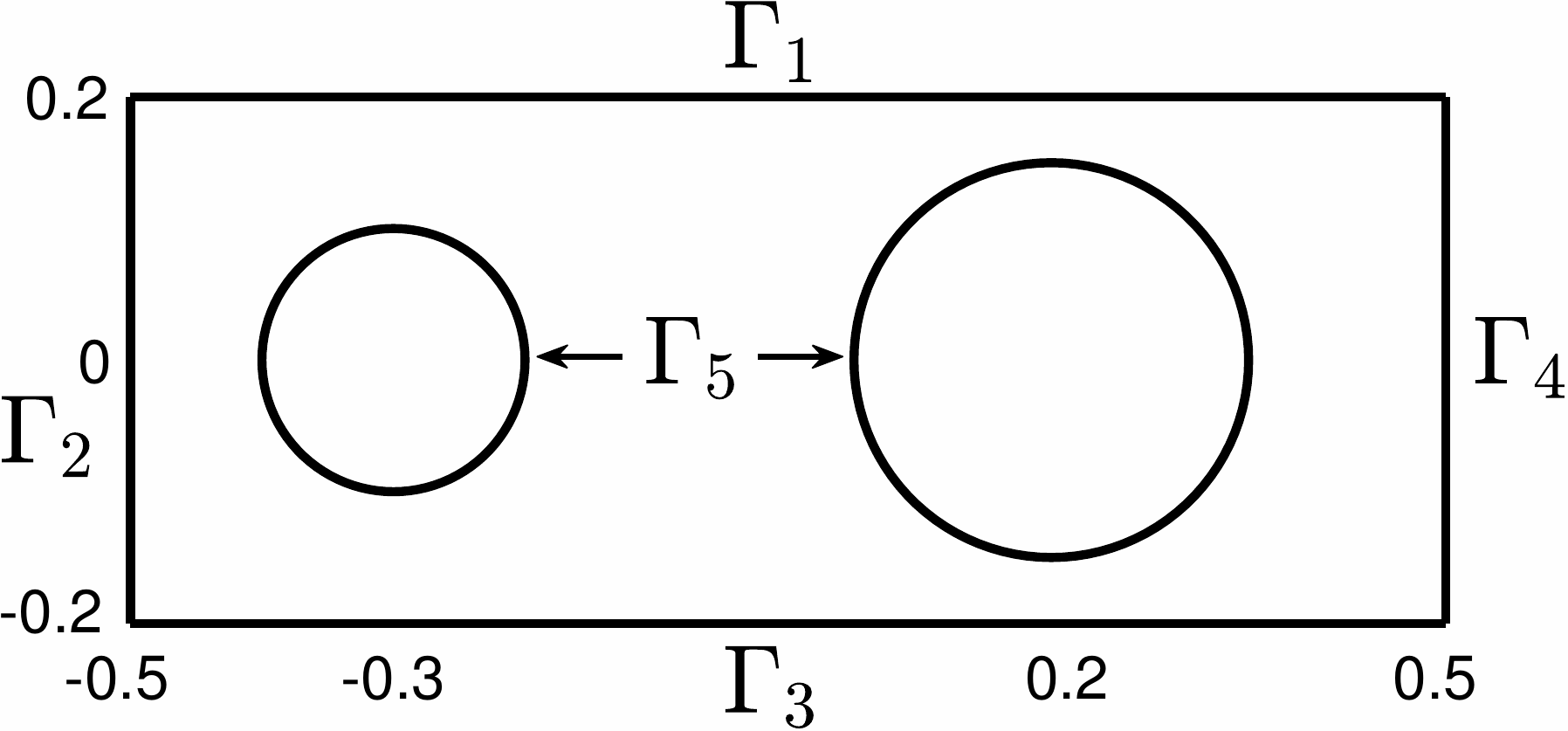}
  \caption{Heat transfer solution (steady state) computational domain.}
  \label{fig:heat_geo}
\end{figure}
\begin{figure}[!h]
  \centering%
  \includegraphics[width=0.6\textwidth]{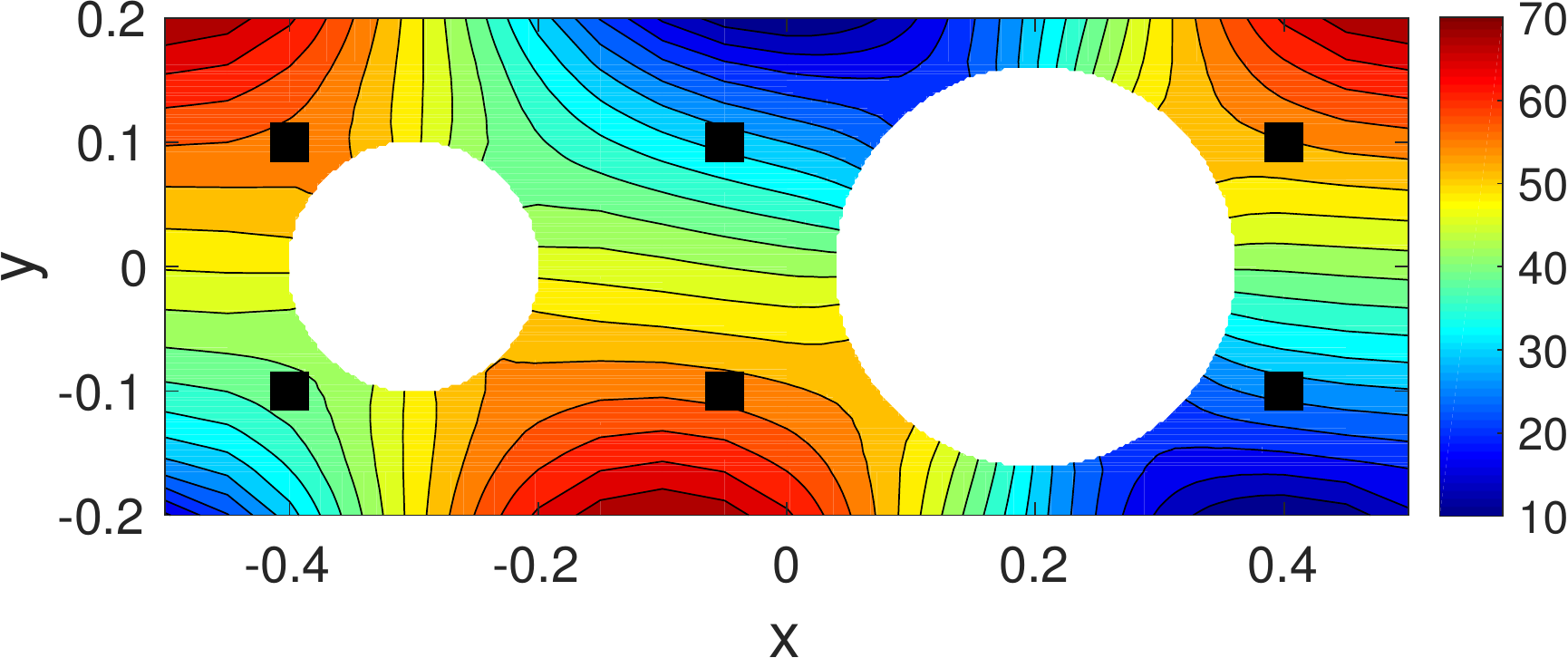}
  \caption{Contours of heat transfer solution (steady state) and locations of 
  six observations (black squares).}
  \label{fig:heat_truth}
\end{figure}

Now we assume that the conductivity model (\ref{eq:heat_cond_true}) is unknown 
and an ``expert knowledge" of $\kappa$ is expressed as 
\begin{equation}
  \label{eq:heat_cond_model}
  \kappa(T;\omega) = 0.1+\xi(\omega) T, 
\end{equation}
where $\xi(\omega)$ is a uniform random variable $\mathcal{U}[0.0012,0.0108]$.
Note that this example represents a biased expert knowledge that systematically
underestimates the heat conductivity and assumes an incorrect functional 
dependence of $\kappa$ on $T$. We sample the stochastic model by generating 
$M = 400$ samples of $\xi(\omega)$ and then solving Eq.~\eqref{eq:heat} for each
realization. We denote the resulting ensemble of temperatures solutions by
$\{\hat{\tensor F}^m\}_{m=1}^M$.

The first row in Figure~\ref{fig:heat_krig_mc} presents the posterior mean, RMSE
and pointwise reconstruction error of Kriging regression obtained with six
measurements whose locations are also shown in this figure. The relative error 
is large (about $27\%$). The second row in Figure~\ref{fig:heat_krig_mc} shows 
the mean and standard deviation of the ensemble $\{\hat{\tensor F}\}_{m=1}^M$
and the difference between the ensemble mean and the exact field. In this case,
the relative error of the ensemble average is $8\%$, which may be acceptable in
some application. For the selected number and locations of observations, Kriging
performs worse than the unconditional stochastic model.
\begin{figure}[!h]
  \centering%
  \begin{subfigure}[b]{0.32\textwidth}
    \centering%
    \includegraphics[width=\textwidth]{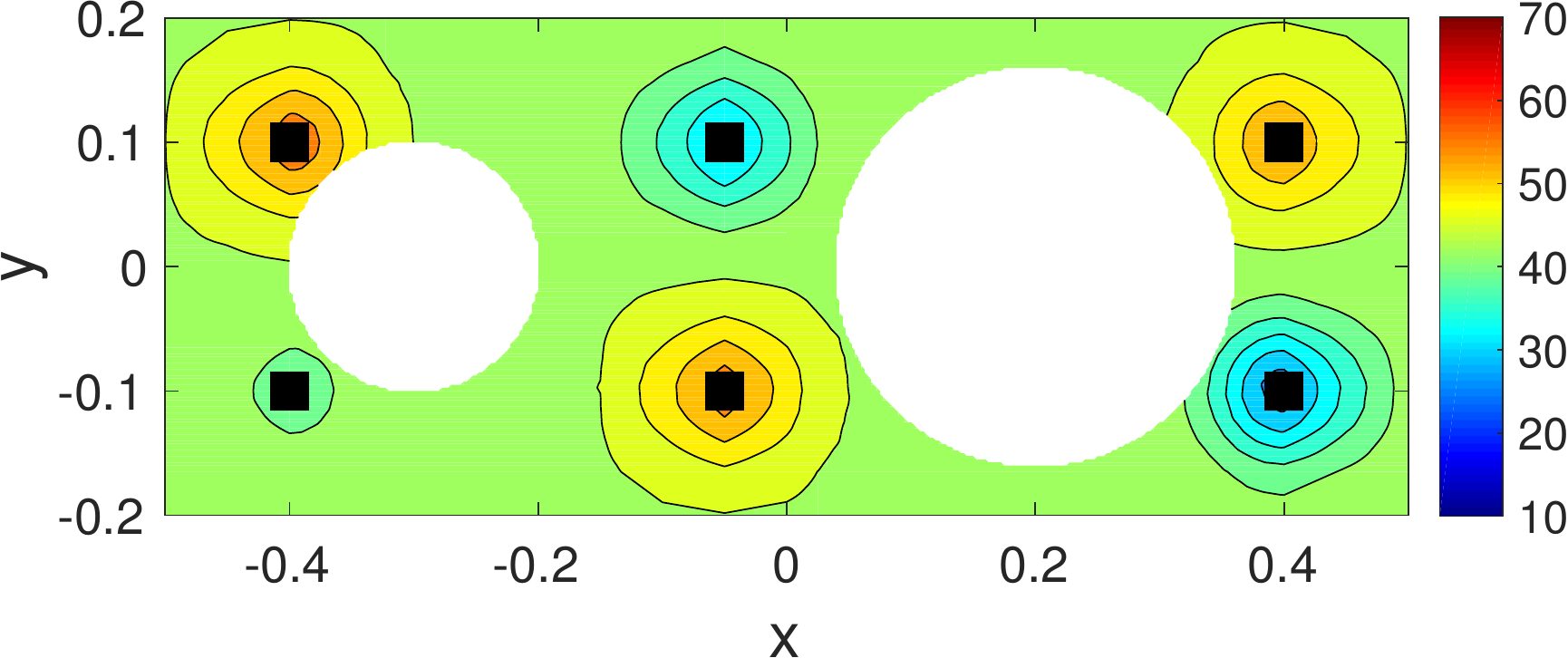}
    \caption{Kriging $\tensor F_r$}
  \end{subfigure}
  \begin{subfigure}[b]{0.32\textwidth}
    \centering%
    \includegraphics[width=\textwidth]{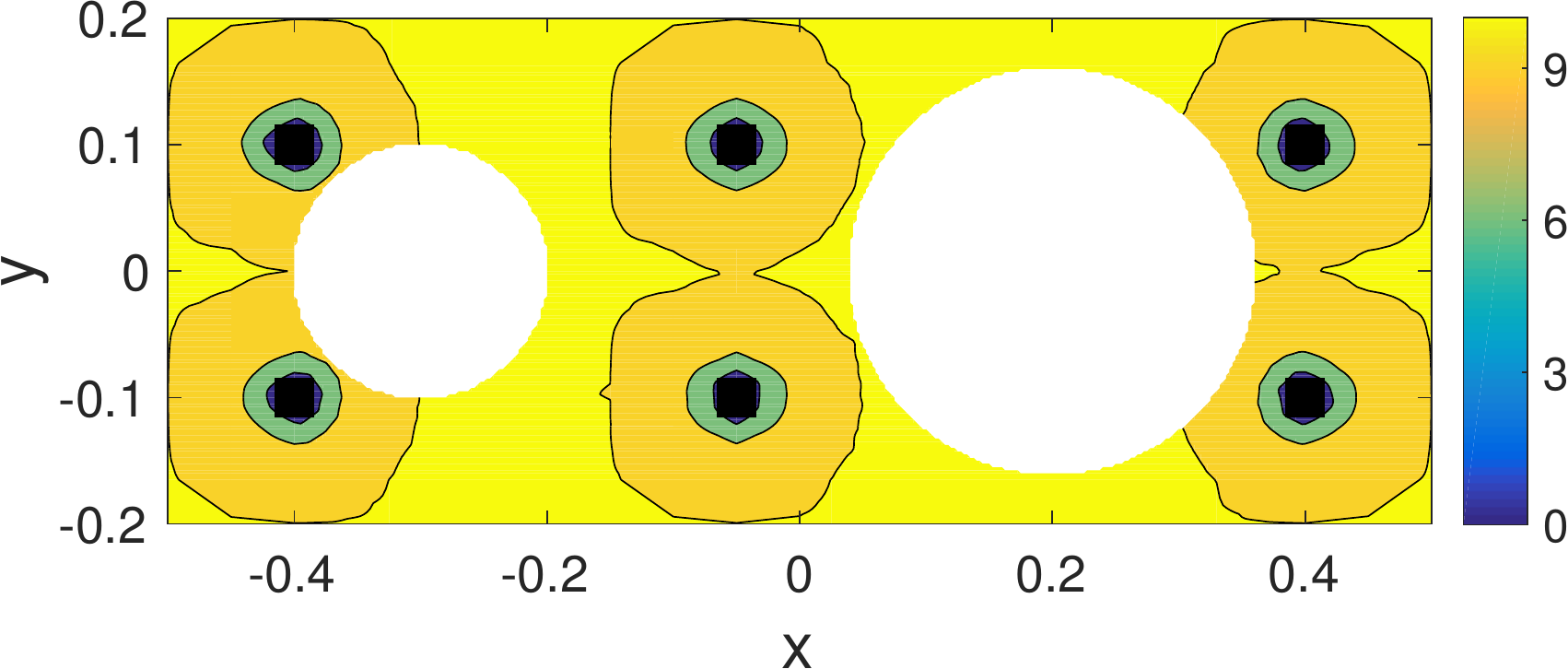}
    \caption{Kriging $\hat s$}
  \end{subfigure}
  \begin{subfigure}[b]{0.32\textwidth}
    \centering%
    \includegraphics[width=\textwidth]{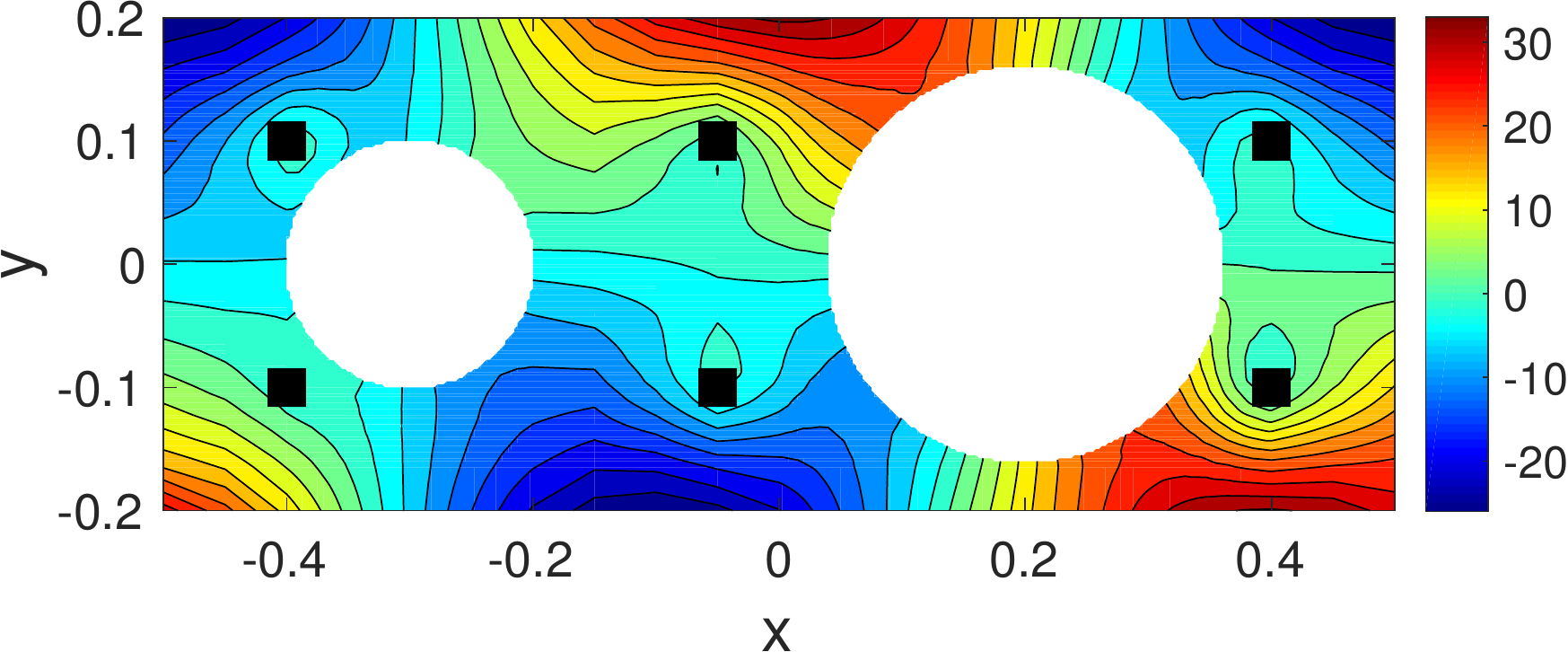}
    \caption{Kriging $\tensor F_r-\tensor F$}
  \end{subfigure}\\
  \begin{subfigure}[b]{0.32\textwidth}
    \centering%
    \includegraphics[width=\textwidth]{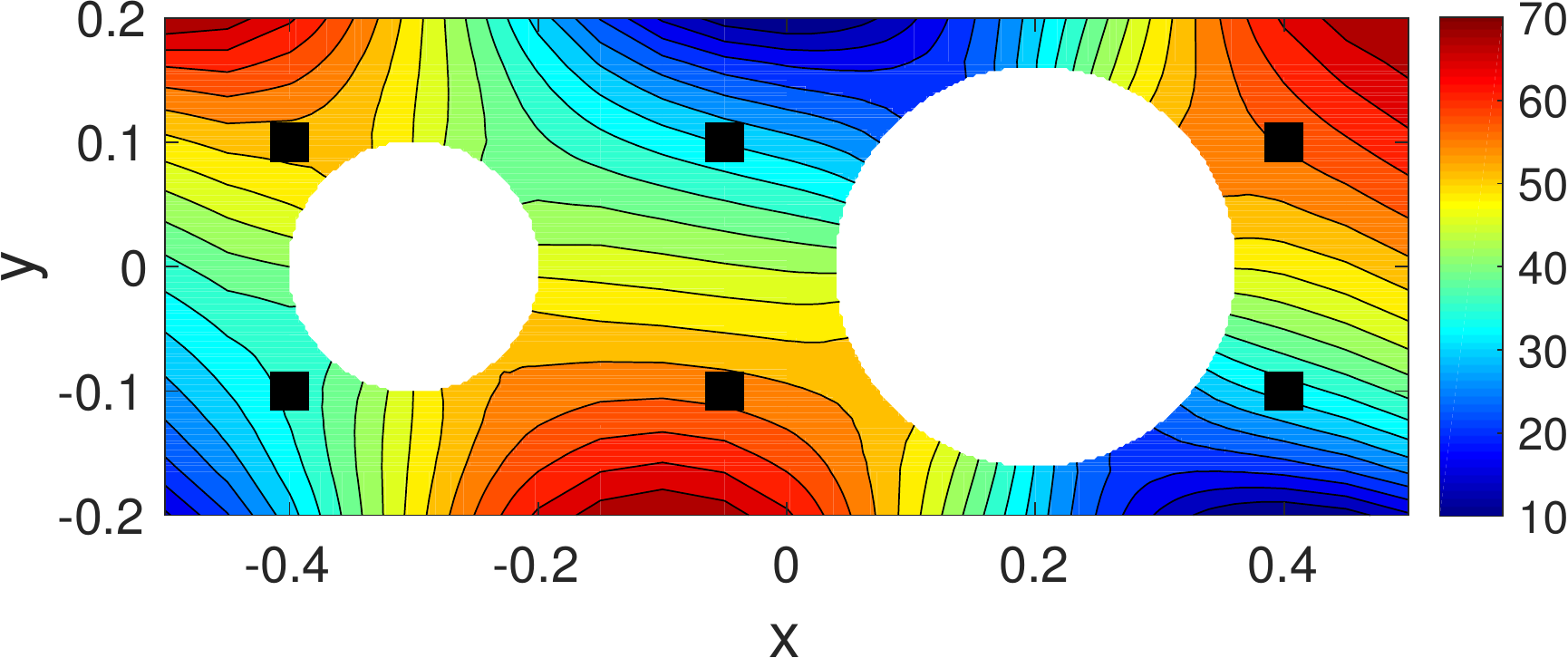}
    \caption{$\mu(\hat{\tensor F}^m)$}
  \end{subfigure}
  \begin{subfigure}[b]{0.32\textwidth}
    \centering%
    \includegraphics[width=\textwidth]{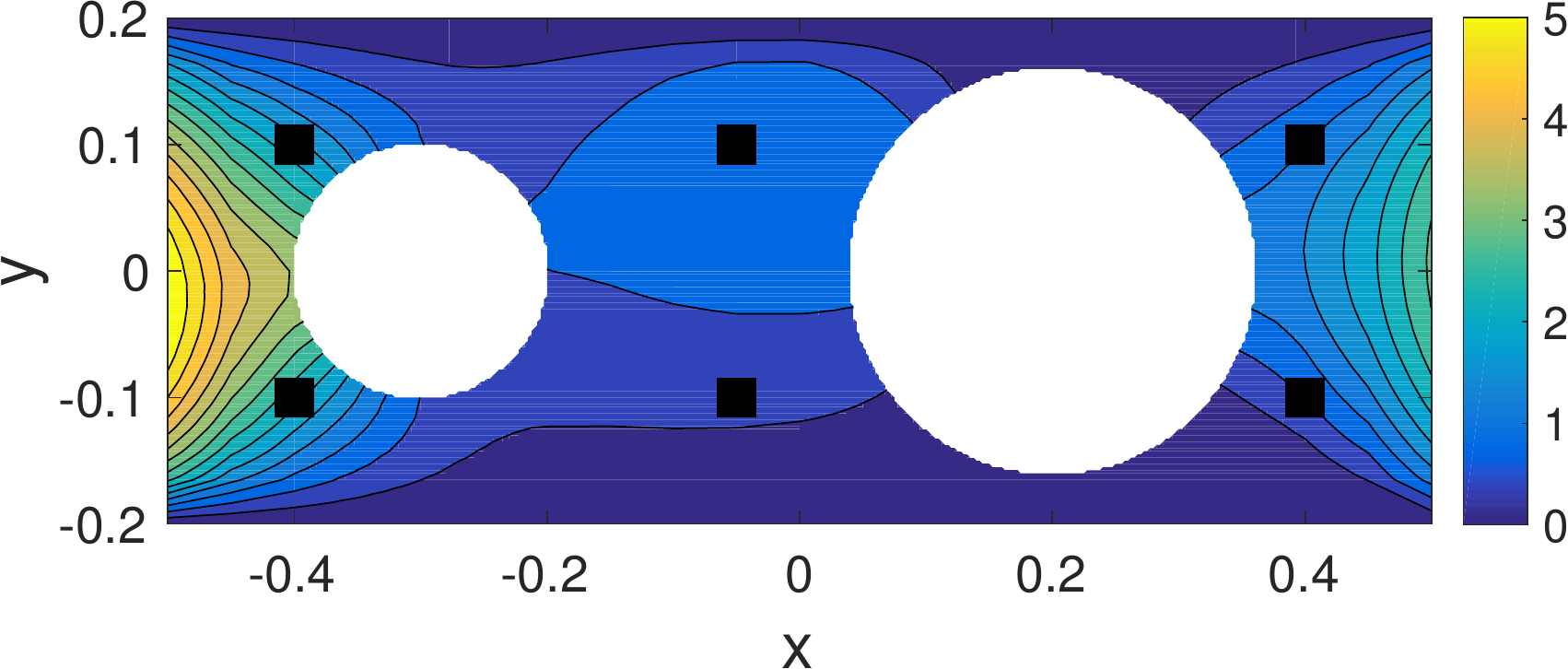}
    \caption{$\sigma(\hat{\tensor F}^m)$}
  \end{subfigure}
  \begin{subfigure}[b]{0.32\textwidth}
    \centering%
    \includegraphics[width=\textwidth]{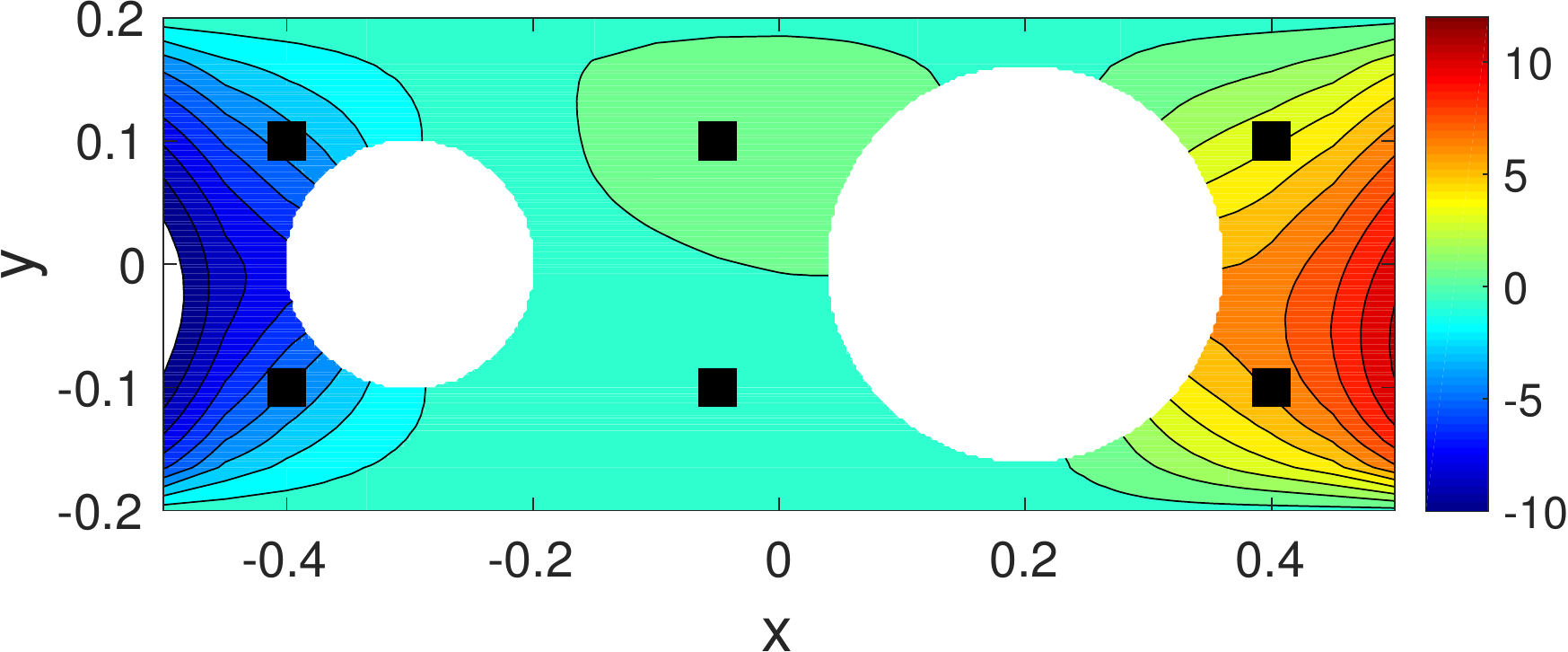}
    \caption{$\mu(\hat{\tensor F}^m)-\tensor F$}
  \end{subfigure}
  \caption{Reconstruction of the steady state solution for heat transfer problem 
           by Kriging (first row) and statistics of the ensemble 
           $\{\hat{\tensor F}^m\}_{m=1}^M$ (second row).}
  \label{fig:heat_krig_mc}
\end{figure}

Next, we use PhIK and CoPhIK to obtain prediction of $T$.
Figure~\ref{fig:heat_phik} shows the results for PhIK in the top row and for 
CoPhIK in the bottom row. CoPhIK outperforms PhIK as it results in smaller
reconstruction errors. The relative errors of the reconstruction are $4.8\%$
for CoPhIK and $7.8\%$ for PhIK. As before, $\hat{s}$ in PhIK is smaller than in
CoPhIK. Both, PhIK and CoPhIK are more accurate and certain than Kriging.
\begin{figure}[!h]
  \centering%
  \begin{subfigure}[b]{0.32\textwidth}
    \centering%
    \includegraphics[width=\textwidth]{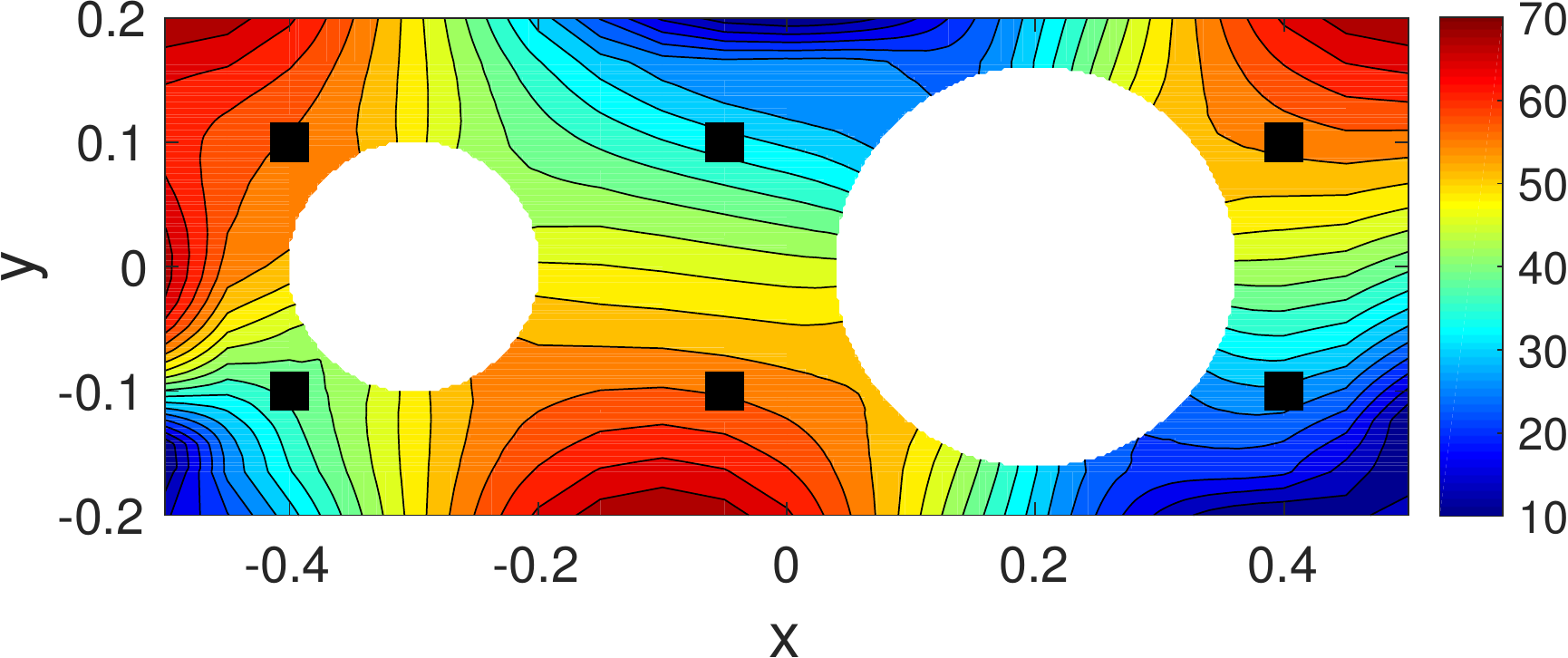}
    \caption{PhIK $\bm F_r$}
  \end{subfigure}
  \begin{subfigure}[b]{0.333\textwidth}
    \centering%
    \includegraphics[width=\textwidth]{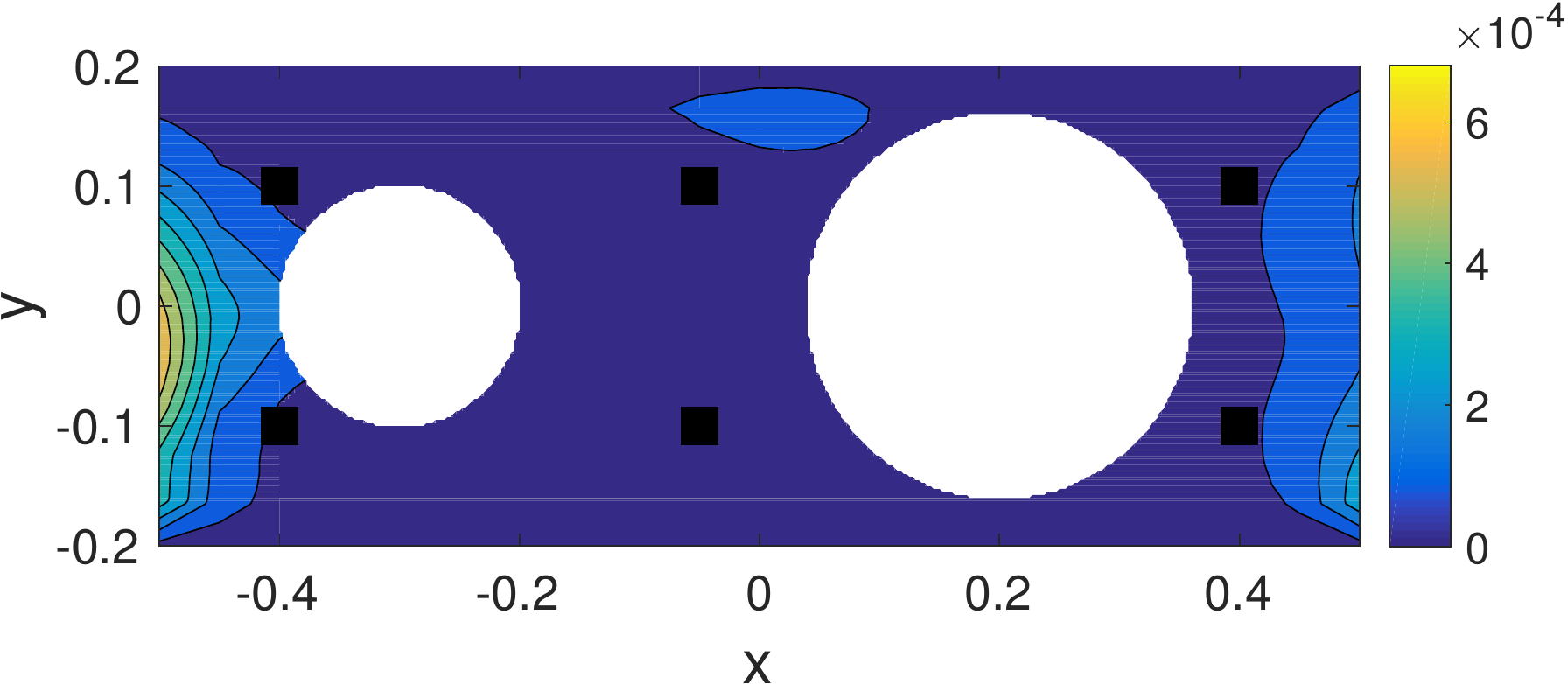}
    \caption{PhIK $\hat s$}
  \end{subfigure}
  \begin{subfigure}[b]{0.32\textwidth}
    \centering%
    \includegraphics[width=\textwidth]{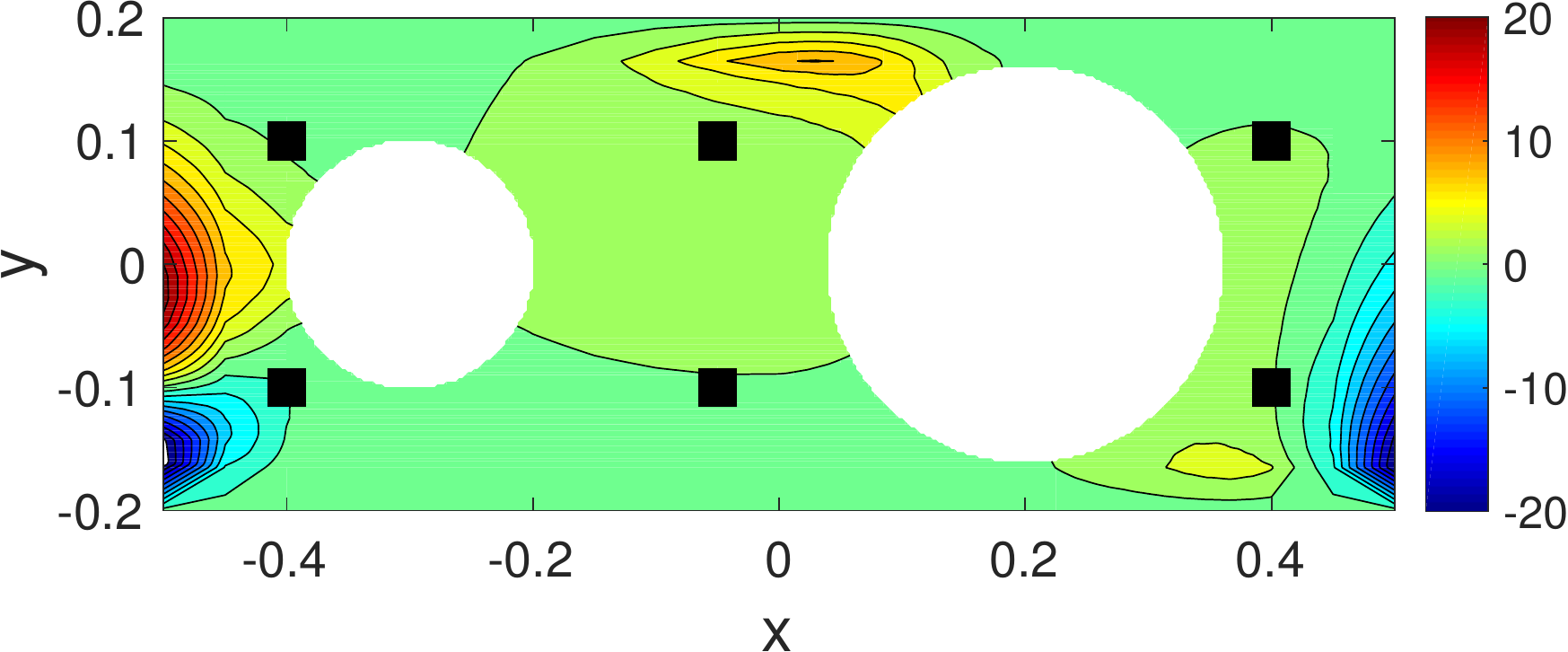}
    \caption{PhIK $\bm F_r-\bm F$}
  \end{subfigure}
  \begin{subfigure}[b]{0.32\textwidth}
    \centering%
    \includegraphics[width=\textwidth]{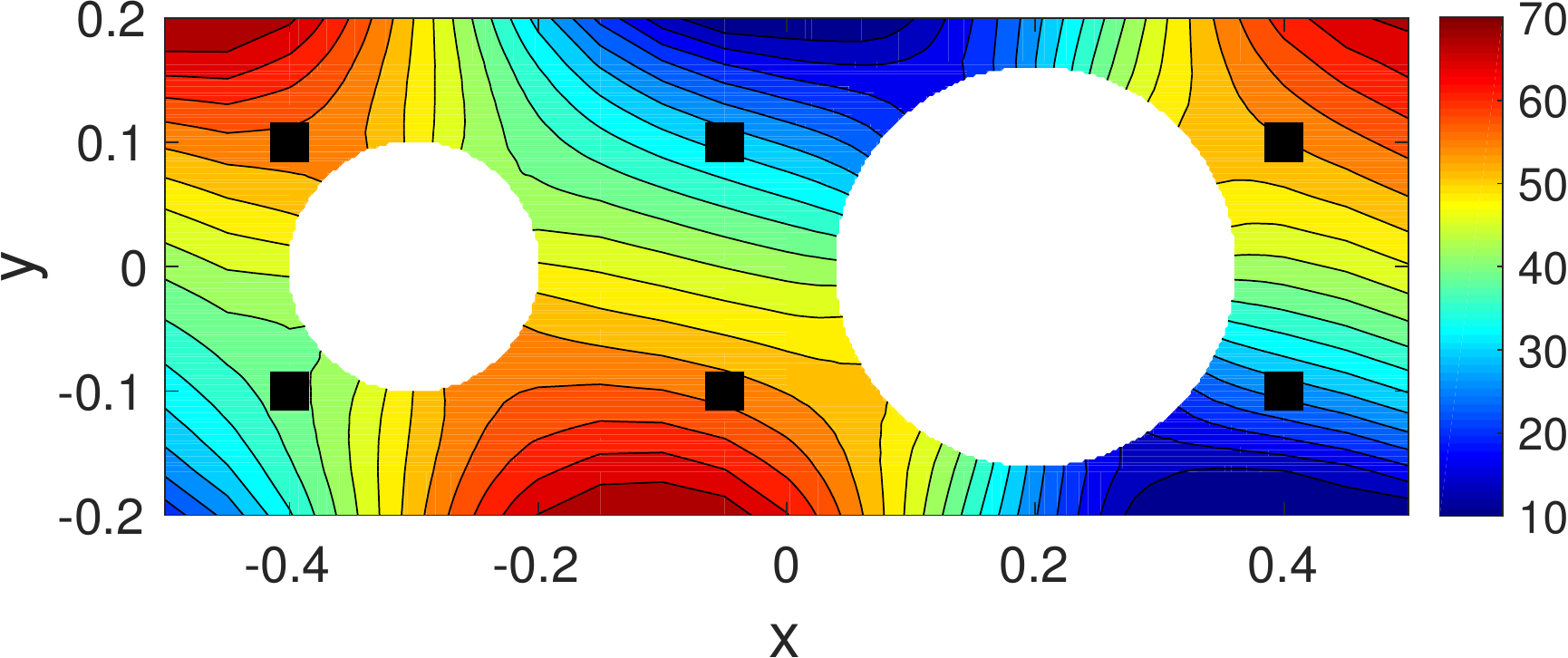}
    \caption{CoPhIK $\bm F_r$}
  \end{subfigure}
  \begin{subfigure}[b]{0.32\textwidth}
    \centering%
    \includegraphics[width=\textwidth]{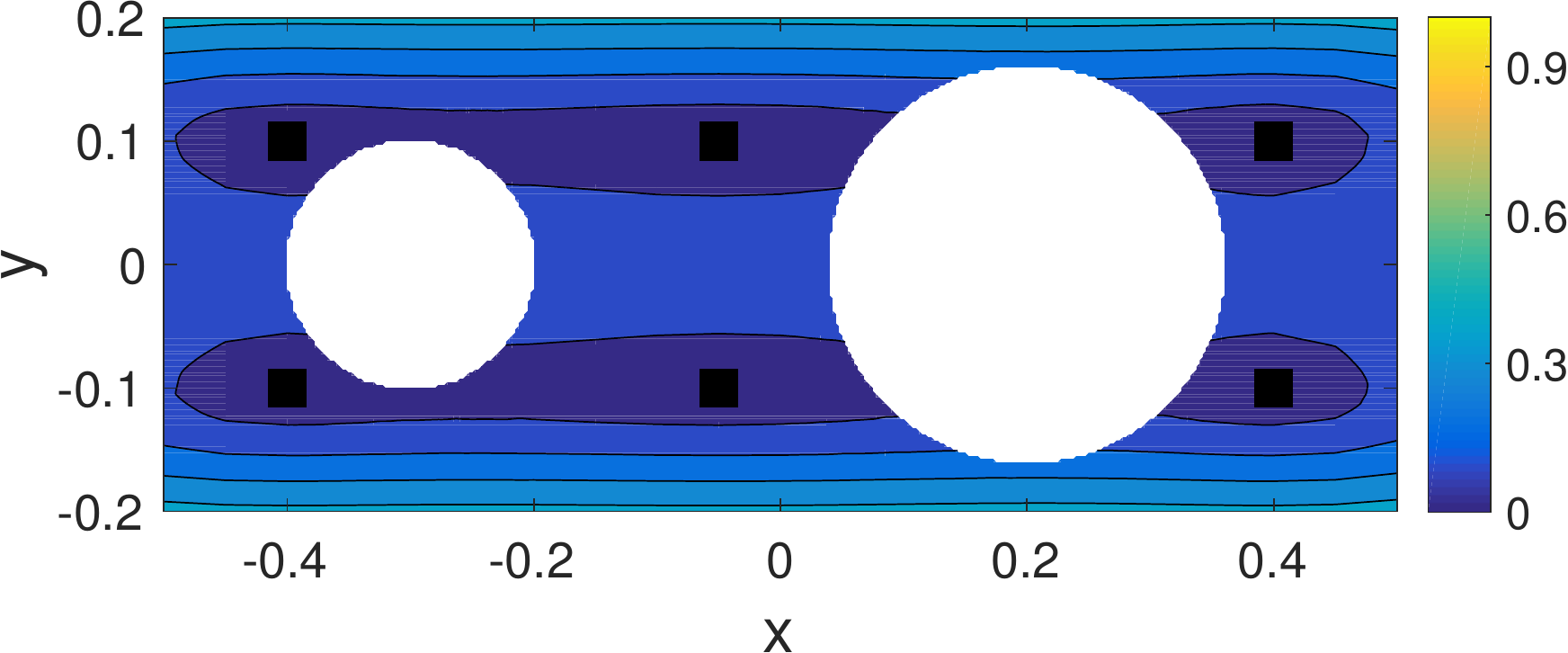}
    \caption{CoPhIK $\hat s$}
  \end{subfigure} 
  \begin{subfigure}[b]{0.32\textwidth}
    \centering%
    \includegraphics[width=\textwidth]{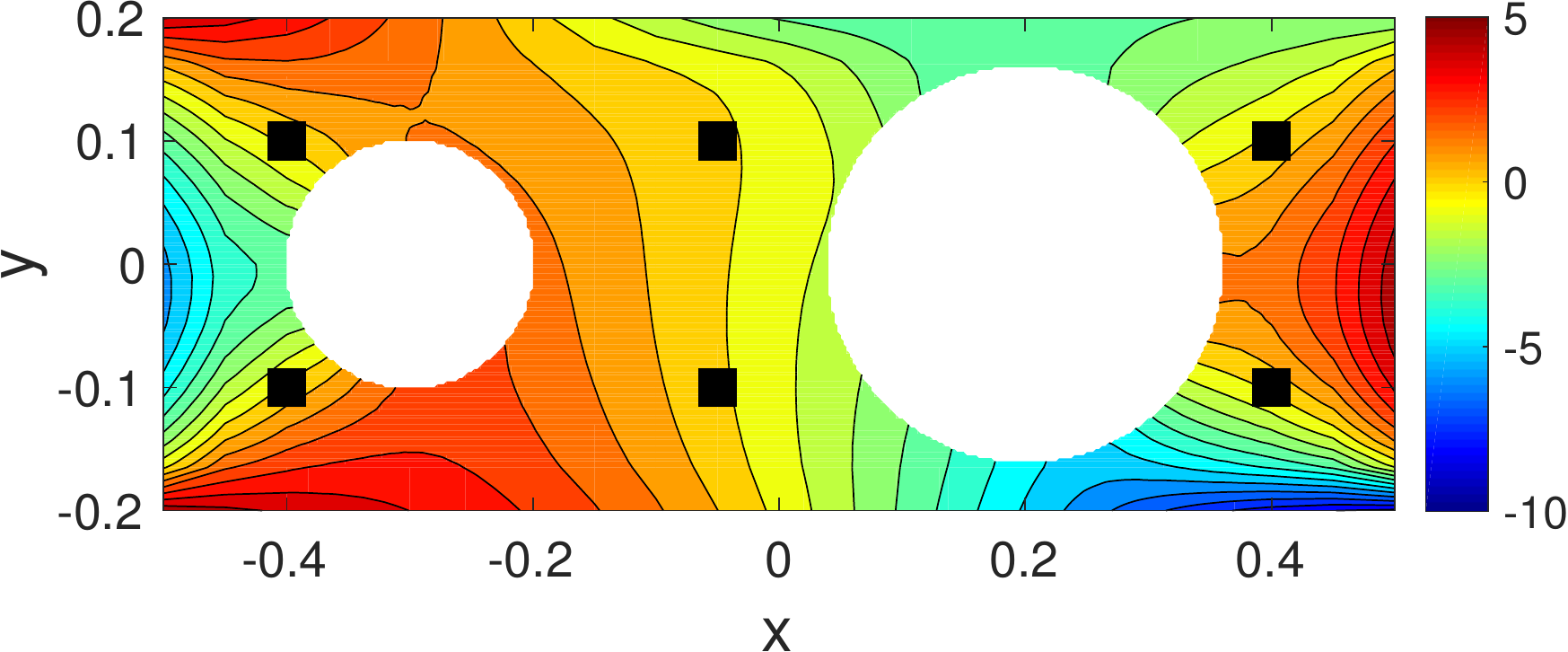}
    \caption{CoPhIK $\bm F_r-\bm F$}
  \end{subfigure}
  \caption{Reconstruction of the steady state solution for heat transfer problem
           by PhIK (first row) and CoPhIK (second row).} 
  \label{fig:heat_phik}
\end{figure}

Finally, we employ active learning to identify additional observation locations.
Figure~\ref{fig:heat_act} displays Kriging, PhIK, and CoPhIK predictions 
obtained with $14$ observations. Eight new observations are marked with stars
and original six observations are denoted with squares. These three methods 
place additional observations at different locations. Kriging suggests new
observation locations mostly along the external boundaries as there are no
original observations on the boundaries and extrapolation in Kriging is the most
uncertain in these subdomains. PhIK identifies additional observation on the
Neumann boundaries $\Gamma_2$ and $\Gamma_4$. CoPhIK identifies new observations
locations on boundaries in a manner similar to Kriging, but also adds an
observation location in the interior of $\mathbb{D}$. 
Figure~\ref{fig:heat_rel_err} presents a quantitative study of the relative
error as a function of the total number of observations for the three methods.
It shows that CoPhIK is more accurate than Kriging and PhIK for a given number 
of observation points. As more observations are available, the errors in Kriging
and CoPhIK decrease while the error of PhIK reaches a constant value after the
first few observations. In this case, when $22$ observations are used (six
original ones plus $14$ added ones through active learning), the relative errors
are $3\%$, $4\%$, and less than $1\%$ for Kriging, PhIK and CoPhIK, respectively.
We also used the modified PhIK method to model the data and found that the 
relative error in this method is slightly smaller than in PhIK. However, the
modified PhIK reconstruction does not satisfy the Dirichlet boundary condition 
on $\Gamma_1$ and $\Gamma_3$. Therefore, we do not report the modified PhIK 
results.
\begin{figure}[!h]
  \centering%
  \begin{subfigure}[b]{0.32\textwidth}
    \centering%
    \includegraphics[width=\textwidth]{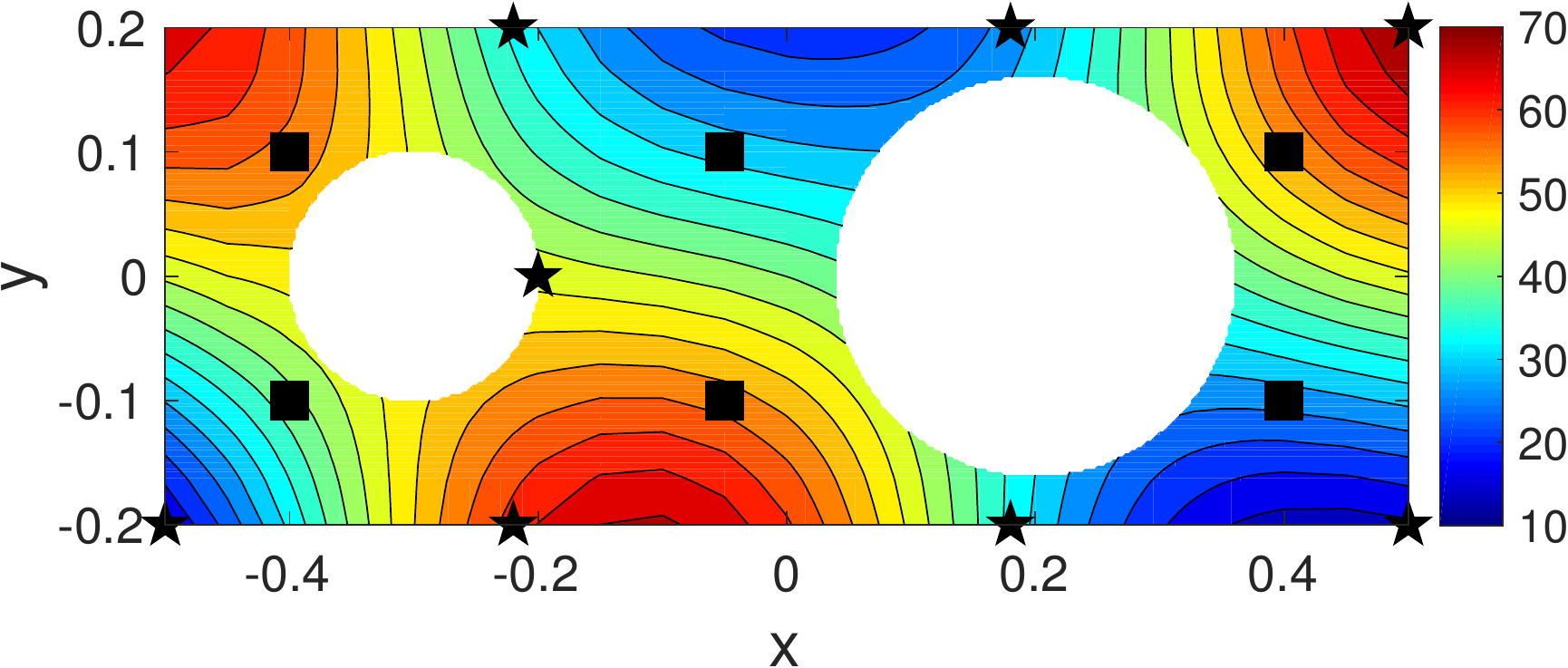}
    \caption{Kriging $\tensor F_r$}
  \end{subfigure}
  \begin{subfigure}[b]{0.32\textwidth}
    \centering%
    \includegraphics[width=\textwidth]{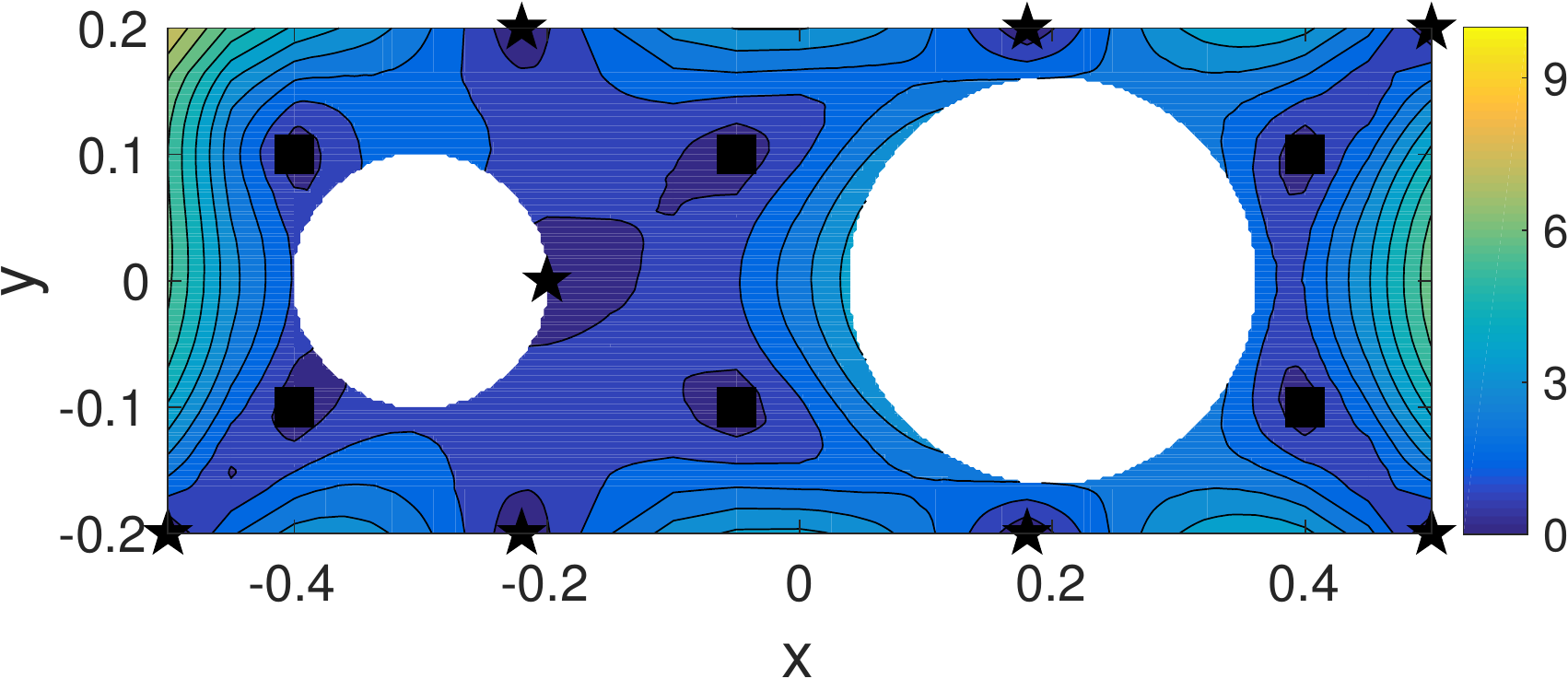}
    \caption{Kriging $\hat s$}
  \end{subfigure}
  \begin{subfigure}[b]{0.32\textwidth}
    \centering%
    \includegraphics[width=\textwidth]{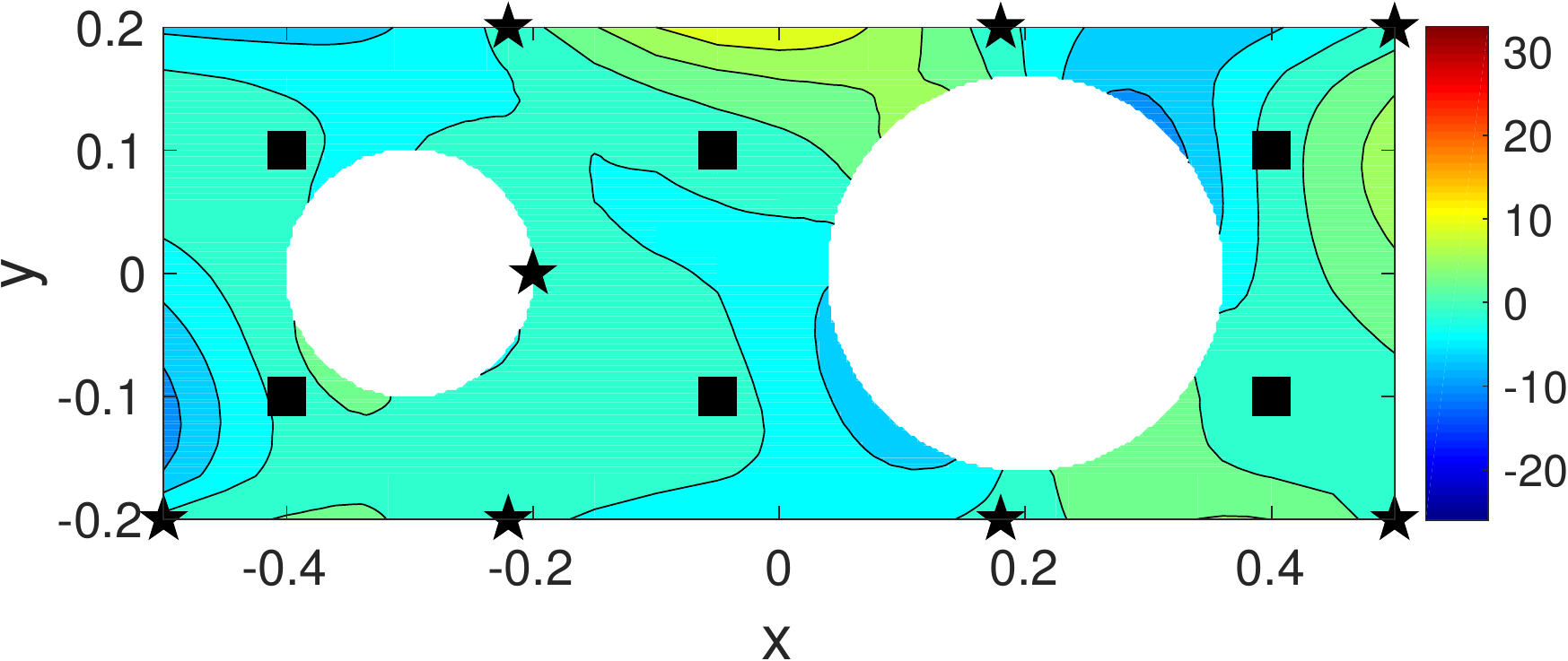}
    \caption{Kriging $\tensor F_r-\tensor F$}
  \end{subfigure}
  \begin{subfigure}[b]{0.32\textwidth}
    \centering%
    \includegraphics[width=\textwidth]{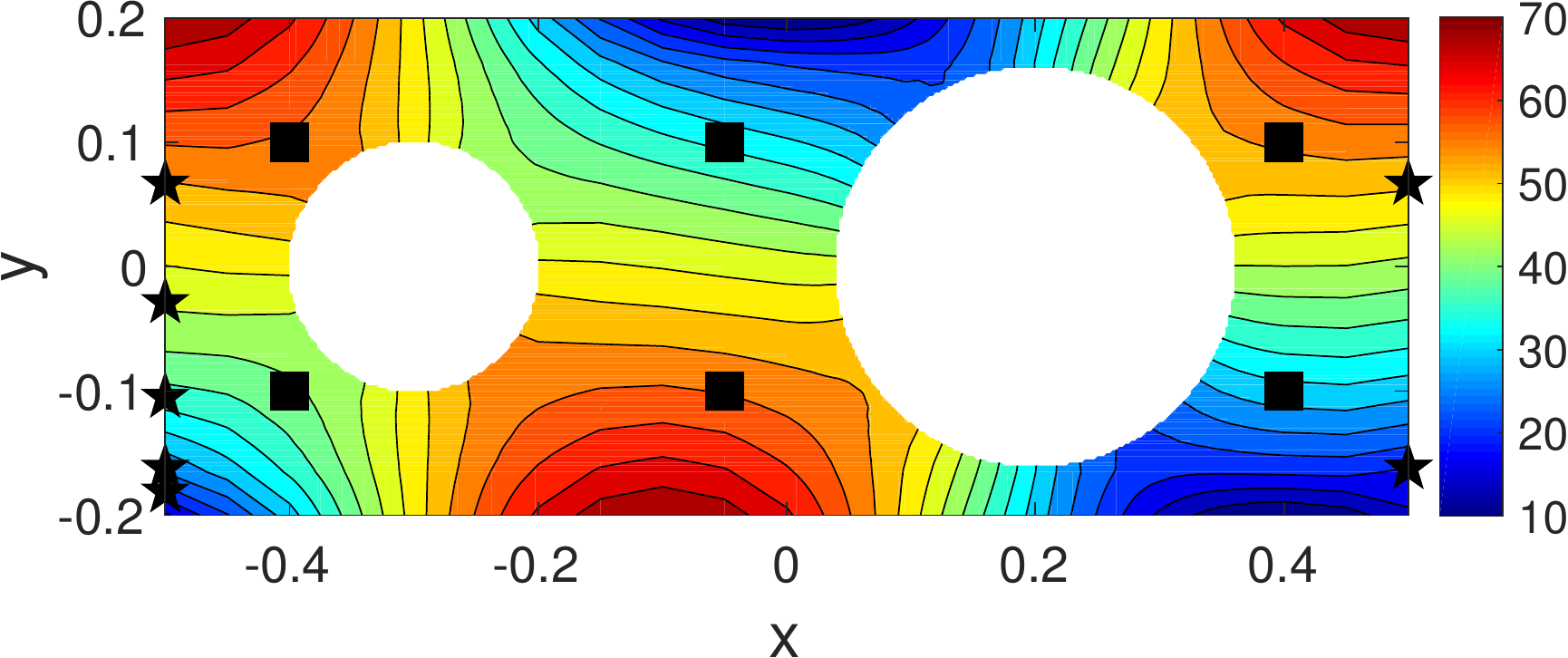}
    \caption{PhIK $\tensor F_r$}
  \end{subfigure}~
  \begin{subfigure}[b]{0.325\textwidth}
    \centering%
    \includegraphics[width=\textwidth]{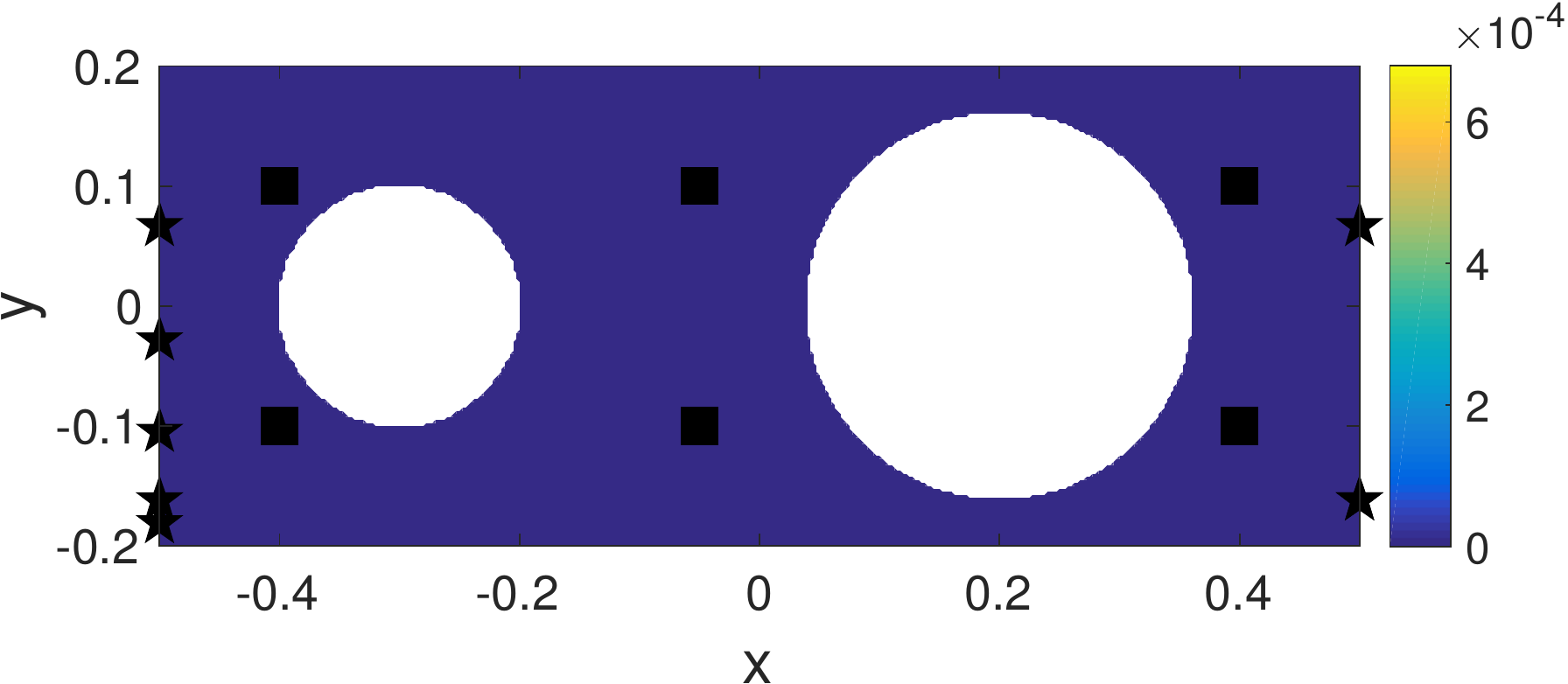}
    \caption{PhIK $\hat s$}
  \end{subfigure}
  \begin{subfigure}[b]{0.32\textwidth}
    \centering%
    \includegraphics[width=\textwidth]{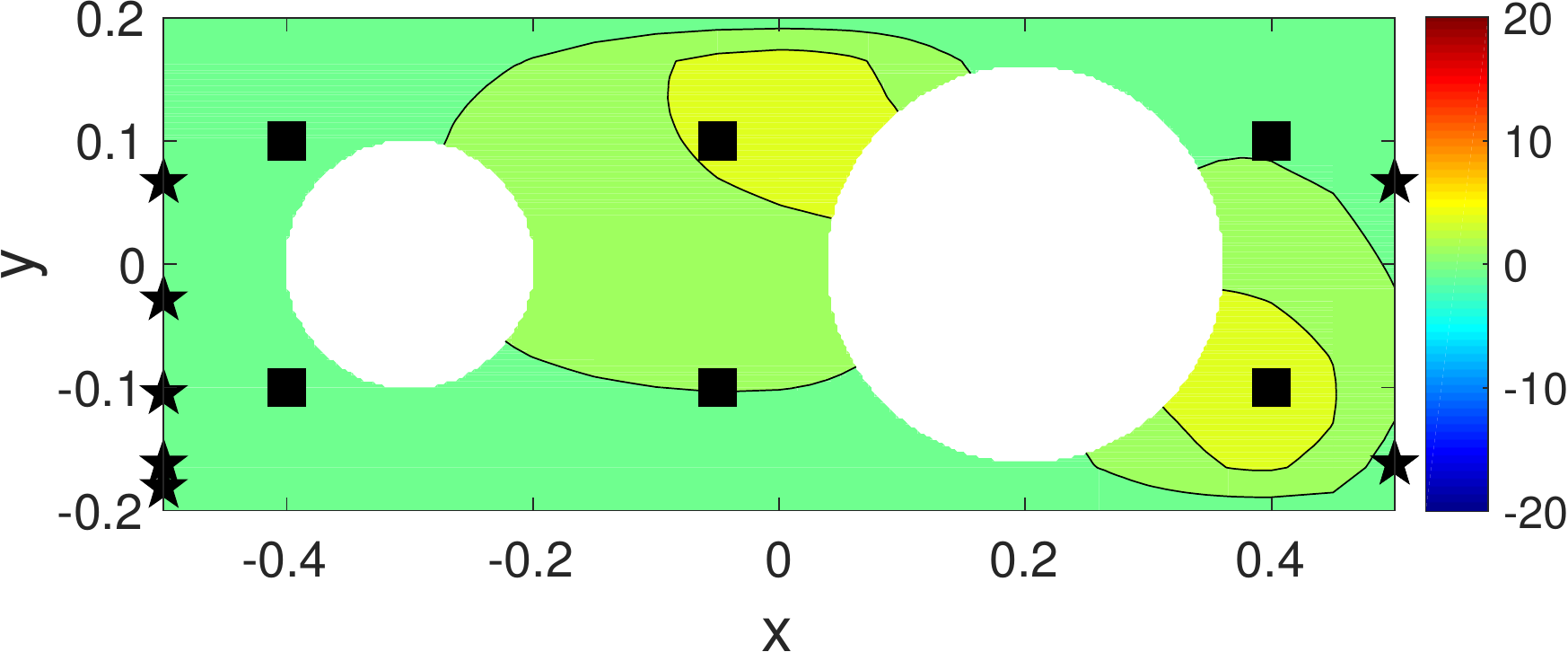}
    \caption{PhIK $\tensor F_r-\tensor F$}
  \end{subfigure}
  \begin{subfigure}[b]{0.32\textwidth}
    \centering%
    \includegraphics[width=\textwidth]{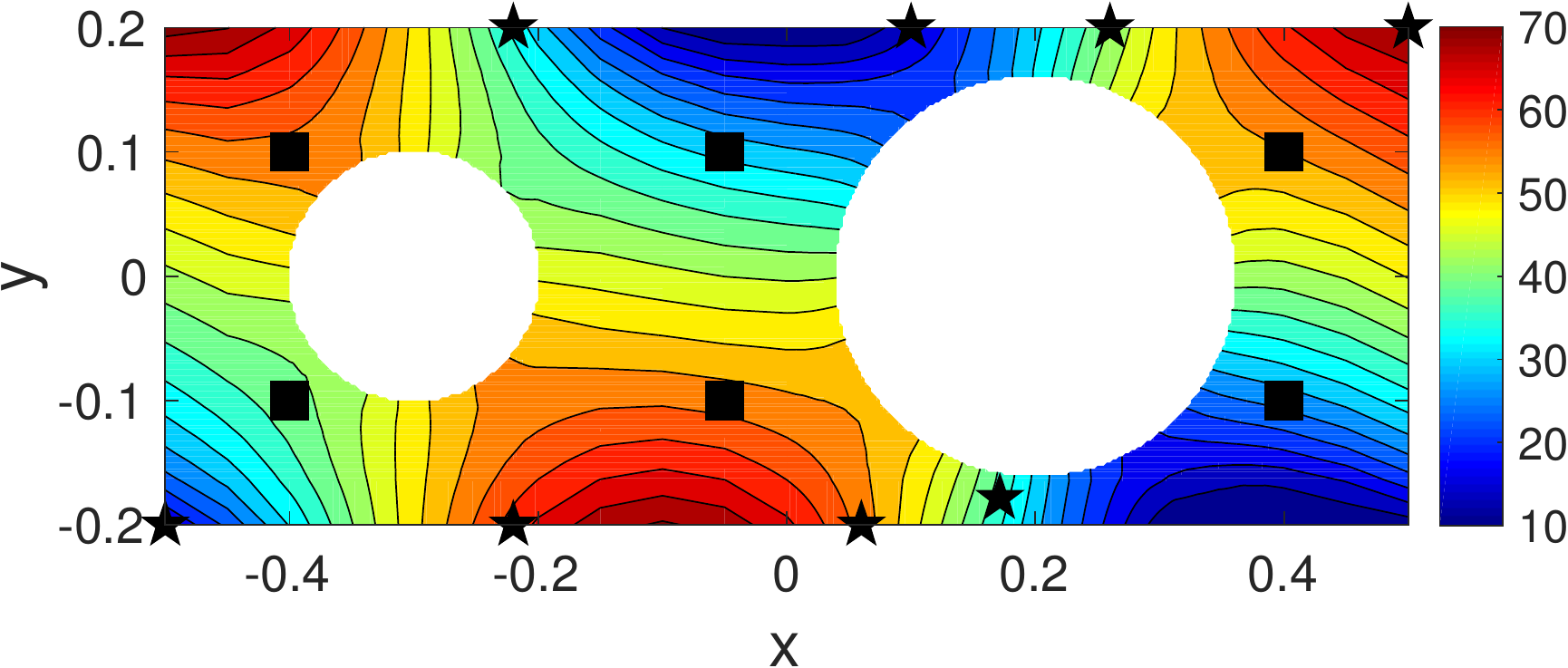}
    \caption{CoPhIK $\tensor F_r$}
  \end{subfigure}\quad
  \begin{subfigure}[b]{0.32\textwidth}
    \centering%
    \includegraphics[width=\textwidth]{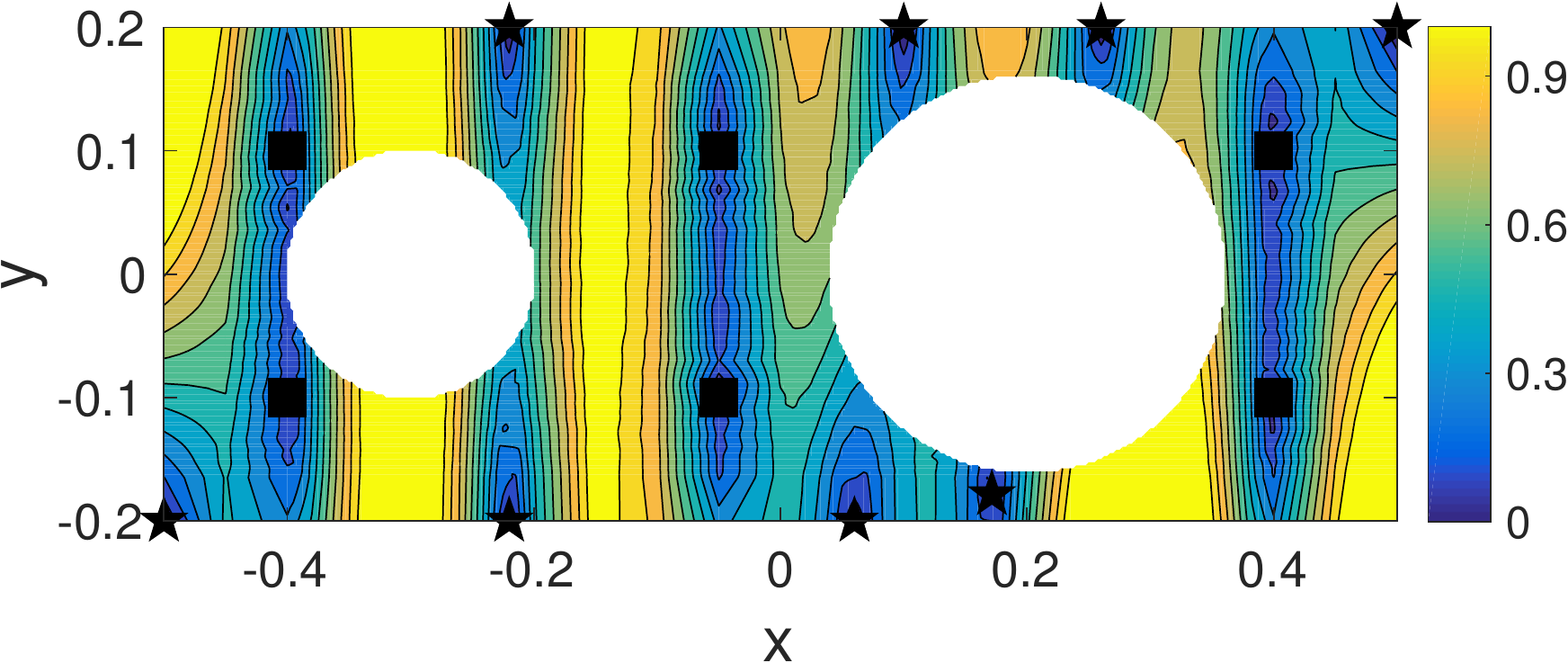}
    \caption{CoPhIK $\hat s$}
  \end{subfigure} 
  \begin{subfigure}[b]{0.32\textwidth}
    \centering%
    \includegraphics[width=\textwidth]{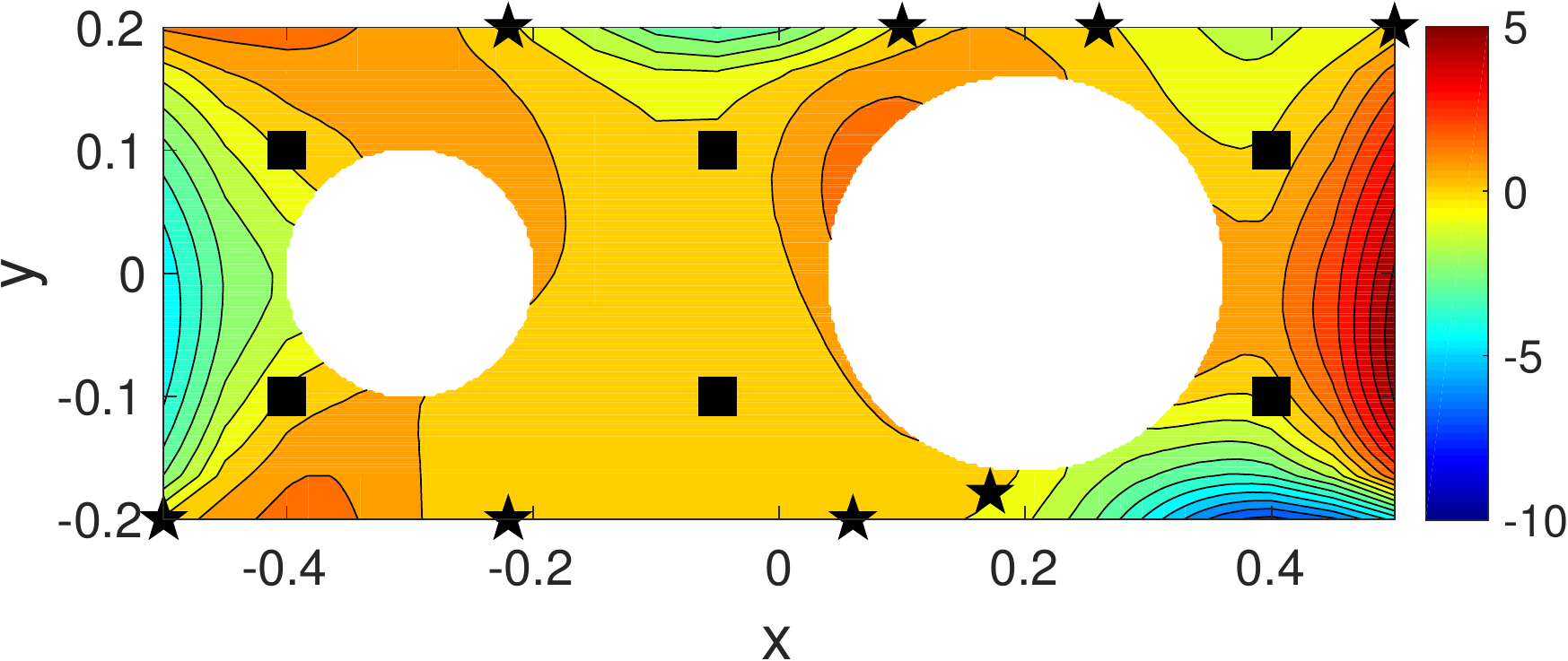}
    \caption{CoPhIK $\tensor F_r-\tensor F$}
  \end{subfigure}
  \caption{Reconstruction of the heat transfer via active learning. Black
  squares are the locations of the original six observation. Stars are newly
  added eight observations based on the actively learning algorithm.}
  \label{fig:heat_act}
\end{figure}
\begin{figure}[!h]
  \centering%
  \includegraphics[width=0.4\textwidth]{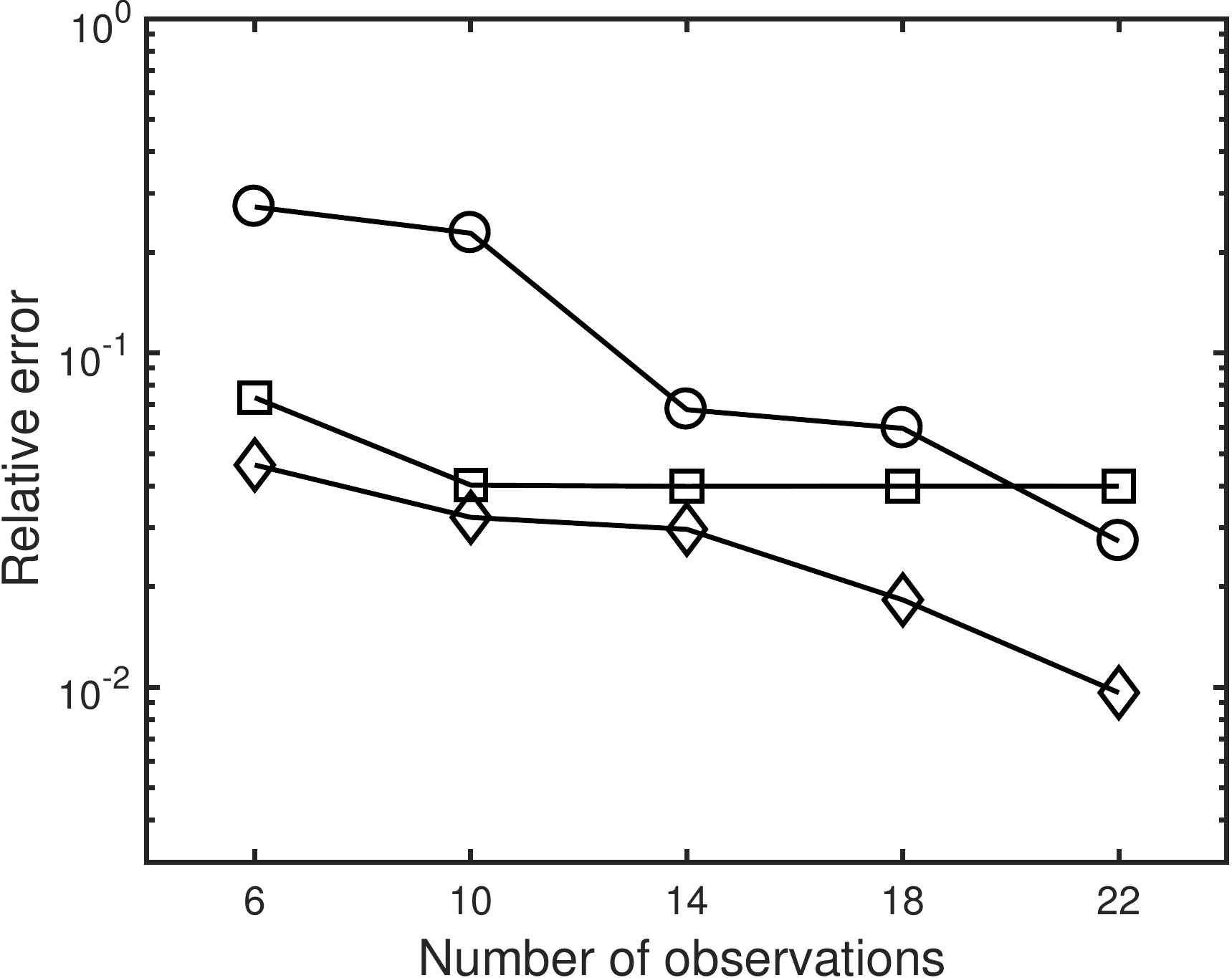}
  \caption{Relative error of reconstructed steady state solution for heat 
  transfer problem $\Vert\tensor F_r-\tensor F\Vert_F/\Vert\tensor F\Vert_F$
  using Kriging (``$\circ$"), PhIK (``$\square$") and CoPhIK (``$\diamond$") 
  with different numbers of total observations via active learning.}
  \label{fig:heat_rel_err}
\end{figure}

\subsection{Solute transport in heterogeneous porous media}

In this example, we consider conservative transport in a steady-state velocity
field in heterogeneous porous media. Let $C(\bm x,t)$ ($\bm x=(x,y)^{\trans}$)
denote the solute concentration. We assume that measurements of $C(\bm x,t)$
are available at several locations at different times. The flow and transport
processes can be described by conservation laws. In particular, the flow is
described by the Darcy flow equation
\begin{equation}
  \label{eq:head}
  \begin{cases}
    \nabla\cdot(K\nabla h)=0,
    & \bm x  \in  \mathbb{D}, \\
    \dfrac{\partial h}{\partial \bm n} = 0, & y=0 ~\text{or}~ y=L_2, \\
    h = H_1 & x=0,\\
    h = H_2 & x=L_1,
  \end{cases}
\end{equation}
where $h(\bm x;\omega)$ is the hydraulic head, 
$\mathbb{D}=[0,L_1] \times [0, L_2]$, $L_{1}=256$, $L_2=128$ is the simulation
domain, $H_1$ and $H_2$ are known boundary head values, and $K(\bm x)$ is the
unknown hydraulic conductivity field. This field is modeled as a random 
log-normally distributed field $K(\bm x;\omega) = \exp(Z(\bm x; \omega))$, where
$Z(\bm x; \omega)$ is a second-order stationary GP with known exponential 
covariance function 
$\cov\{Z(\bm x),Z(\bm x')\}=\sigma^2_Z\exp(-|\bm x - \bm x'|/l_z)$, variance 
$\sigma^2_Z=2$, and correlation length $l_z=5$. The solute transport is governed
by the advection-dispersion equation~\cite{emmanuel2005mixing, lin2009efficient}:
\begin{equation}
  \label{eq:con}
  \begin{cases}
    \dfrac{\partial C}{\partial t} + \nabla\cdot(\bm v C) = \nabla\cdot\left[\left(\dfrac{D_w}{\tau}+ \boldsymbol{\alpha} \Vert\bm v\Vert_2\right)\nabla C\right],
    & \bm x~\text{in}~\mathbb{D}, \\
    C = Q \delta(\bm x-\bm x^*), & t=0, \\
    \dfrac{\partial C}{\partial\bm n} = 0, & y=0
    ~\text{or}~ y=L_2~\text{or}~x=L_1, \\
    C = 0, & x = 0.
  \end{cases}
\end{equation}
Here, $C(\bm x,t;\omega)$ is the solute concentration defined on 
$\mathbb{D}\times [0, T]\times\Omega$; $\bm v$ is the fluid velocity given by
$\bm v(\bm x;\omega)=-K(\bm x;\omega)\nabla h(\bm x;\omega)/\phi$, where $\phi$
is the porosity; $D_w$ is the diffusion coefficient; $\tau$ is the tortuosity;
and $\boldsymbol{\alpha}$ is the dispersivity tensor with the diagonal 
components $\alpha_L$ and $\alpha_T$. In the present work, the transport 
parameters are set to $\phi=0.317$, $\tau=\phi^{1/3}$, $D_w=2.5 \times 10^{-5}$,
$\alpha_L=5$, and $\alpha_T=0.5$. Finally, the solute is instantaneously 
injected at $\bm x^*=(50, 64)$ at $t=0$ with the intensity $Q=1$. 

We are interested in reconstructing the concentration field at $T = 192$ (eight
days) from sparse observations collected at $t = T$. We generate $M$
realizations of $Z(\bm x)$ using the SGSIM (sequential Gaussian simulation)
code~\cite{deutsch1992gslib}, and solve the governing equations for each
realization of $K(\bm x)=\exp(Z(\bm x))$ using the finite volume code STOMP
(subsurface transport over multiple phases)~\cite{white2006stomp} with grid 
size $1\times 1$. The ground truth $C_e(\bm x, T)$ is randomly selected from 
the $M$ realizations of $C(\bm x, T)$; this $C_e$ is excluded from the ensembles
used in PhIK or CoPhIK. Figure~\ref{fig:adv_truth} shows $C_e(\bm x, T)$ with 
sparse observation locations marked by black squares. We assume that six 
uniformly spaced observations are available near the boundary of the simulation
domain, and nine randomly placed observations are available in the interior of 
the domain. As Kriging is known to be less accurate for extrapolation, it is a
common practice to collect data near the boundary of the domain of interest
(e.g., \cite{dai2017geostatistics}).
\begin{figure}[h]
  \centering%
  \includegraphics[width=0.5\textwidth]{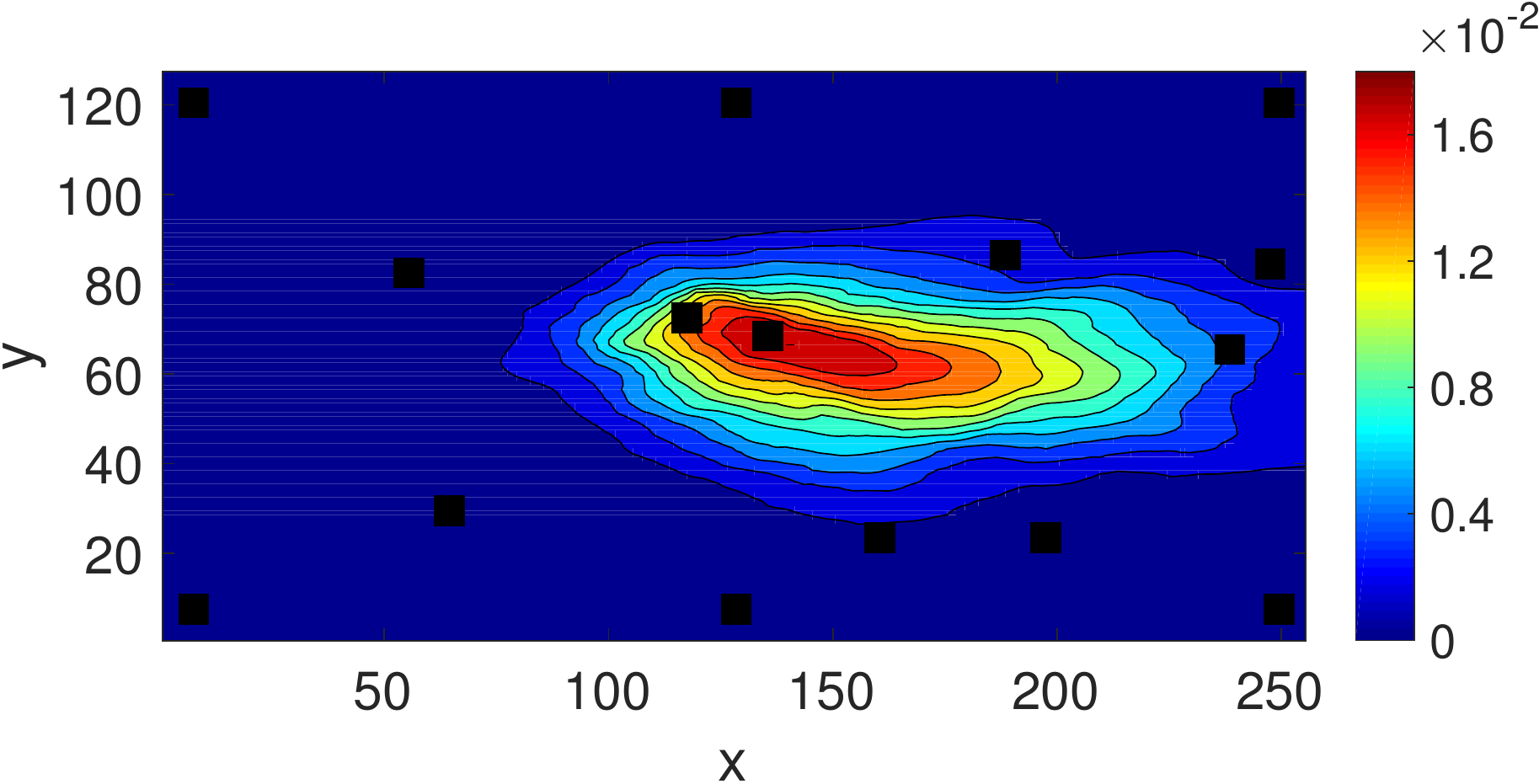}
  \caption{Ground truth of the solute concentration when $T=192$ and observation
  locations (black squares).}
  \label{fig:adv_truth}
\end{figure}

The first row in Figure~\ref{fig:adv_krig_mc} shows the reconstruction results
obtained using Kriging with Gaussian kernel. The relative error is nearly 
$50\%$ because a stationary kernel is not capable of resolving the reference
field accurately. The second row displays the ensemble mean, standard deviation
and the difference of mean and the ground truth, estimated from the stochastic
flow and advection-dispersion equation without conditioning on data. Solving 
these stochastic equations with standard MC requires a large number of 
simulations. To reduce computational cost of MC, we use MLMC (described in 
Appendix A) to compute mean and variance of $C$ as in~\cite{YangTT18}. Later,
we also use MLMC to compute the covariance of $C$. In MLMC, we use $M_H=10$
high-resolution simulations (grid size $1\times 1$) and $M_L=150$ 
low-resolution simulations (grid size $4\times 4$). The MLMC mean is almost
symmetric which does not reflect the real pattern of the ground truth, which
is not symmetric). The relative error of using ensemble mean to estimate the 
ground truth is $30\%$. This figure shows that Kriging prediction is less 
accurate and have larger predictive uncertainty except for the neighborhoods of
the observations. 
\begin{figure}[!h]
  \centering%
  \begin{subfigure}[b]{0.32\textwidth}
    \centering%
    \includegraphics[width=\textwidth]{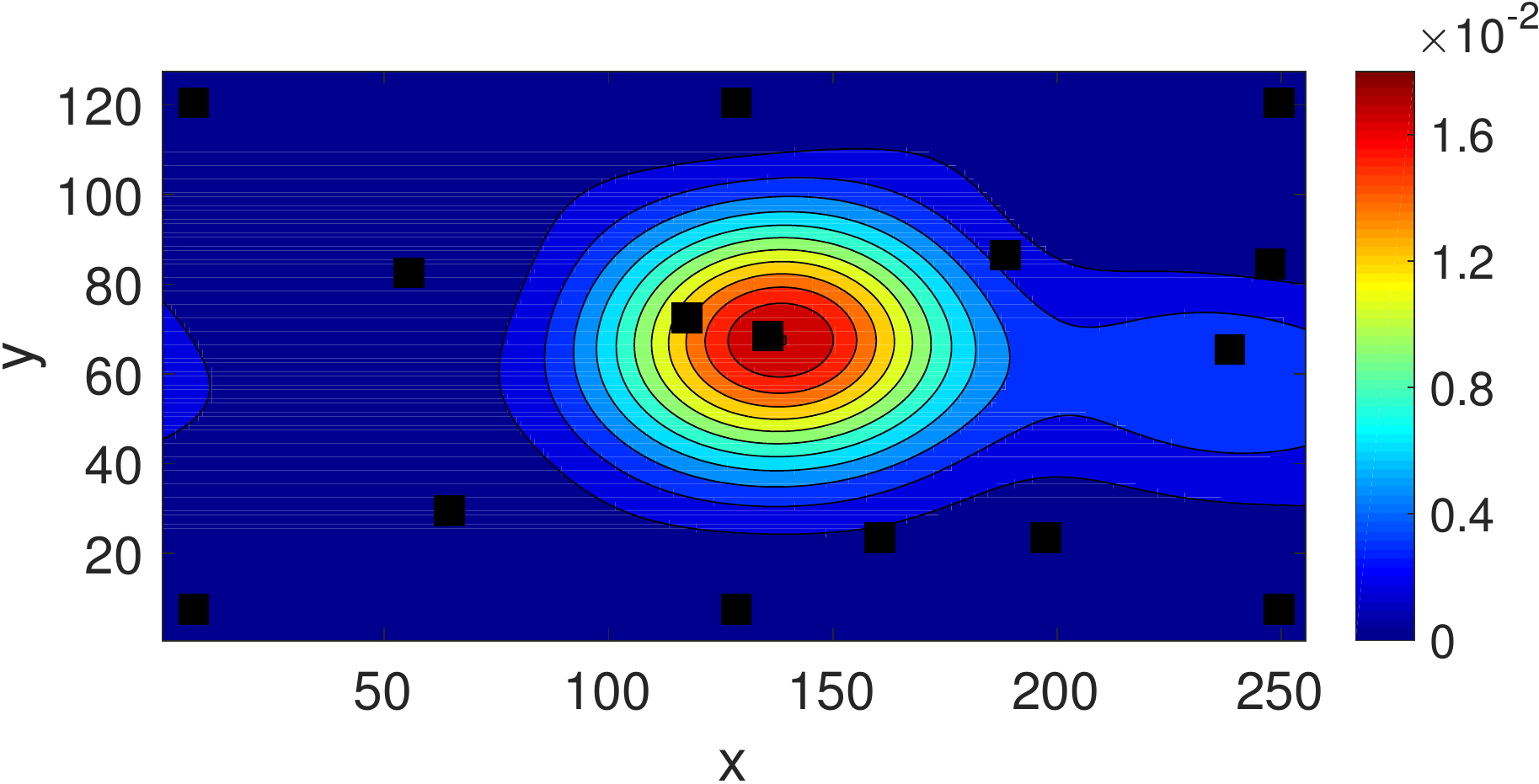}
    \caption{Kriging $\tensor F_r$}
  \end{subfigure}
  \begin{subfigure}[b]{0.32\textwidth}
    \centering%
    \includegraphics[width=\textwidth]{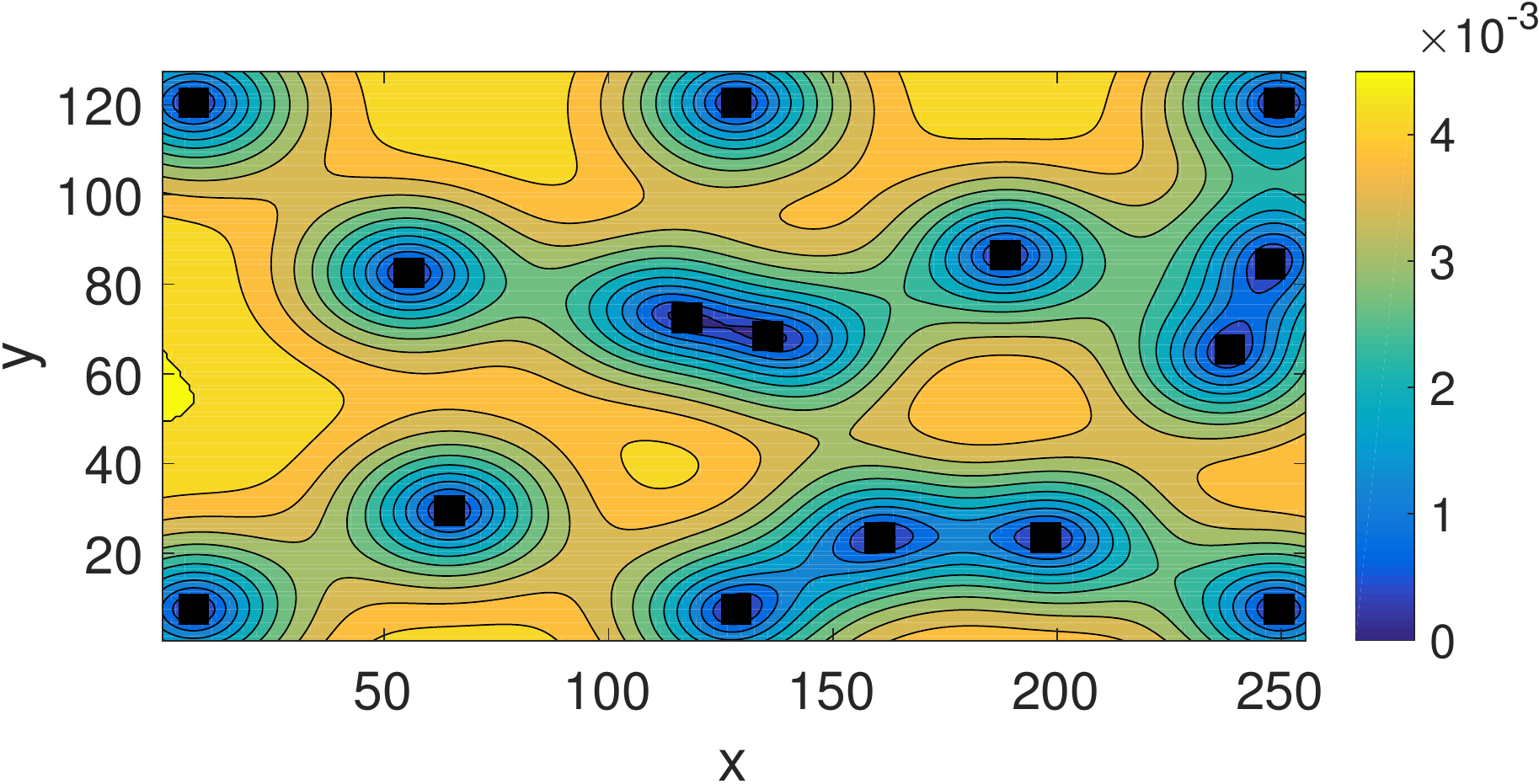}
    \caption{Kriging $\hat s$}
  \end{subfigure}
  \begin{subfigure}[b]{0.32\textwidth}
    \centering%
    \includegraphics[width=\textwidth]{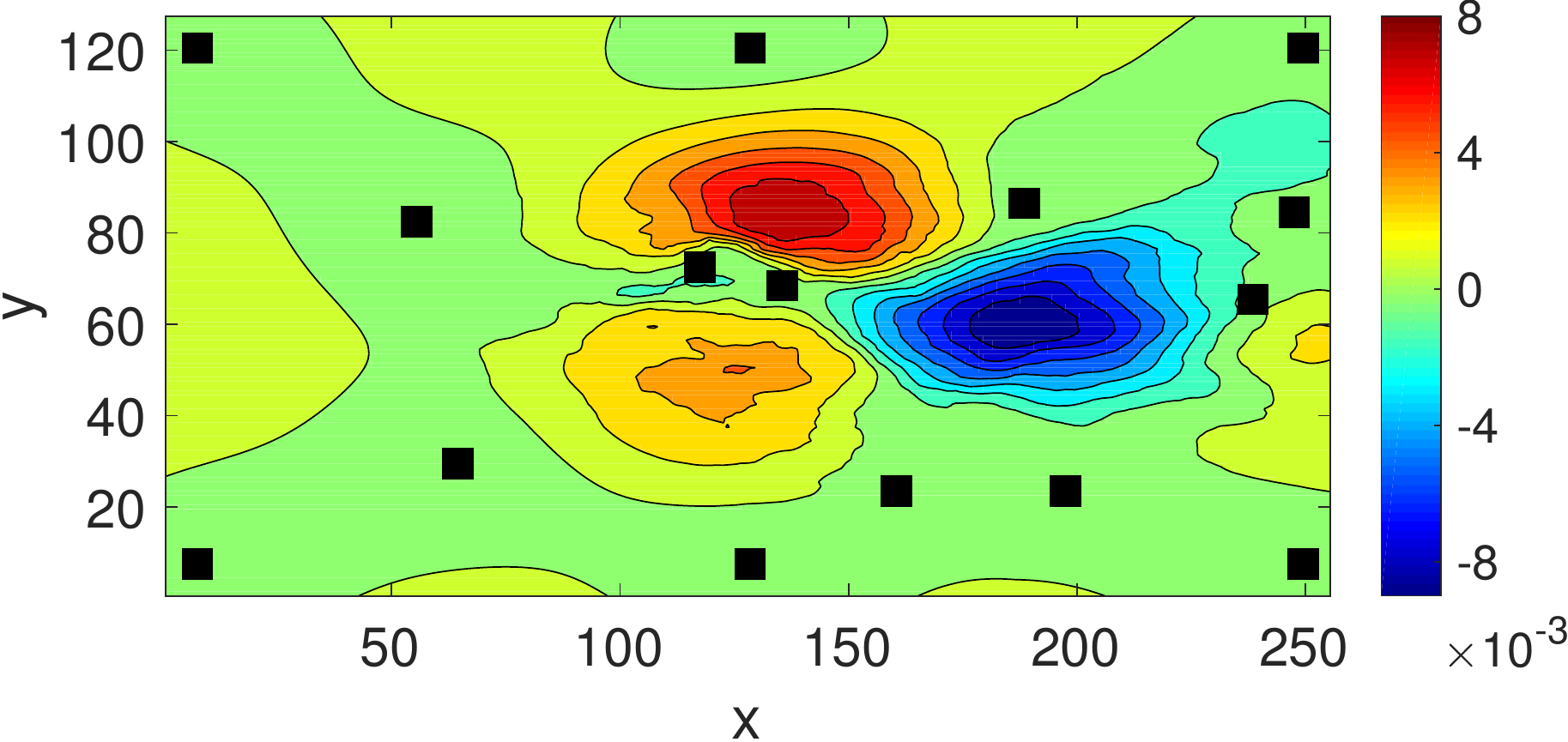}
    \caption{Kriging $\tensor F_r-\tensor F$}
  \end{subfigure}
  \begin{subfigure}[b]{0.32\textwidth}
    \centering%
    \includegraphics[width=\textwidth]{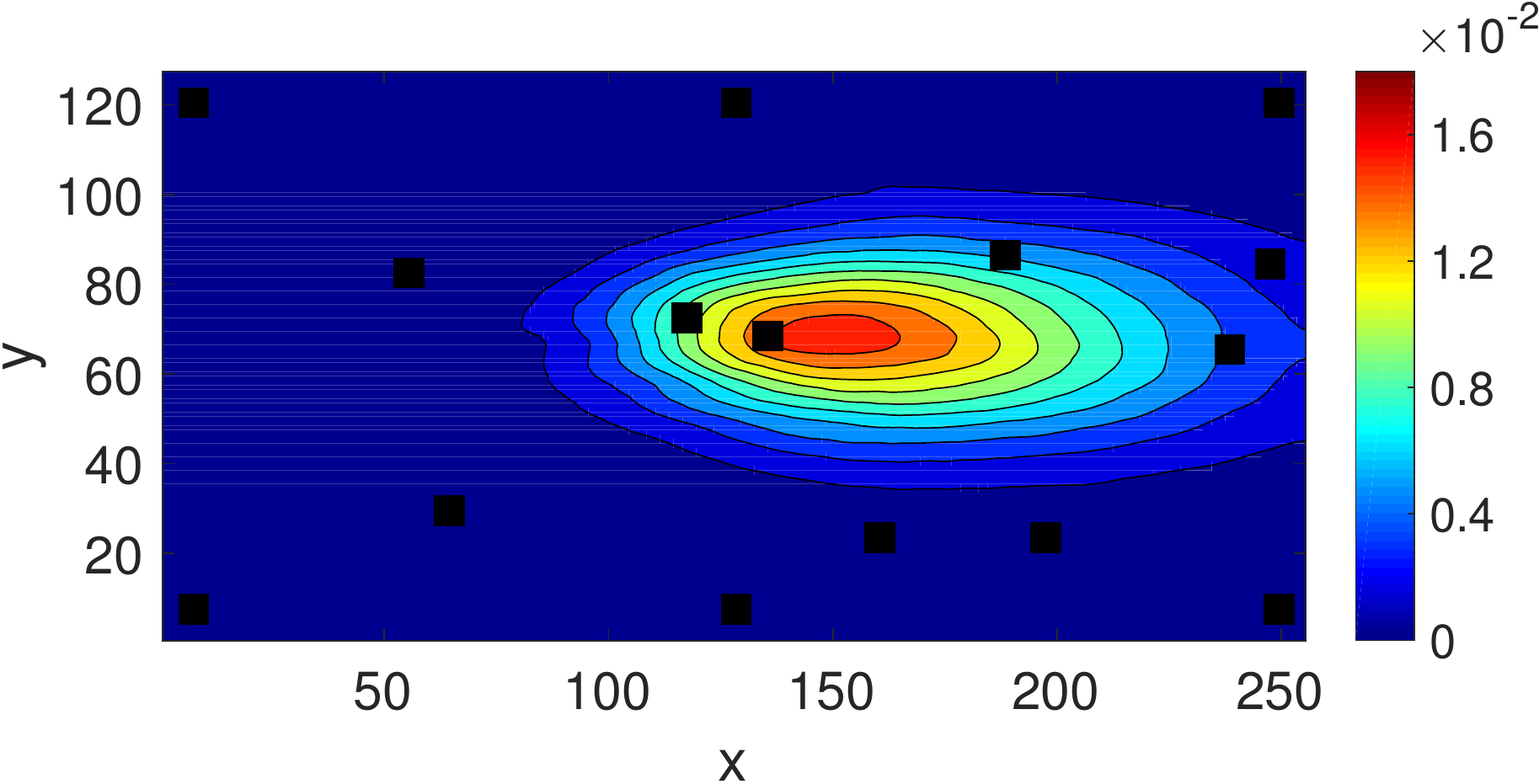}
    \caption{MLMC mean}
  \end{subfigure}
  \begin{subfigure}[b]{0.32\textwidth}
    \centering%
    \includegraphics[width=\textwidth]{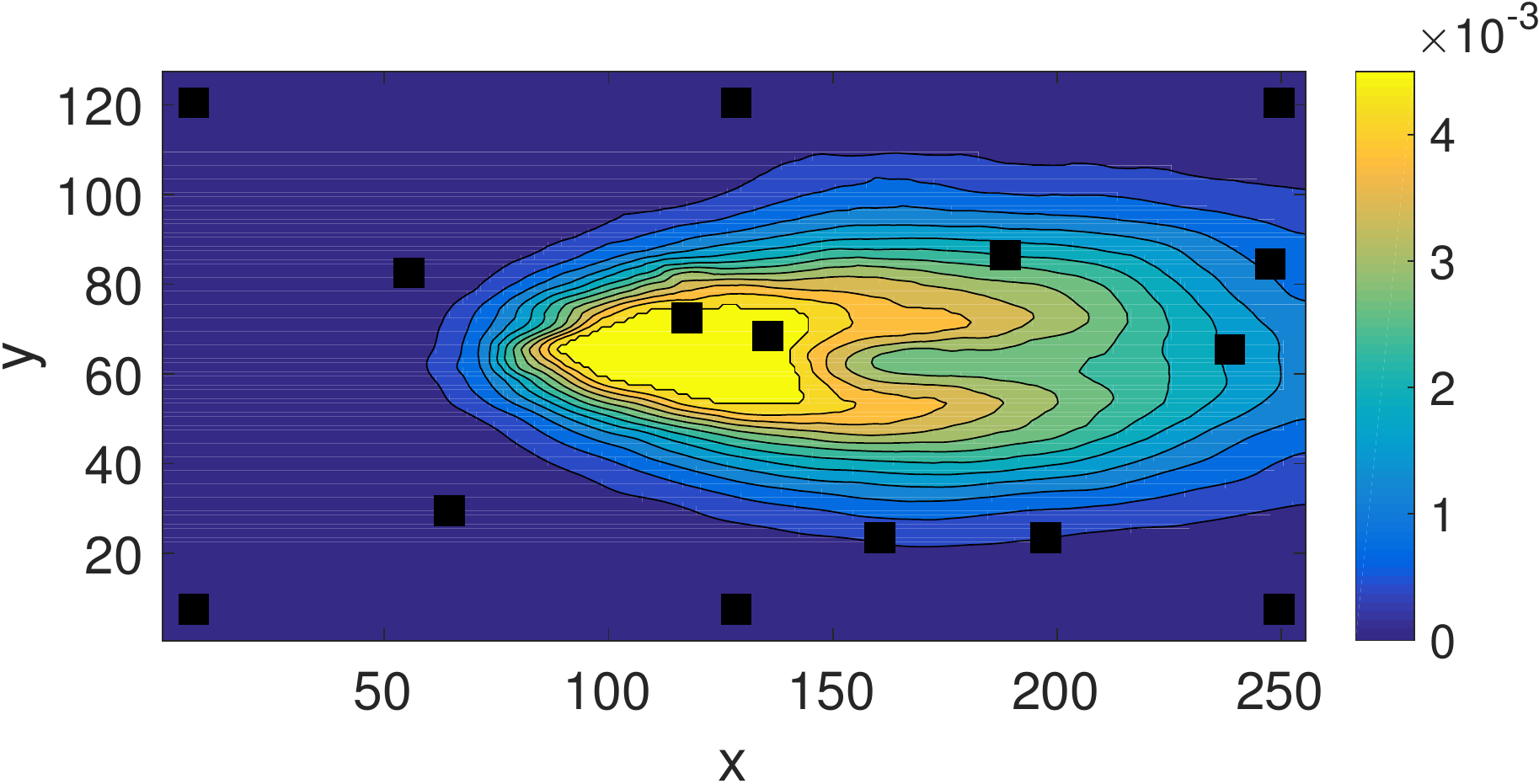}
    \caption{MLMC std.}
  \end{subfigure}
  \begin{subfigure}[b]{0.32\textwidth}
    \centering%
    \includegraphics[width=\textwidth]{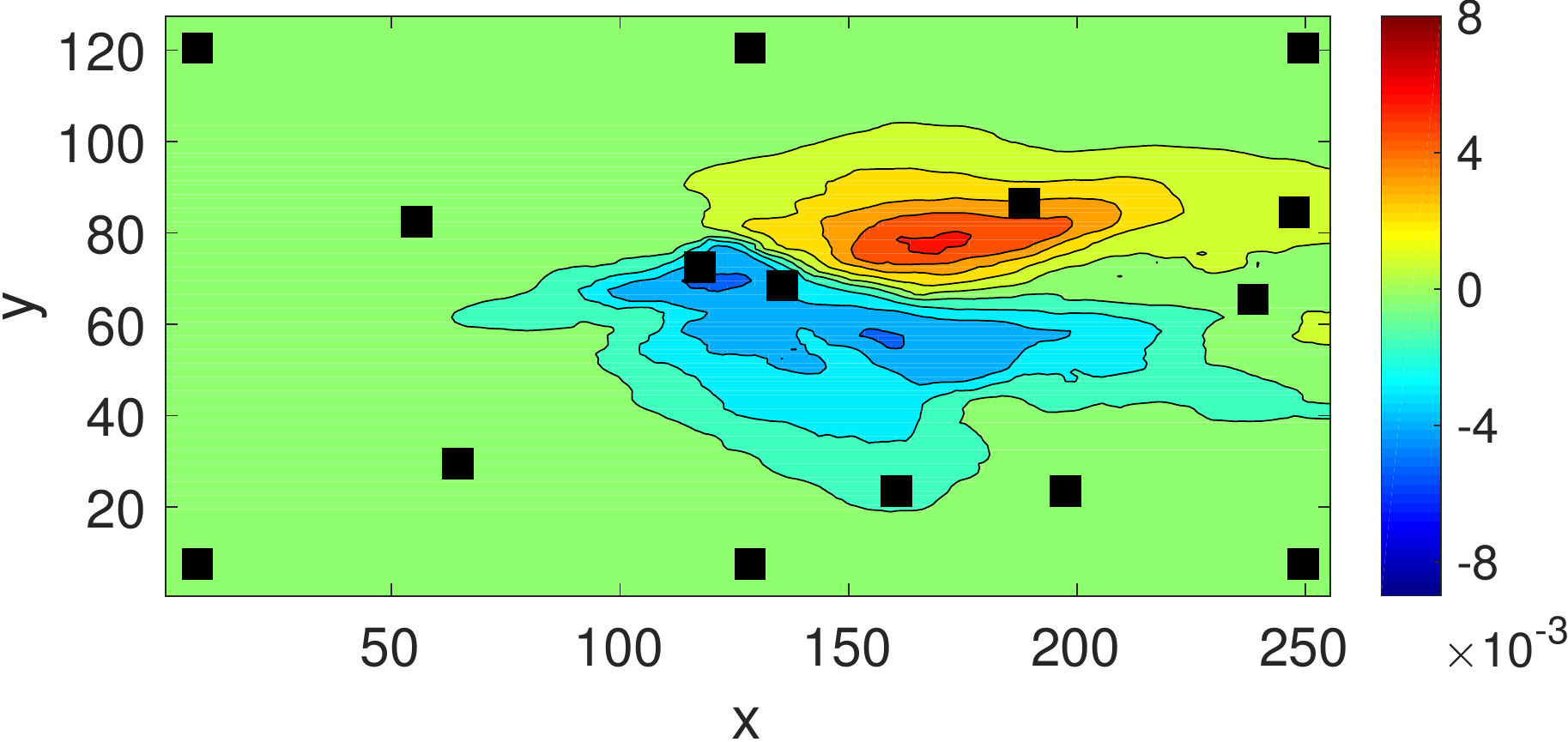}
    \caption{MLMC mean $-\tensor F$}
  \end{subfigure}
  \caption{Reconstruction of the solute concentration by Kriging (first row) and statistics of MLMC combining the ensemble $\{\hat{\tensor F}^m_L\}_{m=1}^{M_L}$ and $\{\hat{\tensor F}^m_H\}_{m=1}^{M_H}$ (second row).}
  \label{fig:adv_krig_mc}
\end{figure}

Figure~\ref{fig:adv_phik} shows $\tensor F_r$, $\hat s$, and $\tensor F_r-\bm F$
obtained with PhIK and CoPhIK. In this case, PhIK is more accurate than CoPhiK.
The reconstructed field $\tensor F_r$ from both methods are closer to the ground
truth than the Kriging results, as evident from smaller $\hat s$ and 
$\tensor F_r-\bm F$.
\begin{figure}[!h]
  \centering%
  \begin{subfigure}[b]{0.32\textwidth}
    \centering%
    \includegraphics[width=\textwidth]{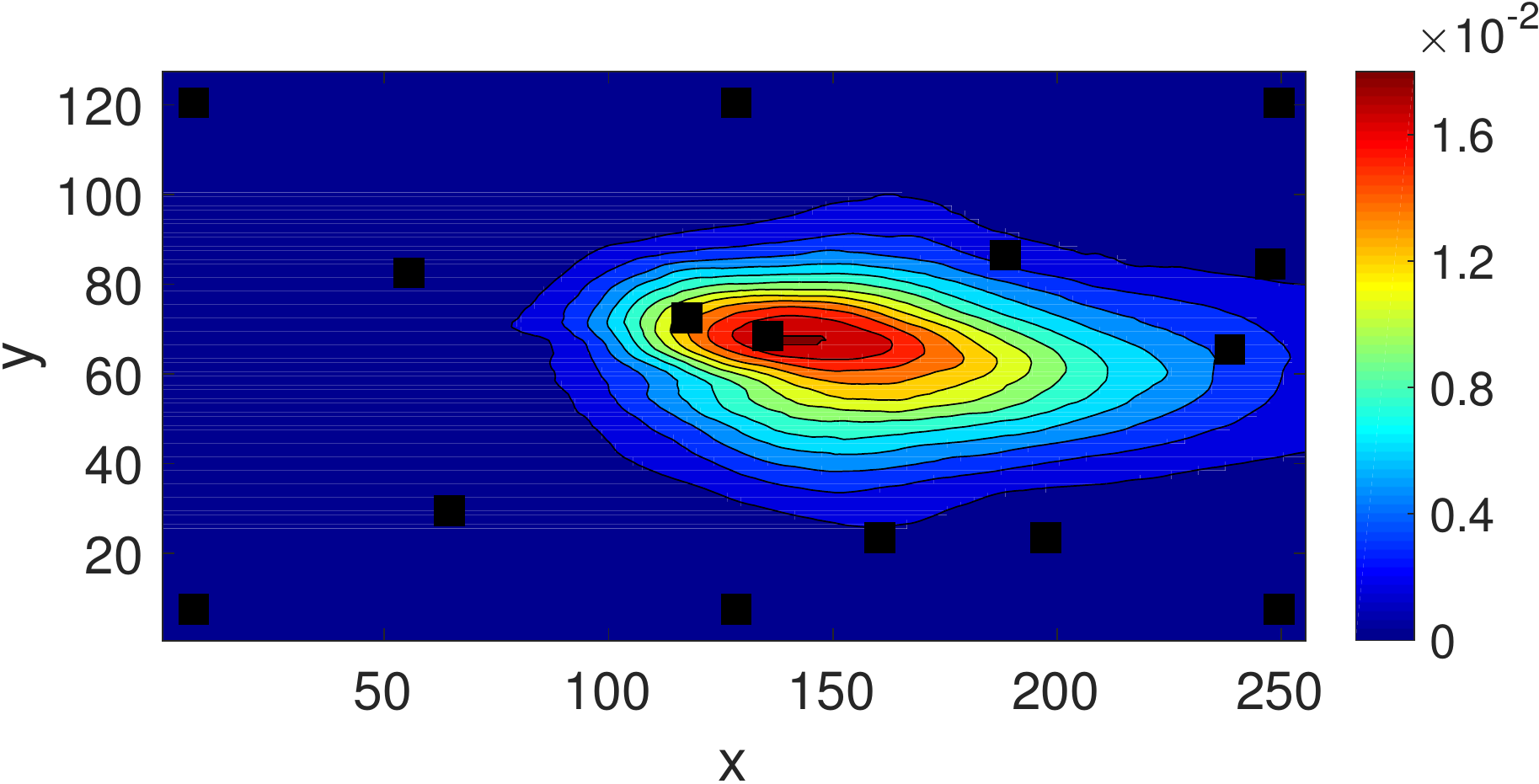}
    \caption{PhIK $\tensor F_r$}
  \end{subfigure}
  \begin{subfigure}[b]{0.32\textwidth}
    \centering%
    \includegraphics[width=\textwidth]{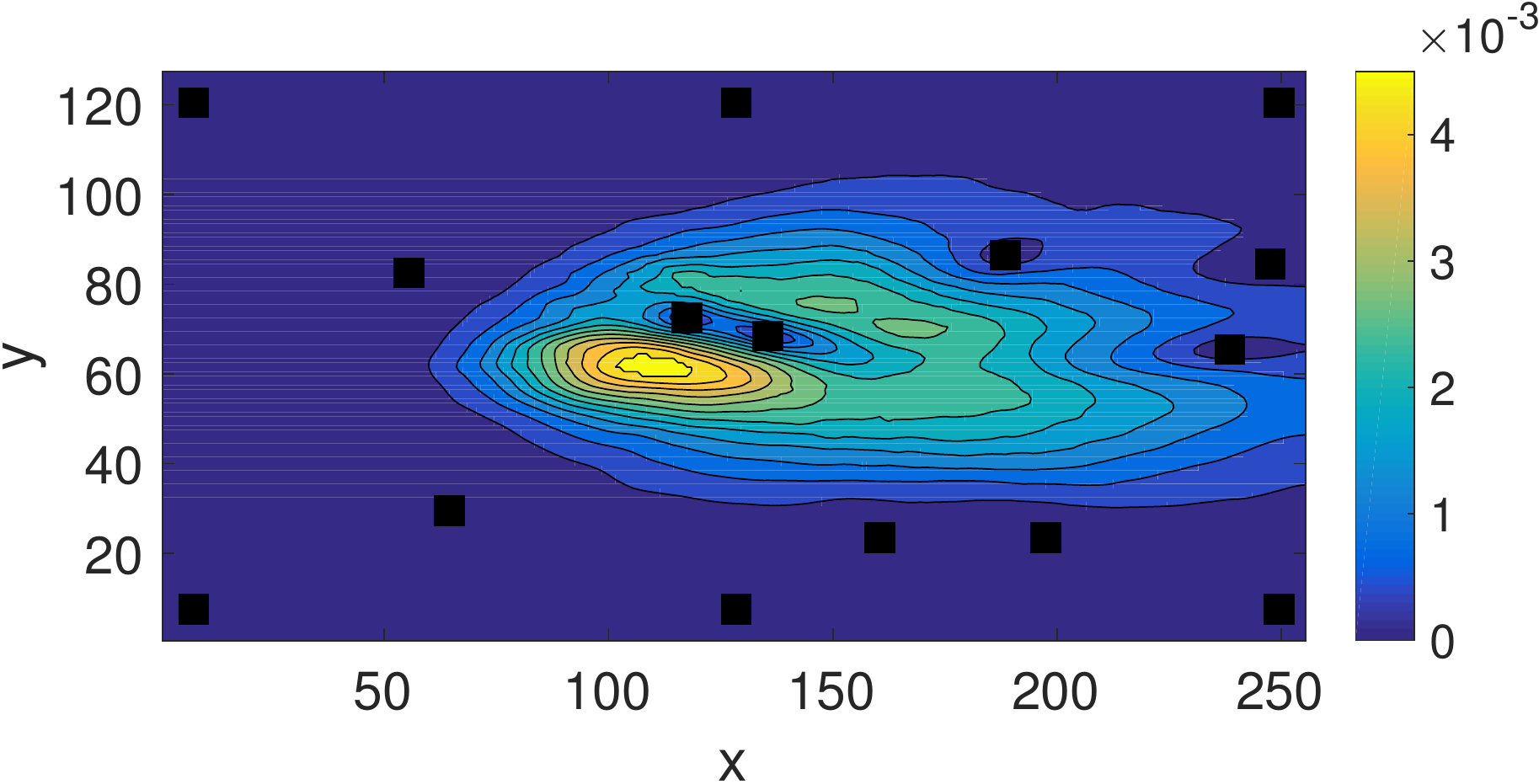}
    \caption{PhIK $\hat s$}
  \end{subfigure}
  \begin{subfigure}[b]{0.32\textwidth}
    \centering%
    \includegraphics[width=\textwidth]{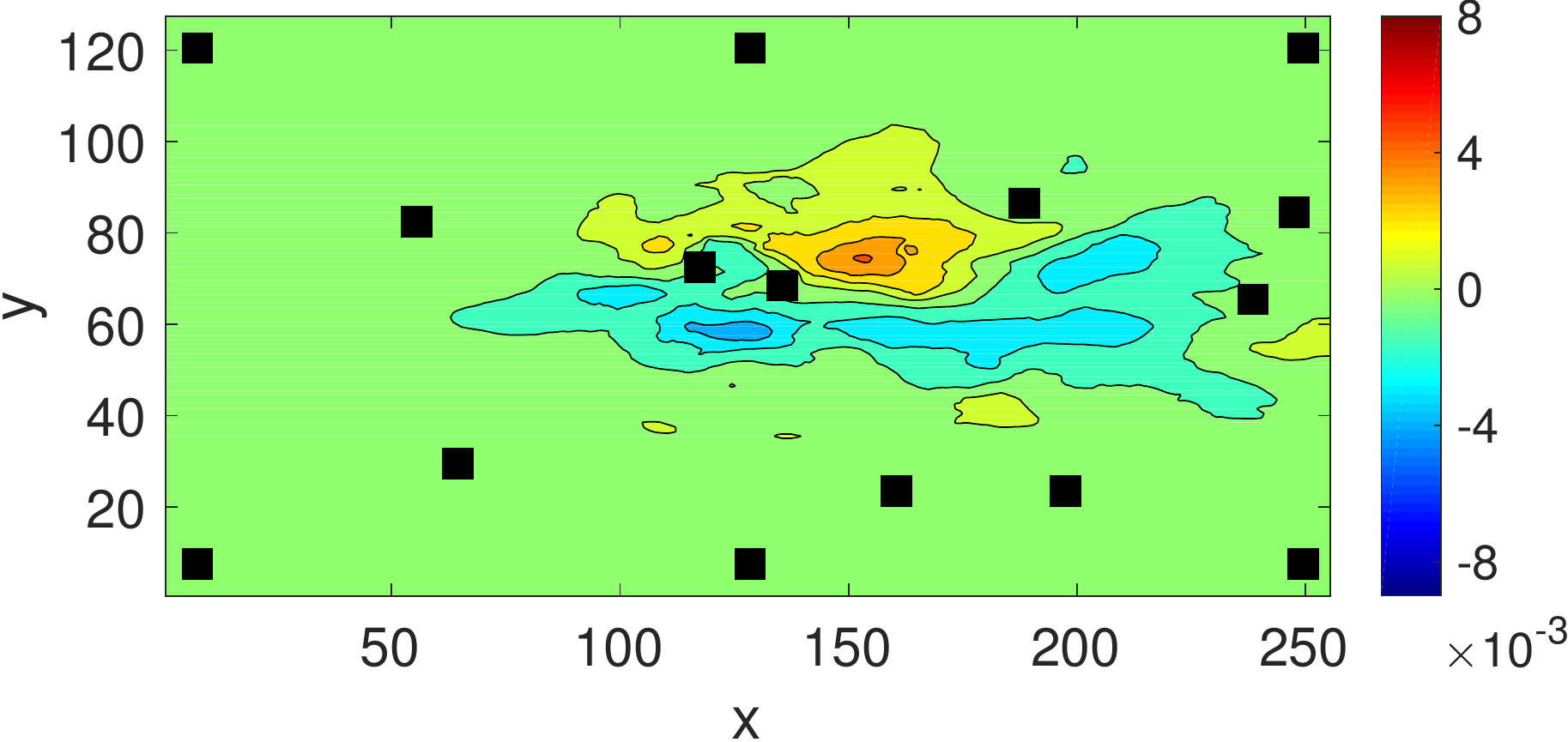}
    \caption{PhIK $\hat{\tensor F}_r - \tensor F$}
  \end{subfigure}
  \begin{subfigure}[b]{0.32\textwidth}
    \centering%
    \includegraphics[width=\textwidth]{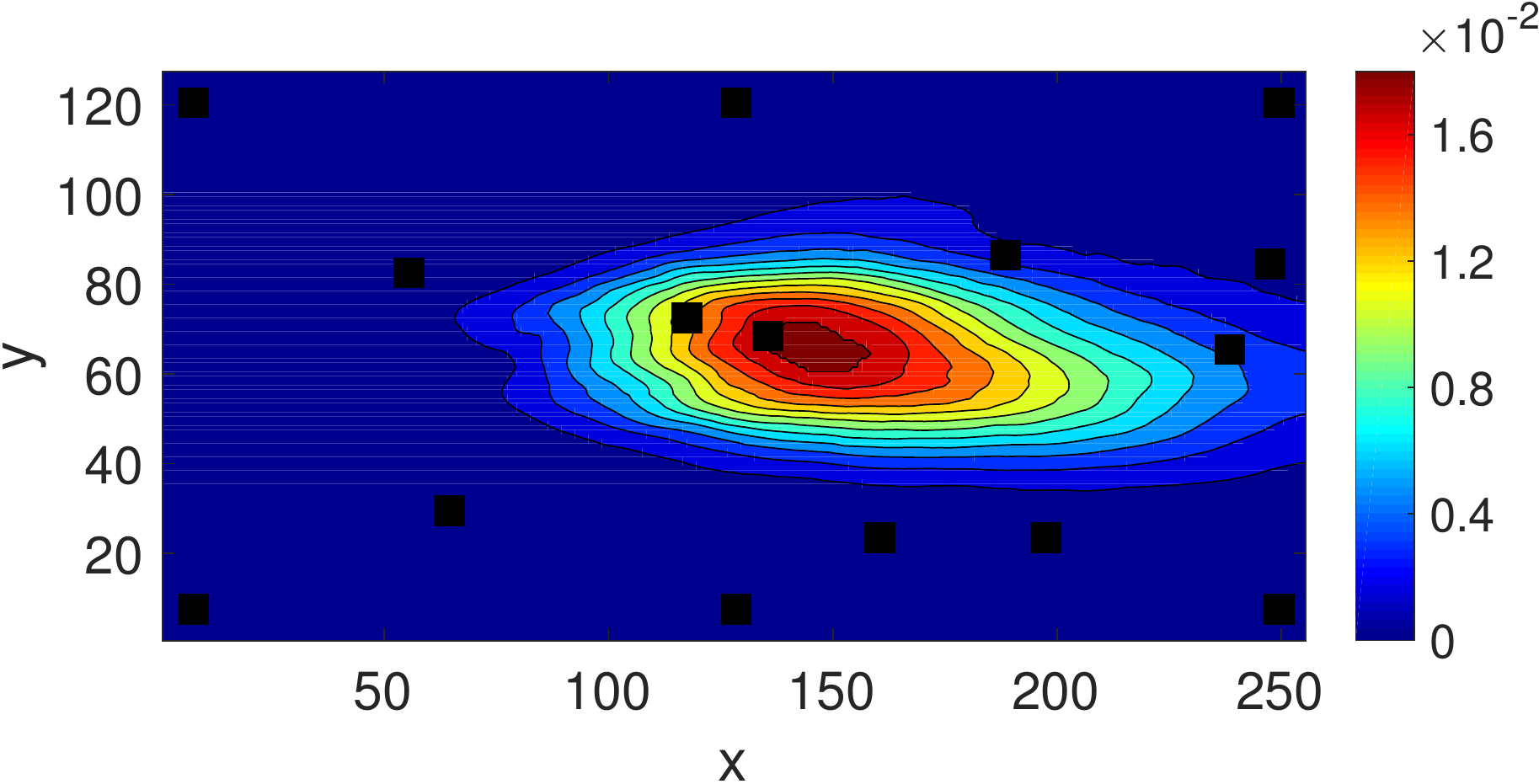}
    \caption{CoPhIK $\tensor F_r$}
  \end{subfigure}
  \begin{subfigure}[b]{0.32\textwidth}
    \centering%
    \includegraphics[width=\textwidth]{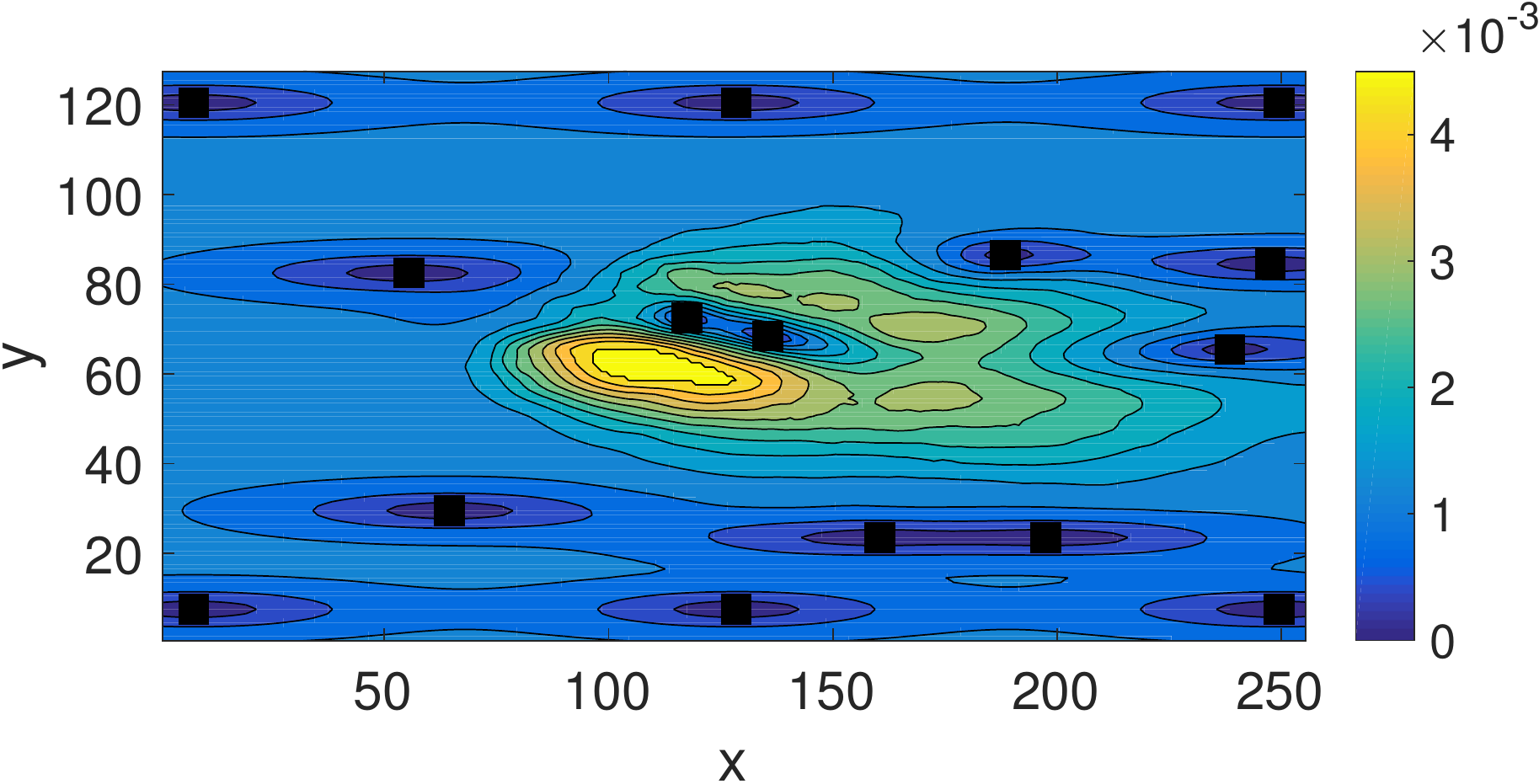}
    \caption{CoPhIK $\hat s$}
  \end{subfigure}
  \begin{subfigure}[b]{0.32\textwidth}
    \centering%
    \includegraphics[width=\textwidth]{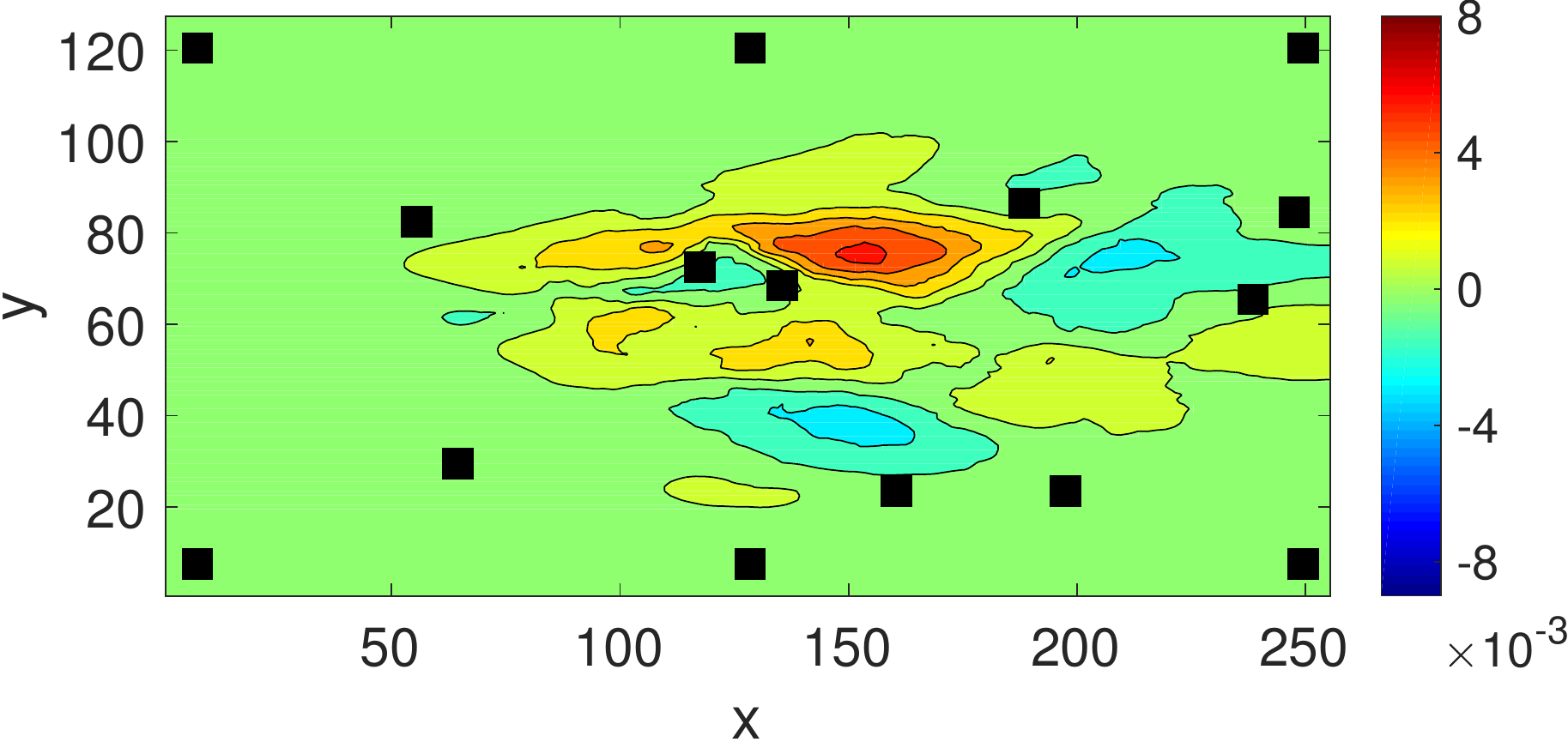}
    \caption{CoPhIK $\hat{\tensor F}_r - \tensor F$}
  \end{subfigure}
  \caption{Reconstruction of the solute concentration field using PhIK (first row) and CoPhIK (second row).}
  \label{fig:adv_phik}
\end{figure}

Finally, we employ the active learning algorithm~\ref{algo:act} to identify
additional observation locations. Figure~\ref{fig:adv_act} presents $15$ 
additional observation locations, indicated by black stars, identified using 
Kriging, PhIK and CoPhIK, and the resulting field reconstructions. Both PhIK and
CoPhIK outperforms Kriging, which can be seen qualitatively in terms of the
structure of the reconstructed plume and quantitative from the pointwise 
difference $\tensor F_r-\tensor F$. It can be seen that the additional 
observations identified by PhIK cluster around the plume, where the 
concentration is high, while Kriging distributes additional observations more
uniformly throughout the entire simulation domain. The behavior of CoPhIK is
between that of Kriging and PhIK, placing additional observations around the 
plume less tightly than PhIK but less spread out than Kriging.
\begin{figure}[!h]
  \centering%
  \begin{subfigure}[b]{0.32\textwidth}
    \centering%
    \includegraphics[width=\textwidth]{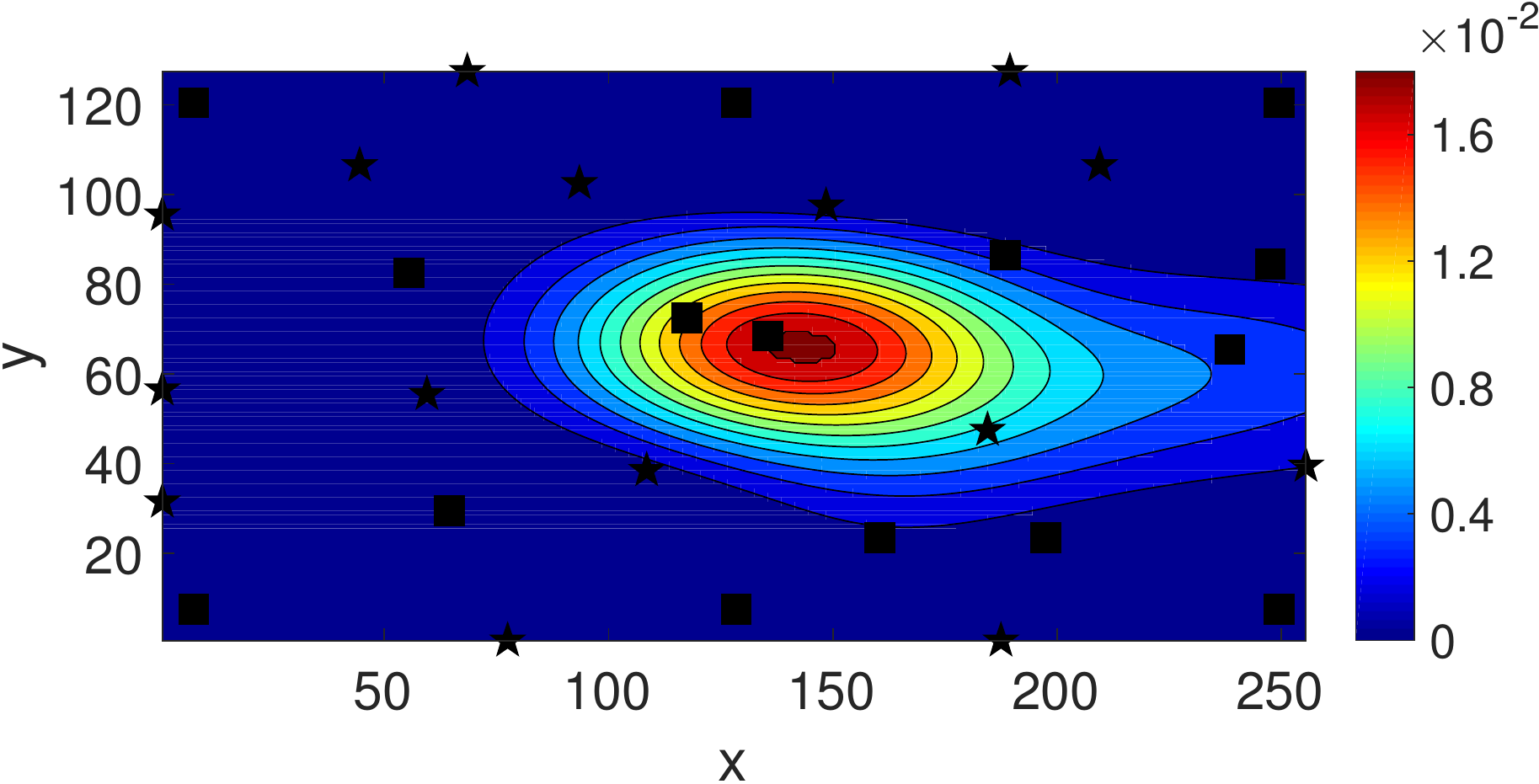}
    \caption{Kriging $\tensor F_r$}
  \end{subfigure}
  \begin{subfigure}[b]{0.32\textwidth}
    \centering%
    \includegraphics[width=\textwidth]{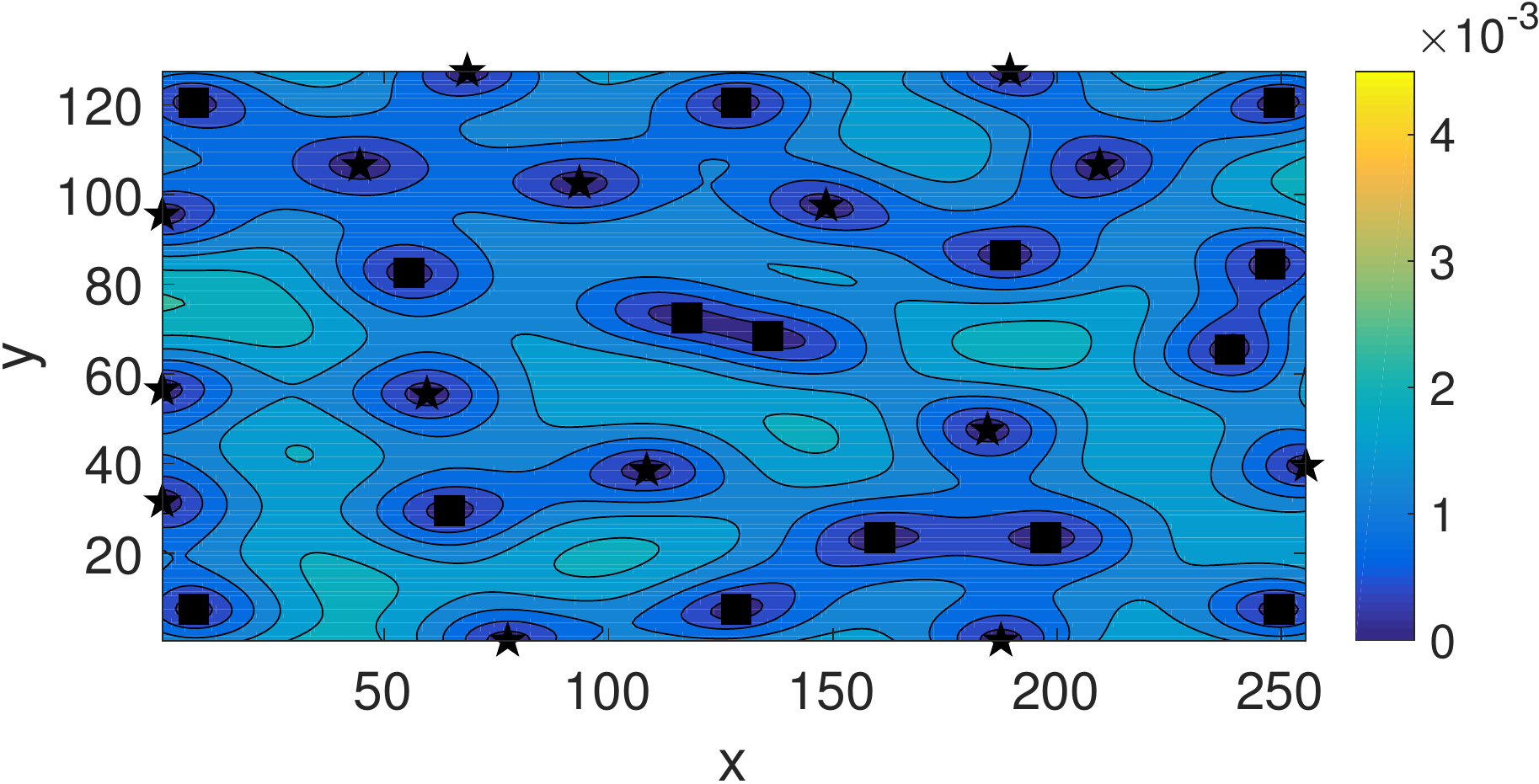}
    \caption{Kriging $\hat s$}
  \end{subfigure}
  \begin{subfigure}[b]{0.32\textwidth}
    \centering%
    \includegraphics[width=\textwidth]{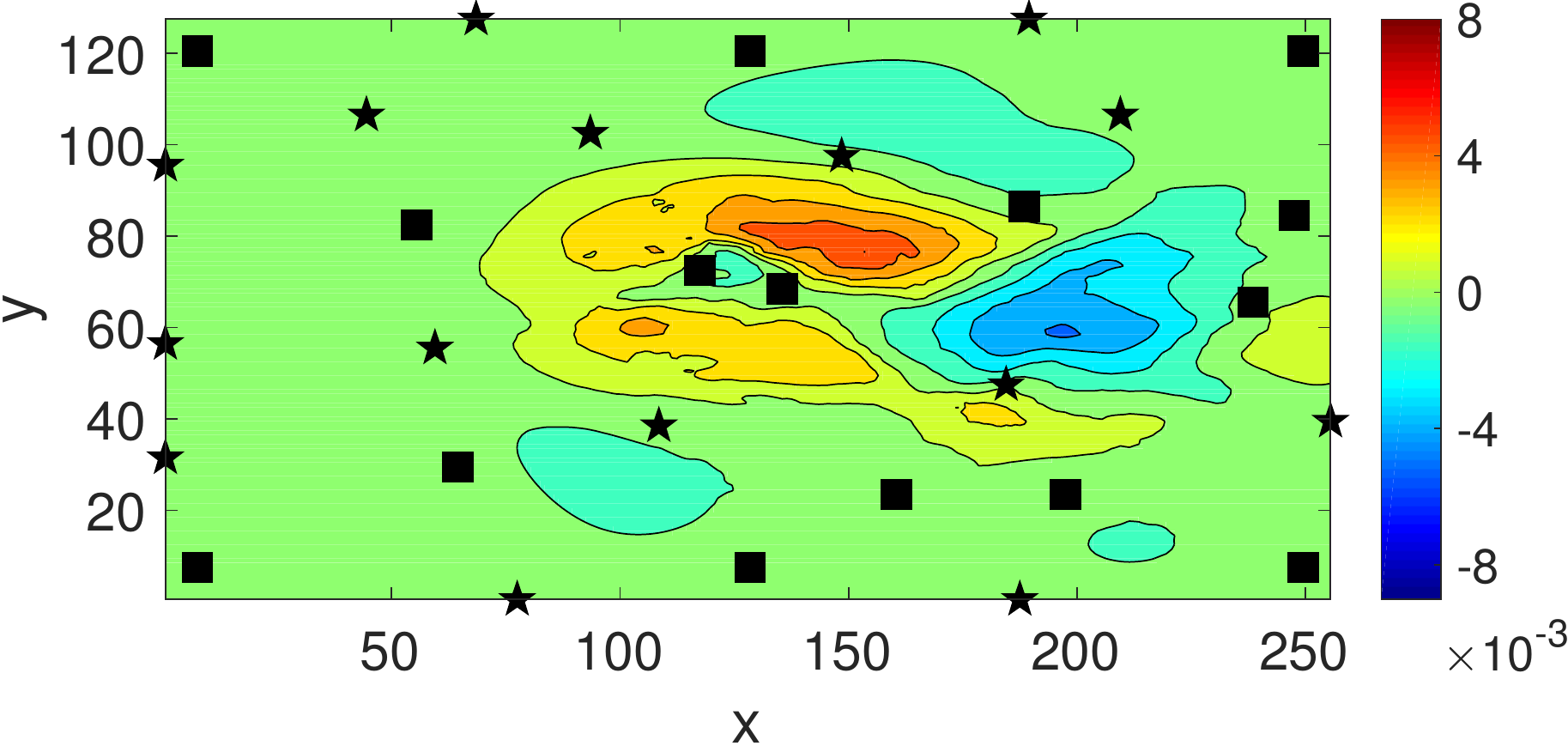}
    \caption{Kriging $\tensor F_r-\tensor F$}
  \end{subfigure}
  \begin{subfigure}[b]{0.32\textwidth}
    \centering%
    \includegraphics[width=\textwidth]{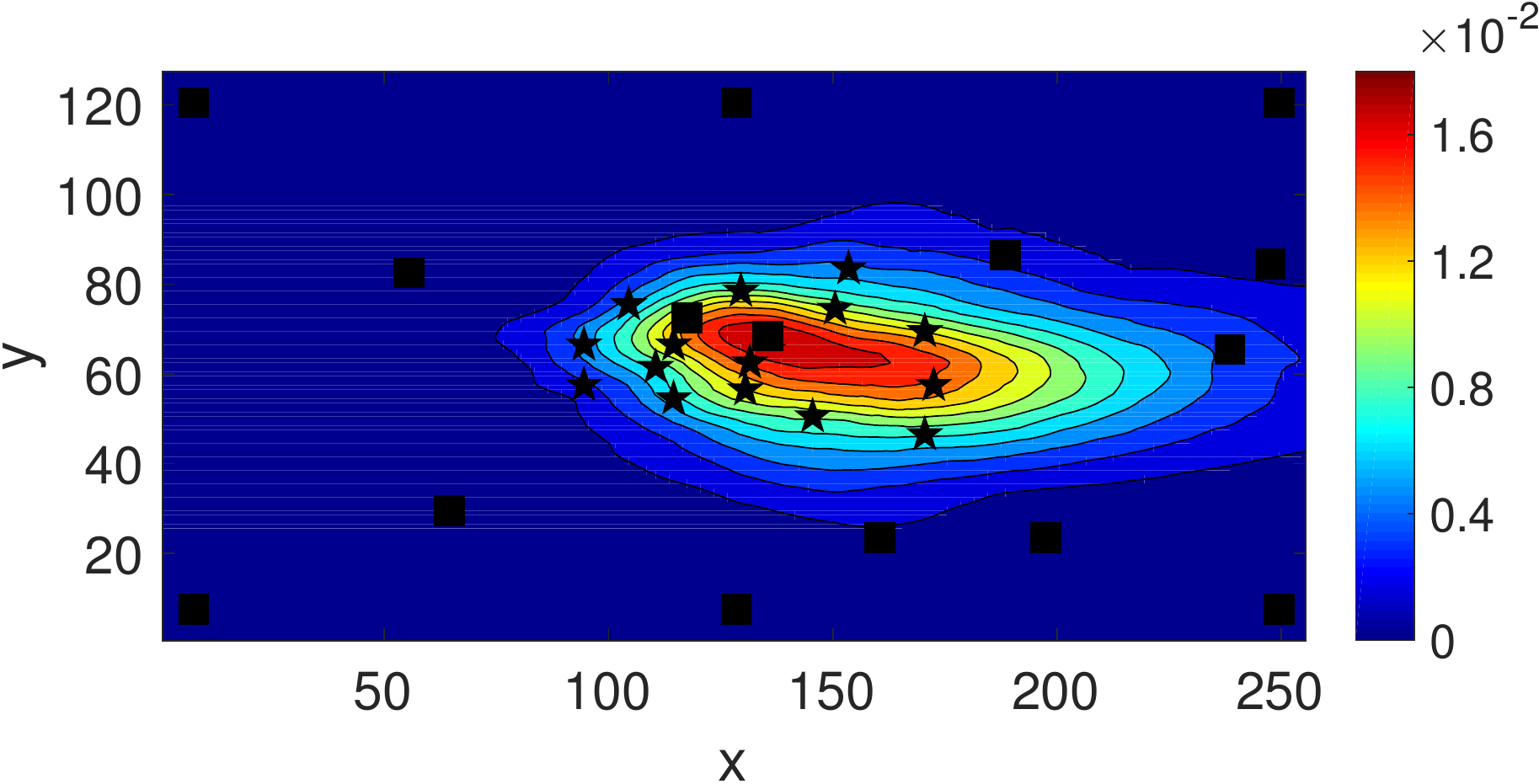}
    \caption{PhIK $\tensor F_r$}
  \end{subfigure}
  \begin{subfigure}[b]{0.32\textwidth}
    \centering%
    \includegraphics[width=\textwidth]{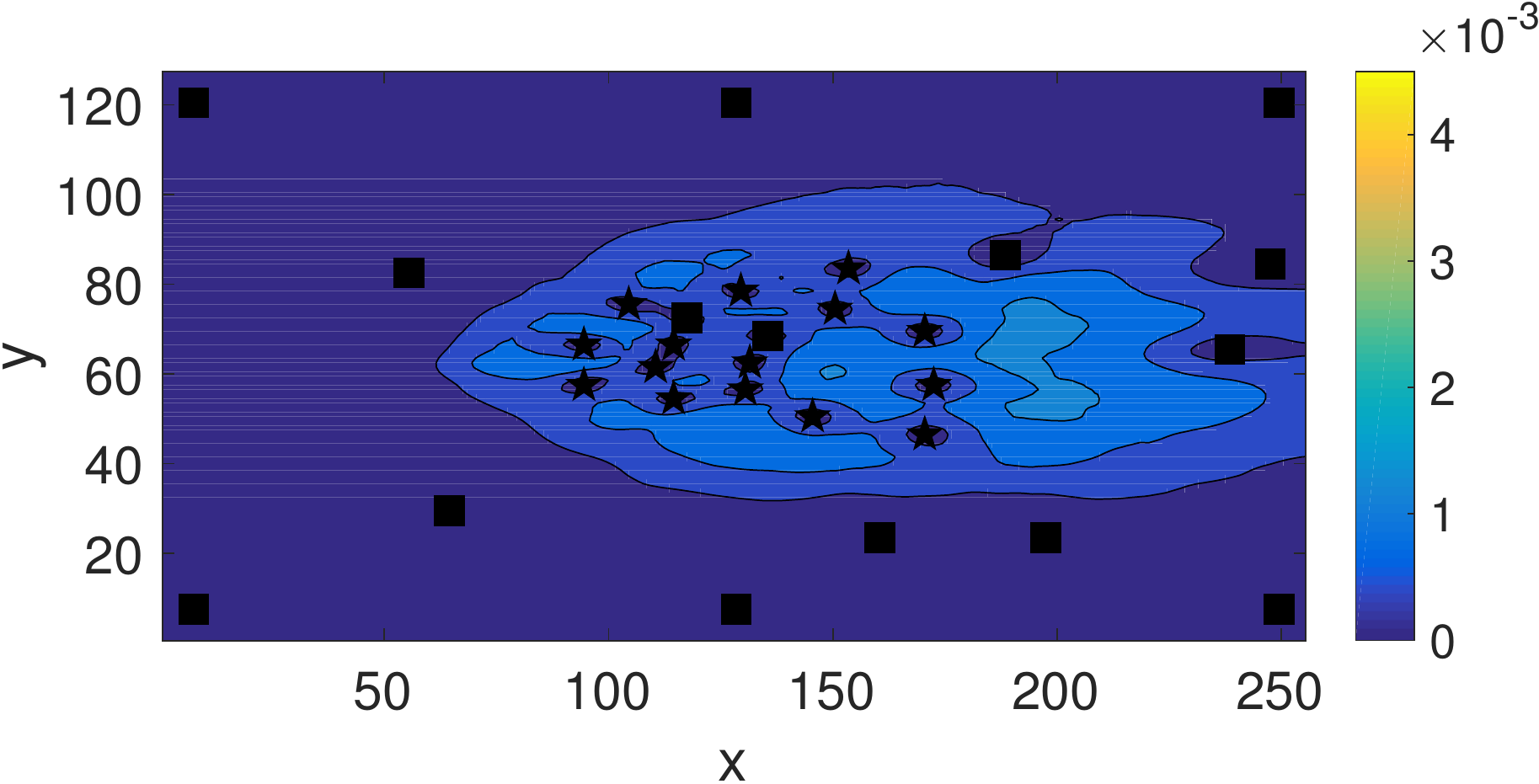}
    \caption{PhIK $\hat s$}
  \end{subfigure}
  \begin{subfigure}[b]{0.32\textwidth}
    \centering%
    \includegraphics[width=\textwidth]{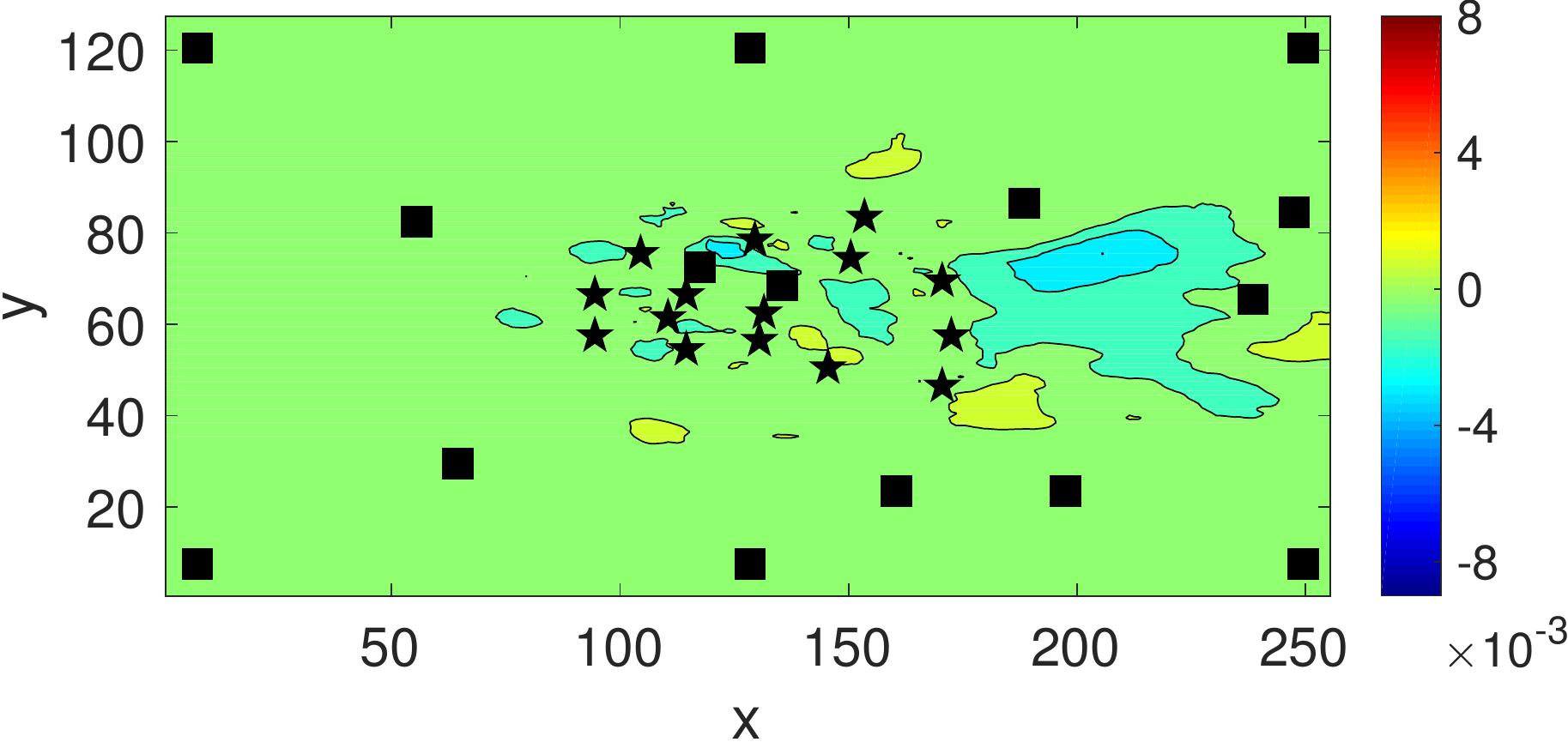}
    \caption{PhIK $\tensor F_r-\tensor F$}
  \end{subfigure}
  \begin{subfigure}[b]{0.32\textwidth}
    \centering%
    \includegraphics[width=\textwidth]{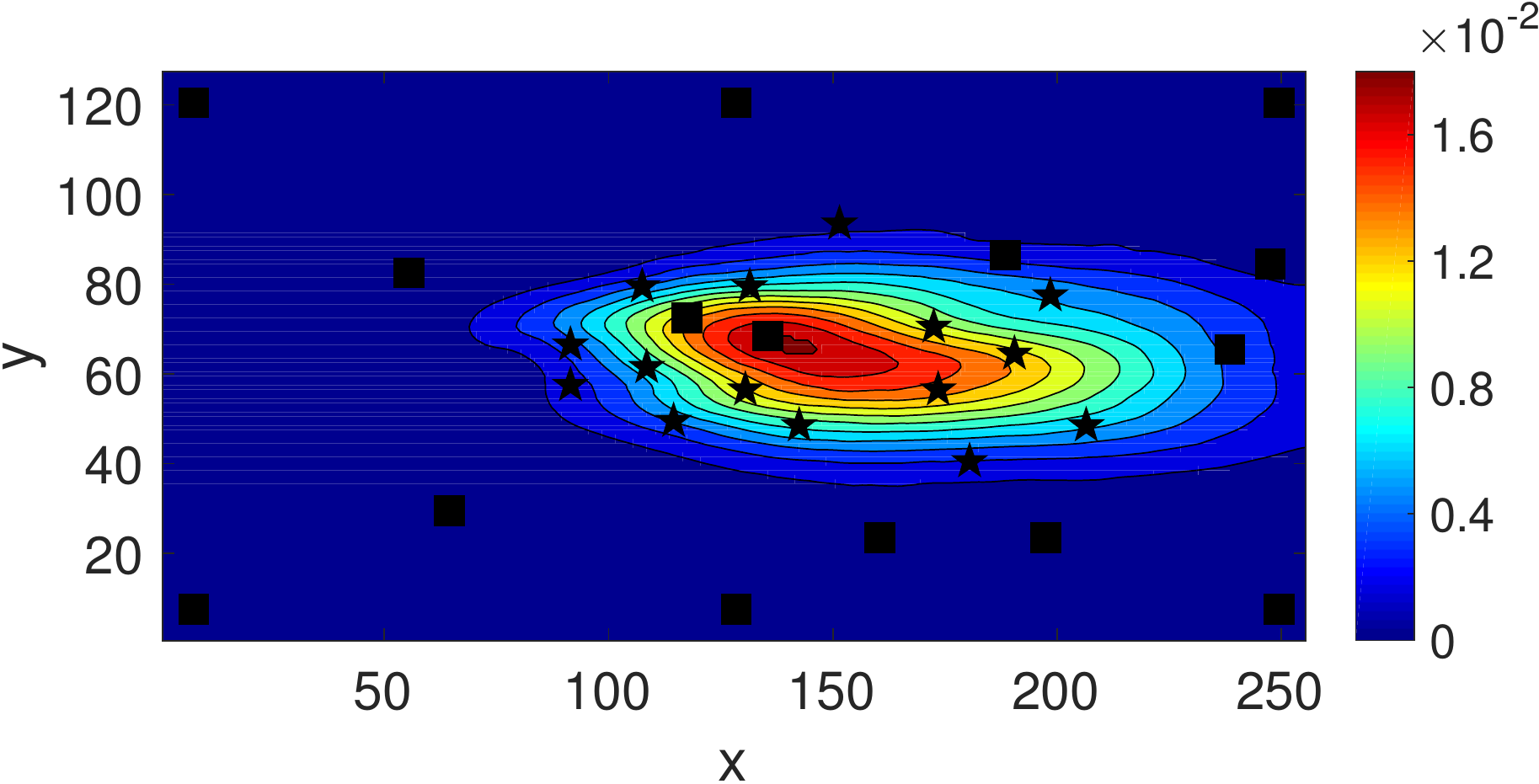}
    \caption{CoPhIK $\tensor F_r$}
  \end{subfigure}
  \begin{subfigure}[b]{0.32\textwidth}
    \centering%
    \includegraphics[width=\textwidth]{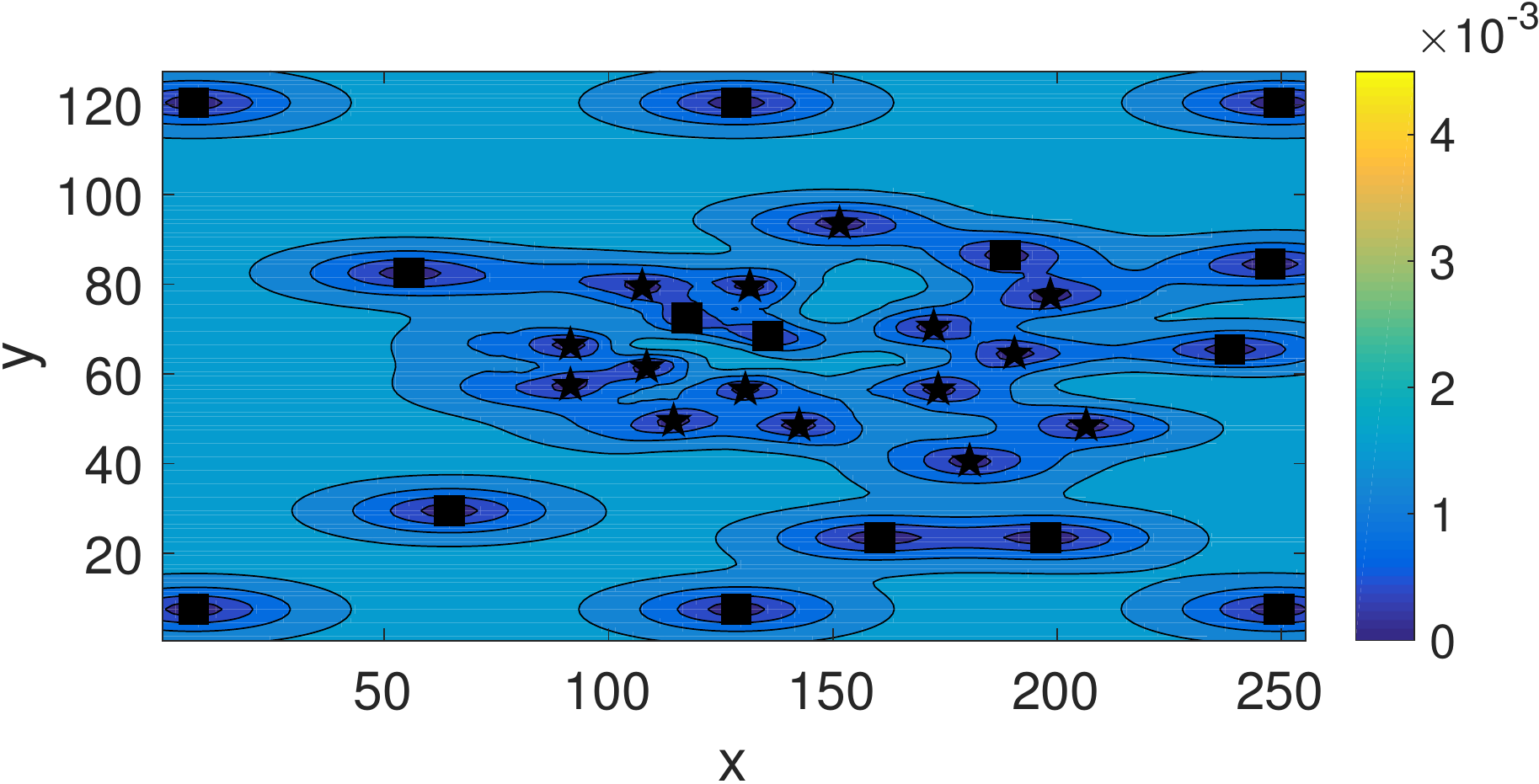}
    \caption{CoPhIK $\hat s$}
  \end{subfigure}
  \begin{subfigure}[b]{0.32\textwidth}
    \centering%
    \includegraphics[width=\textwidth]{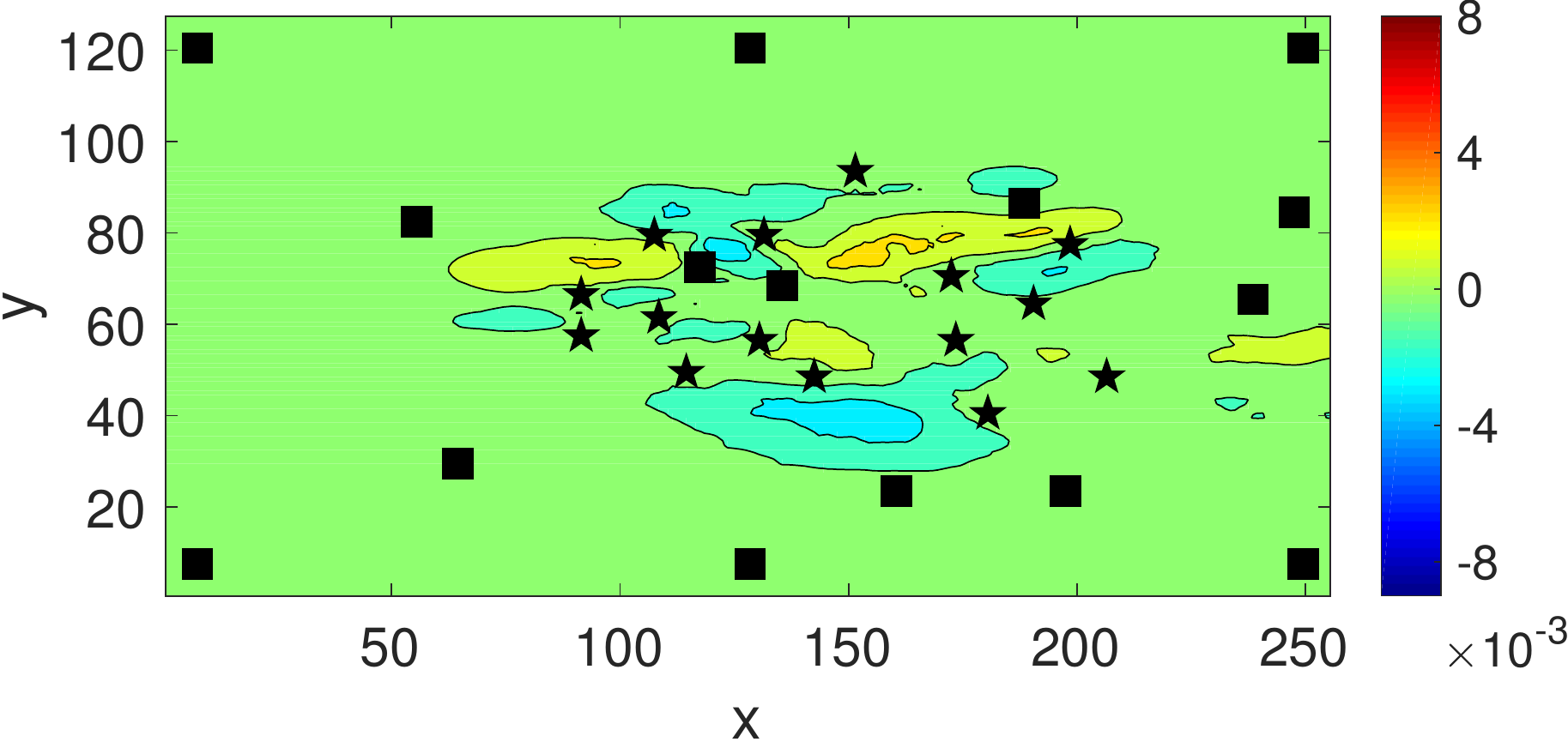}
    \caption{CoPhIK $\tensor F_r-\tensor F$}
  \end{subfigure}
  \caption{Reconstruction of the solute concentration field via active learning
  using Kriging (first row), PhIK (second row) and CoPhIK (third row). Black
squares are the locations of the original $15$ observation. Stars are $15$ 
newly added observations.}
  \label{fig:adv_act}
\end{figure}

Figure~\ref{fig:adv_rel_err} presents the relative error as a function of the 
number of additional observation locations identified via active learning. It 
can be seen that PhIK is more accurate than CoPhIK, especially when number of
observations is small. The difference between these two methods becomes smaller
as more observations are introduced. The error of Kriging decreases in general
with increasing number of observations, but is much larger than that of PhIK and
CoPhIK. For modified PhIK, the magnitude of $\Delta\mu$ is
$\mathcal{O}(10^{-6})$, so that its behavior is similar to that of PhIK.
\begin{figure}[!h]
  \centering%
  \includegraphics[width=0.4\textwidth]{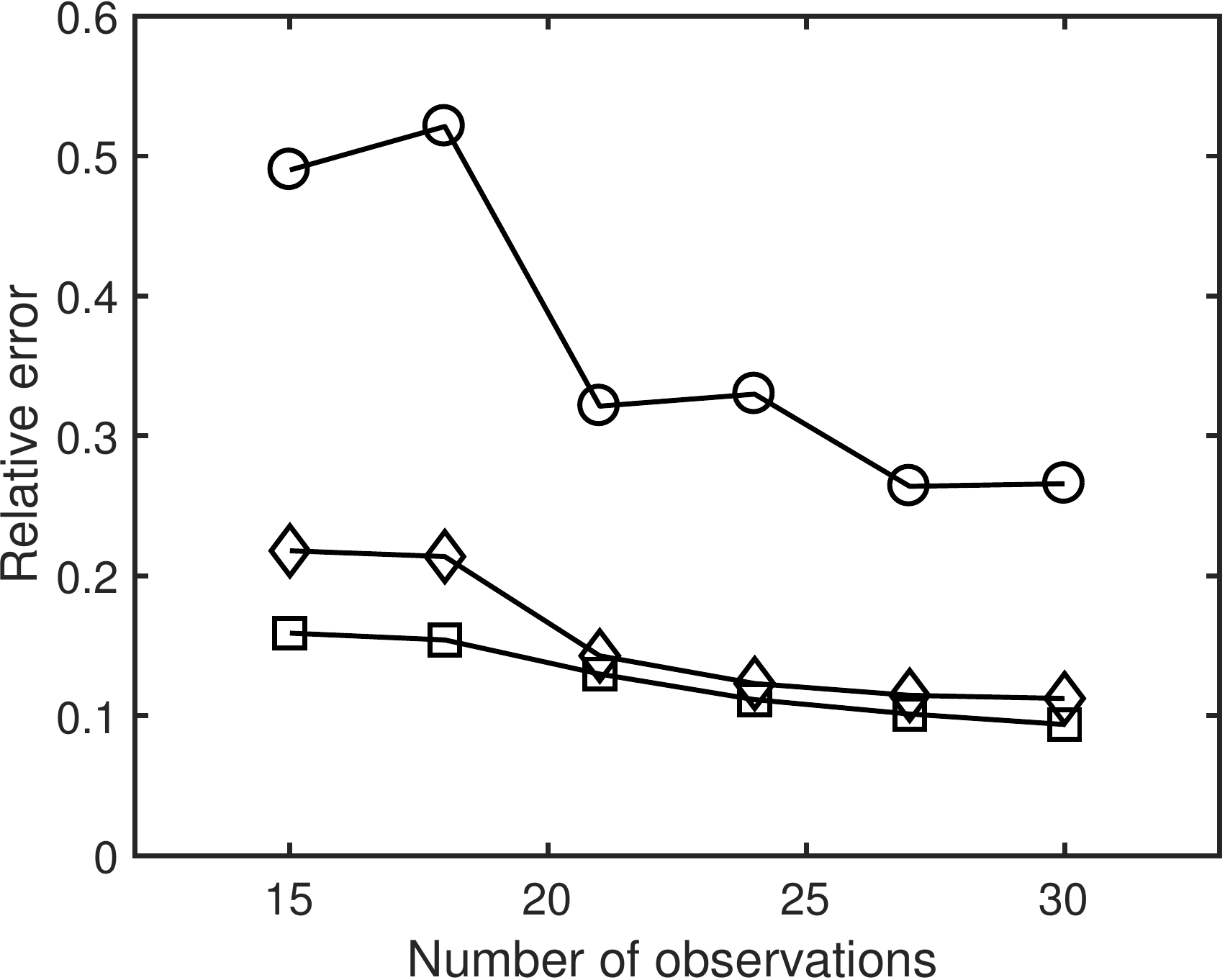}
  \caption{Relative error of reconstructed solute concentration 
  $\Vert\tensor F_r-\tensor F\Vert_F/\Vert\tensor F\Vert_F$ of Kriging
  (``$\circ$"), PhIK (``$\square$") and CoPhIK (``$\diamond$") using different
  numbers of total observations via active learning.}
  \label{fig:adv_rel_err}
\end{figure}

This example is different from the previous two in that a stationary
kernel is not suitable for reconstructing the reference solute concentration 
field. Similarly, a stationary kernel is not adequate for modeling the GP $Y_d$
in CoPhIK. In addition, in this case PhIK outperforms CoPhIK because the ground
truth is a realization of the stochastic model, i.e., the stochastic model in 
this case is accurate. This is different from the previous two examples where we
use incorrect stochastic physical models. A carefully chosen non-stationary 
kernel function would be necessary to improve the accuracy of Kriging and CoPhIK.



%% file: append.tex
\appendix

\section*{Appendices}

\renewcommand{\theequation}{A.\arabic{equation}}

\subsection*{A. Constructing GP $Y(\bm x)$ in PhIK using MLMC}
\label{subsec:mlmc}
For simplicity, we demonstrate the idea via two-level MLMC \cite{YangTT18}. We
use $u_{_L}^m(\bm x)$ ($m=1,...,M_L$) and $u_{_H}^m(\bm x)$ ($m=1,...,M_H$) to 
denote $M_L$ low-accuracy and $M_H$ high-accuracy realizations of the stochastic
model $u(\bm x;\omega)$ for the system. In this work, $u_{_L}^m$ are simulation 
results on coarse grids $\mathbb{D}_L$, and $u_{_H}^m$ are simulations results
on fine grids $\mathbb{D}_H$. We denote 
$\overline u(\bm x)=u_{_H}(\bm x)-u_{_L}(\bm x)$.
Here, when computing $\overline u$, we interpolate $u_L$ from $\mathbb{D}_L$ to
$\mathbb{D}_H$. The mean of $Y(\bm x)$ is estimated as
\begin{equation}
  \label{eq:mlmc_mean}
  \mexp{Y(\bm x)}=\mu(\bm x)\approx \mu_{_\mathrm{MLMC}}(\bm x)
  =\dfrac{1}{M_L}\sum_{m=1}^{M_L} u_{_L}^m(\bm x) +
      \dfrac{1}{M_H}\sum_{m=1}^{M_H}\overline u^m(\bm x).
\end{equation}
which is the standard MLMC estimate of the mean \cite{giles2008multilevel}.
The covariance function of $Y(\bm x)$ is estimated as: 
\begin{equation}
\label{eq:mlmc_cov}
  \begin{aligned}
    & \cov\left\{Y(\bm x), Y(\bm x')\right\} 
    \approx k_{_\mathrm{MLMC}}(\bm x, \bm x') \\
    = &\dfrac{1}{M_L-1} \sum_{m=1}^{M_L} \bigg(u_{_L}^m(\bm x)-
    \dfrac{1}{M_L}\sum_{m=1}^{M_L}u_{_L}^m(\bm x) \bigg) 
    \bigg(u_{_L}^m(\bm x')-\dfrac{1}{M_L}\sum_{m=1}^{M_L}u_{_L}^m(\bm x')\bigg)  \\
    & + \dfrac{1}{M_H-1} \sum_{m=1}^{M_H} \bigg(\overline u^m(\bm x)
      -\dfrac{1}{M_H}\sum_{m=1}^{M_H}\overline u^m(\bm x)\bigg)
  \bigg(\overline u^m(\bm x')-\dfrac{1}{M_H}\sum_{m=1}^{M_H}\overline u^m(\bm x')\bigg).
 \end{aligned}
\end{equation}
Finally, the MLMC-based PhIK model takes the form
\begin{equation}
\label{eq:mlmc_pred}
  \hat y(\bm x^*) = \mu_{_\mathrm{MLMC}}(\bm x^*) + 
  \bm c_{_\mathrm{MLMC}}^\trans\tensor{C}_{_\mathrm{MLMC}}^{-1}(\bm y-\bm\mu_{_\mathrm{MLMC}}), 
\end{equation}
where $\bm\mu_{_\mathrm{MLMC}}=\left(\mu_{_\mathrm{MLMC}}(\bm x^{(1)}), \dotsc, \mu_{_\mathrm{MLMC}}(\bm x^{(N)})\right)^\trans$. 
The matrix $\tensor C_{_\mathrm{MLMC}}$ and vector $\bm c_{_\mathrm{MLMC}}$ are
approximations of $\tensor C$ in Eq.~\eqref{eq:cov_matrix0} and $\bm c$ in 
Eq.~\eqref{eq:cov_vec} using $k_{_\mathrm{MLMC}}$ in Eq.~\eqref{eq:mlmc_cov}. 
The MSE of this prediction is
\begin{equation}
\label{eq:mse_mlmc}
  \hat s^2(\bm x^*) = \sigma^2_{_\mathrm{MLMC}}(\bm x^*)-\bm c_{_\mathrm{MLMC}}^\trans 
  \tensor{C}_{_\mathrm{MLMC}}^{-1}\bm c_{_\mathrm{MLMC}},
\end{equation}
where $\sigma_{_\mathrm{MLMC}}^2(\bm x^*)=k_{_\mathrm{MLMC}}(\bm x^*,\bm x^*)$.
If i.i.d. Gaussian noise is assumed in the observation, replace 
$\tensor C_{_\mathrm{MLMC}}$ with $\tensor C_{_\mathrm{MLMC}}+\delta^2\tensor I$,
where $\delta^2$ is the variance of the noise.

\subsection*{B. Active learning}
In this work, \emph{active learning} is a process of identifying locations for
additional observations that minimize the prediction error and reduce MSE or
uncertainty, 
e.g.,~\cite{cohn1996active,jones1998efficient,tong2001support,collet2015optimism}.
We use a greedy algorithm to add additional observations, i.e., to add new 
observations at the maxima of $s(\bm x)$, 
e.g.,~\cite{forrester2008engineering, raissi2017machine}. Then, we can make a 
new prediction $\hat y(\bm x)$ for $\bm x\in\mathbb{D}$ and compute a new 
$\hat s^2(\bm x)$ to select the next location for additional observation
(see Algorithm~\ref{algo:act}). 
\begin{algorithm}[!h]
  \caption{Active learning based on GPR}
  \label{algo:act}
  \begin{algorithmic}[1]
    \State Specify the locations $\bm X$, corresponding observations $\bm y$, and 
    the maximum number of observations $N_{\max}$ affordable. The number of 
    available observations is denoted as $N$.
    \While {$N_{\max}>N$}
    \State Compute the MSE $\hat s^2(\bm x)$ of MLE prediction $\hat y(\bm x)$ for 
    $\bm x\in \mathbb{D}$.
    \State Locate the location $\bm x_m$ for the maximum of $\hat s^2(\bm x)$ for
    $\bm x\in\mathbb{D}$. 
    \State Obtain observation $y_m$ at $\bm x_m$ and set 
    $\bm X = \{\bm X, \bm x_m\}, \bm y = (\bm y^\trans, y_m)^\trans, N=N+1$.
    \EndWhile
    \State Construct the MLE prediction of $\hat y(\bm x)$ on $\mathbb{D}$ using $\bm X$ 
    and $\bm y$.
  \end{algorithmic}
\end{algorithm}
This selection criterion is based on the statistical interpretation of the
interpolation. More sophisticated sensor placement algorithms can be found in
literature, e.g.,~\cite{jones1998efficient, krause2008near, garnett2010bayesian}, 
and PhIK or CoPhIK are complementary to these methods.


%% file: main.bbl
\begin{thebibliography}{10}

\bibitem{armstrong1984problems}
Margaret Armstrong.
\newblock Problems with universal kriging.
\newblock {\em Journal of the International Association for Mathematical
  Geology}, 16(1):101--108, 1984.

\bibitem{barajas2018multivariate}
David~A Barajas-Solano and Alexandre~M Tartakovsky.
\newblock Multivariate gaussian process regression for multiscale data
  assimilation and uncertainty reduction.
\newblock {\em arXiv preprint arXiv:1804.06490}, 2018.

\bibitem{barajassolano-2018-probability}
David~A. Barajas-Solano and Alexandre~M. Tartakovsky.
\newblock Probability and cumulative density function methods for the
  stochastic advection-reaction equation.
\newblock {\em SIAM/ASA Journal on Uncertainty Quantification}, 6(1):180--212,
  2018.

\bibitem{brahim2004gaussian}
Sofiane Brahim-Belhouari and Amine Bermak.
\newblock Gaussian process for nonstationary time series prediction.
\newblock {\em Computational Statistics \& Data Analysis}, 47(4):705--712,
  2004.

\bibitem{brooks2011multi}
Christopher~James Brooks, AIJ Forrester, AJ~Keane, and S~Shahpar.
\newblock Multi-fidelity design optimisation of a transonic compressor rotor.
\newblock 2011.

\bibitem{chkrebtii2016bayesian}
Oksana~A Chkrebtii, David~A Campbell, Ben Calderhead, Mark~A Girolami, et~al.
\newblock Bayesian solution uncertainty quantification for differential
  equations.
\newblock {\em Bayesian Analysis}, 11(4):1239--1267, 2016.

\bibitem{cockayne2017bayesian}
Jon Cockayne, Chris Oates, Tim Sullivan, and Mark Girolami.
\newblock Bayesian probabilistic numerical methods.
\newblock {\em arXiv preprint arXiv:1702.03673}, 2017.

\bibitem{cohn1996active}
David~A Cohn, Zoubin Ghahramani, and Michael~I Jordan.
\newblock Active learning with statistical models.
\newblock {\em Journal of Artificial Intelligence Research}, 4:129--145, 1996.

\bibitem{collet2015optimism}
Timoth{\'e} Collet and Olivier Pietquin.
\newblock Optimism in active learning with {G}aussian processes.
\newblock In {\em International Conference on Neural Information Processing},
  pages 152--160. Springer, 2015.

\bibitem{dai2017geostatistics}
Heng Dai, Xingyuan Chen, Ming Ye, Xuehang Song, and John~M Zachara.
\newblock A geostatistics-informed hierarchical sensitivity analysis method for
  complex groundwater flow and transport modeling.
\newblock {\em Water Resources Research}, 53(5):4327--4343, 2017.

\bibitem{deutsch1992gslib}
Clayton~V Deutsch and Andr{\'e}~G Journel.
\newblock {\em GSLIB: Geostatistical Software Library and User's Guide}.
\newblock Oxford University Press, 1992.

\bibitem{emmanuel2005mixing}
Simon Emmanuel and Brian Berkowitz.
\newblock Mixing-induced precipitation and porosity evolution in porous media.
\newblock {\em Advances in Water Resources}, 28(4):337--344, 2005.

\bibitem{evensen2003ensemble}
Geir Evensen.
\newblock The ensemble {K}alman filter: {T}heoretical formulation and practical
  implementation.
\newblock {\em Ocean Dynamics}, 53(4):343--367, 2003.

\bibitem{forrester2008engineering}
Alexander Forrester, Andy Keane, et~al.
\newblock {\em Engineering Design via Surrogate Modelling: A Practical Guide}.
\newblock John Wiley \& Sons, 2008.

\bibitem{forrester2007multi}
Alexander~IJ Forrester, Andr{\'a}s S{\'o}bester, and Andy~J Keane.
\newblock Multi-fidelity optimization via surrogate modelling.
\newblock In {\em Proceedings of the Royal Society of London A: mathematical,
  physical and engineering sciences}, volume 463, pages 3251--3269. The Royal
  Society, 2007.

\bibitem{garnett2010bayesian}
Roman Garnett, Michael~A Osborne, and Stephen~J Roberts.
\newblock Bayesian optimization for sensor set selection.
\newblock In {\em Proceedings of the 9th ACM/IEEE international conference on
  Information Processing in Sensor Networks}, pages 209--219. ACM, 2010.

\bibitem{giles2008multilevel}
Michael~B Giles.
\newblock Multilevel {M}onte {C}arlo path simulation.
\newblock {\em Operations Research}, 56(3):607--617, 2008.

\bibitem{hennig2015probabilistic}
Philipp Hennig, Michael~A Osborne, and Mark Girolami.
\newblock Probabilistic numerics and uncertainty in computations.
\newblock {\em Proceedings of the Royal Society London A}, 471(2179):20150142,
  2015.

\bibitem{jones1998efficient}
Donald~R Jones, Matthias Schonlau, and William~J Welch.
\newblock Efficient global optimization of expensive black-box functions.
\newblock {\em Journal of Global Optimization}, 13(4):455--492, 1998.

\bibitem{kennedy2000predicting}
Marc~C Kennedy and Anthony O'Hagan.
\newblock Predicting the output from a complex computer code when fast
  approximations are available.
\newblock {\em Biometrika}, 87(1):1--13, 2000.

\bibitem{kitanidis1997introduction}
Peter~K Kitanidis.
\newblock {\em Introduction to Geostatistics: Applications in Hydrogeology}.
\newblock Cambridge University Press, 1997.

\bibitem{knotters1995comparison}
M~Knotters, DJ~Brus, and JH~Oude Voshaar.
\newblock A comparison of kriging, co-kriging and kriging combined with
  regression for spatial interpolation of horizon depth with censored
  observations.
\newblock {\em Geoderma}, 67(3-4):227--246, 1995.

\bibitem{krause2008near}
Andreas Krause, Ajit Singh, and Carlos Guestrin.
\newblock Near-optimal sensor placements in {G}aussian processes: {T}heory,
  efficient algorithms and empirical studies.
\newblock {\em Journal of Machine Learning Research}, 9(Feb):235--284, 2008.

\bibitem{laurenceau2008building}
Julien Laurenceau and P~Sagaut.
\newblock Building efficient response surfaces of aerodynamic functions with
  kriging and cokriging.
\newblock {\em AIAA journal}, 46(2):498--507, 2008.

\bibitem{le2014recursive}
Loic Le~Gratiet and Josselin Garnier.
\newblock Recursive co-kriging model for design of computer experiments with
  multiple levels of fidelity.
\newblock {\em International Journal for Uncertainty Quantification}, 4(5),
  2014.

\bibitem{lin2009efficient}
Guang Lin and Alexandre~M Tartakovsky.
\newblock An efficient, high-order probabilistic collocation method on sparse
  grids for three-dimensional flow and solute transport in randomly
  heterogeneous porous media.
\newblock {\em Advances in Water Resources}, 32(5):712--722, 2009.

\bibitem{mackay1998introduction}
David~JC MacKay.
\newblock Introduction to gaussian processes.
\newblock {\em NATO ASI Series F Computer and Systems Sciences}, 168:133--166,
  1998.

\bibitem{monterrubio2018posterior}
Karla Monterrubio-G{\'o}mez, Lassi Roininen, Sara Wade, Theo Damoulas, and Mark
  Girolami.
\newblock Posterior inference for sparse hierarchical non-stationary models.
\newblock {\em arXiv preprint arXiv:1804.01431}, 2018.

\bibitem{neal2012bayesian}
Radford~M Neal.
\newblock {\em Bayesian learning for neural networks}, volume 118.
\newblock Springer Science \& Business Media, 2012.

\bibitem{niederreiter1992random}
Harald Niederreiter.
\newblock {\em Random Number Generation and Quasi-Monte Carlo Methods},
  volume~63.
\newblock SIAM, 1992.

\bibitem{ohagan-1998-markov}
Anthony O'Hagan.
\newblock A {M}arkov property for covariance structures, 1998.

\bibitem{paciorek2004nonstationary}
Christopher~J Paciorek and Mark~J Schervish.
\newblock Nonstationary covariance functions for {G}aussian process regression.
\newblock In {\em Advances in Neural Information Processing Systems}, pages
  273--280, 2004.

\bibitem{pan2017optimizing}
Wenxiao Pan, Xiu Yang, Jie Bao, and Michelle Wang.
\newblock Optimizing discharge capacity of li-o2 batteries by design of
  air-electrode porous structure: Multifidelity modeling and optimization.
\newblock {\em Journal of The Electrochemical Society}, 164(11):E3499--E3511,
  2017.

\bibitem{pang2018neural}
Guofei Pang, Liu Yang, and George~Em Karniadakis.
\newblock Neural-net-induced {G}aussian process regression for function
  approximation and {PDE} solution.
\newblock {\em arXiv preprint arXiv:1806.11187}, 2018.

\bibitem{perdikaris2015multi}
P~Perdikaris, D~Venturi, JO~Royset, and GE~Karniadakis.
\newblock Multi-fidelity modelling via recursive co-kriging and
  {G}aussian--{M}arkov random fields.
\newblock {\em Proc. R. Soc. A}, 471(2179):20150018, 2015.

\bibitem{perdikaris2017nonlinear}
Paris Perdikaris, Maziar Raissi, Andreas Damianou, ND~Lawrence, and George~Em
  Karniadakis.
\newblock Nonlinear information fusion algorithms for data-efficient
  multi-fidelity modelling.
\newblock {\em Proceedings of the Royal Society London A}, 473(2198):20160751,
  2017.

\bibitem{plagemann2008nonstationary}
Christian Plagemann, Kristian Kersting, and Wolfram Burgard.
\newblock Nonstationary gaussian process regression using point estimates of
  local smoothness.
\newblock In {\em Joint European Conference on Machine Learning and Knowledge
  Discovery in Databases}, pages 204--219. Springer, 2008.

\bibitem{raissi2017machine}
Maziar Raissi, Paris Perdikaris, and George~Em Karniadakis.
\newblock Machine learning of linear differential equations using {G}aussian
  processes.
\newblock {\em Journal of Computational Physics}, 348:683--693, 2017.

\bibitem{raissi2018numerical}
Maziar Raissi, Paris Perdikaris, and George~Em Karniadakis.
\newblock Numerical {G}aussian processes for time-dependent and nonlinear
  partial differential equations.
\newblock {\em SIAM Journal on Scientific Computing}, 40(1):A172--A198, 2018.

\bibitem{sacks1989design}
Jerome Sacks, William~J Welch, Toby~J Mitchell, and Henry~P Wynn.
\newblock Design and analysis of computer experiments.
\newblock {\em Statistical Science}, pages 409--423, 1989.

\bibitem{sampson1992nonparametric}
Paul~D Sampson and Peter Guttorp.
\newblock Nonparametric estimation of nonstationary spatial covariance
  structure.
\newblock {\em Journal of the American Statistical Association},
  87(417):108--119, 1992.

\bibitem{schober2014probabilistic}
Michael Schober, David~K Duvenaud, and Philipp Hennig.
\newblock Probabilistic ode solvers with {R}unge-{K}utta means.
\newblock In {\em Advances in Neural Information Processing Systems}, pages
  739--747, 2014.

\bibitem{stein1991universal}
A~Stein and LCA Corsten.
\newblock Universal kriging and cokriging as a regression procedure.
\newblock {\em Biometrics}, pages 575--587, 1991.

\bibitem{stein2012interpolation}
Michael~L Stein.
\newblock {\em Interpolation of Spatial Data: Some Theory for Kriging}.
\newblock Springer Science \& Business Media, 2012.

\bibitem{Tart2017WRR}
A.~M. Tartakovsky, M.~Panzeri, G.~D. Tartakovsky, and A.~Guadagnini.
\newblock Uncertainty quantification in scale-dependent models of flow in
  porous media.
\newblock {\em Water Resources Research}, 53:9392--9401, 2017.

\bibitem{tong2001support}
Simon Tong and Daphne Koller.
\newblock Support vector machine active learning with applications to text
  classification.
\newblock {\em Journal of Machine Learning Research}, 2(Nov):45--66, 2001.

\bibitem{white2006stomp}
Mark~D White and Martinus Oostrom.
\newblock {STOMP} subsurface transport over multiple phases, version 4.0,
  user’s guide.
\newblock Technical report, PNNL-15782, Richland, WA, 2006.

\bibitem{williams2006gaussian}
Christopher~KI Williams and Carl~Edward Rasmussen.
\newblock Gaussian processes for machine learning.
\newblock {\em The MIT Press}, 2(3):4, 2006.

\bibitem{XiuH05}
Dongbin Xiu and Jan~S. Hesthaven.
\newblock High-order collocation methods for differential equations with random
  inputs.
\newblock {\em SIAM Journal on Scientific Computing}, 27(3):1118--1139, 2005.

\bibitem{YangCLK12}
Xiu Yang, Minseok Choi, Guang Lin, and George~Em Karniadakis.
\newblock Adaptive {ANOVA} decomposition of stochastic incompressible and
  compressible flows.
\newblock {\em Journal of Computational Physics}, 231(4):1587--1614, 2012.

\bibitem{YangK13}
Xiu Yang and George~Em Karniadakis.
\newblock Reweighted {$\ell_1$} minimization method for stochastic elliptic
  differential equations.
\newblock {\em Journal of Computational Physics}, 248(1):87--108, 2013.

\bibitem{YangTT18}
Xiu Yang, Guzel Tartakovsky, and Alexandre Tartakovsky.
\newblock {PhIK}: {A} physics informed {G}aussian process regression method for
  data-model convergence.
\newblock {\em arXiv preprint arXiv:1809.03461}, 2018.

\bibitem{zhu2017multi}
Xueyu Zhu, Erin~M Linebarger, and Dongbin Xiu.
\newblock Multi-fidelity stochastic collocation method for computation of
  statistical moments.
\newblock {\em Journal of Computational Physics}, 341:386--396, 2017.

\end{thebibliography}
